\documentclass[10pt]{article} 

\usepackage[preprint]{tmlr}

\usepackage{amsmath}
\usepackage{amssymb}
\usepackage{amsthm}
\usepackage{mathtools}
\usepackage{bm}
\usepackage{graphicx}
\usepackage{booktabs}
\usepackage{array}
\usepackage{multirow}
\usepackage{enumitem}
\usepackage{hyperref}
\usepackage{url}

\theoremstyle{plain}
\newtheorem{theorem}{Theorem}
\newtheorem{proposition}{Proposition}
\newtheorem{lemma}{Lemma}
\newtheorem{corollary}{Corollary}
\theoremstyle{definition}

\newtheorem{remark}{Remark}
\newtheorem{assumption}{Assumption}

\newcommand{\R}{\mathbb{R}}

\newcommand{\xbase}{\mathbf{x}_{\text{base}}}
\newcommand{\xspec}{\mathbf{x}_{\text{spec}}}
\newcommand{\Abase}{A_{\text{base}}}
\newcommand{\Aspec}{A_{\text{spec}}}
\newcommand{\bvec}{\mathbf{b}}

\newcommand{\Description}[1]{}

\title{Contextual Deconvolution for Variance-Stable Demand Sensing:\\
Kernel-Modulated Operators in Promotional Retail}

\author{\name Mohammad Forouhesh \\ \email mforouhesh@aut.ac.ir \\
      \addr Amirkabir University of Technology, Tehran, Iran}

\def\month{MM}  
\def\year{YYYY} 
\def\openreview{\url{https://openreview.net/forum?id=XXXX}}

\begin{document}

\maketitle


\begin{abstract}
Machine learning forecasts optimize statistical accuracy yet leave excess operational volatility, inflating safety stock and amplifying the Bullwhip effect. In promotional retail, autoregressive models conflate transient demand shocks with structural shifts, propagating unsmoothed noise upstream. We introduce \textbf{Contextual Deconvolution} (CD), a two-stage estimator that reframes demand sensing as a convex decomposition of transfer-function dynamics: a kernel-modulated banded operator encodes promotional carryover, separating transient shocks from structural baselines without per-SKU training, and hierarchical partial pooling regularizes kernels across SKUs for catalog-scale deployment. The operator is data-derived, not imposed: it reduces to the identity wherever the promotional response is impulsive (most of M5, all of Favorita) and contributes only where genuine multi-day carryover exists (e.g.\ the Hobbies category), so the operational gains rest on the structural decomposition itself, with the kernel an adaptive refinement. We evaluate strictly out-of-sample on 30,490 M5 SKUs and 2,845 Favorita items, giving every calendar-aware baseline the same known future promotional and SNAP calendar CD uses. The load-bearing result is a full inventory-cost accounting: CD's smoothing lowers safety stock, holding cost, and order variance but under-provisions event-driven spikes, so it reduces \emph{total} inventory cost only when holding costs exceed $\sim$20\% of stockout costs (95\% CI $[17\%, 25\%]$), and is otherwise an operational-stability and inventory-capital layer rather than an expected-cost minimizer. The operational signature is a sharply lower Variance Ratio (median 0.0007 versus 0.0135 for Tuned XGBoost on M5, 0.0072 versus 0.0395 on Favorita) and substantially less safety stock ($0.16\times$ Tuned XGBoost's on M5, $0.38\times$ on Favorita) at matched nominal service level---though, consistent with the cost accounting, this lower buffer coincides with higher stockout rates on promotional spikes. Accuracy gains are modest at the median but substantial in the tail: across eleven methods spanning boosted, structural and intermittent families, CD attains the lowest cross-sectional dispersion of per-SKU error (SD $0.23\pm0.02$ versus 1.11--4.51) and mis-forecasts by more than 200\% on $0.8\%\pm0.1$ of SKUs against 9.9--20.6\% for every baseline, ranking first on both measures in all four independent 2{,}000-SKU M5 draws. Because the Variance Ratio and std-based safety stock are both minimized by any sufficiently smooth forecast, we treat them as diagnostics rather than objectives. A supporting operator-theoretic analysis shows the learned demand operators are non-normal, and that CD's compact parametric kernel matches their operational performance while remaining interpretable.
\end{abstract}


\section{Introduction}

Enterprise supply chains rely on data-driven capacity planning. ML models, gradient-boosted trees and temporal transformers, optimize point-wise accuracy (RMSE), leaving excess operational volatility in FMCG environments. Unsmoothed transients propagate upstream into replenishment, inflating safety stock and order variance. This creates a statistical-operational tradeoff: lower RMSE can coincide with higher operational instability where delayed exogenous shocks (promotions, SNAP) project into endogenous demand.

\begin{figure}[t]
\centering
\includegraphics[width=\columnwidth]{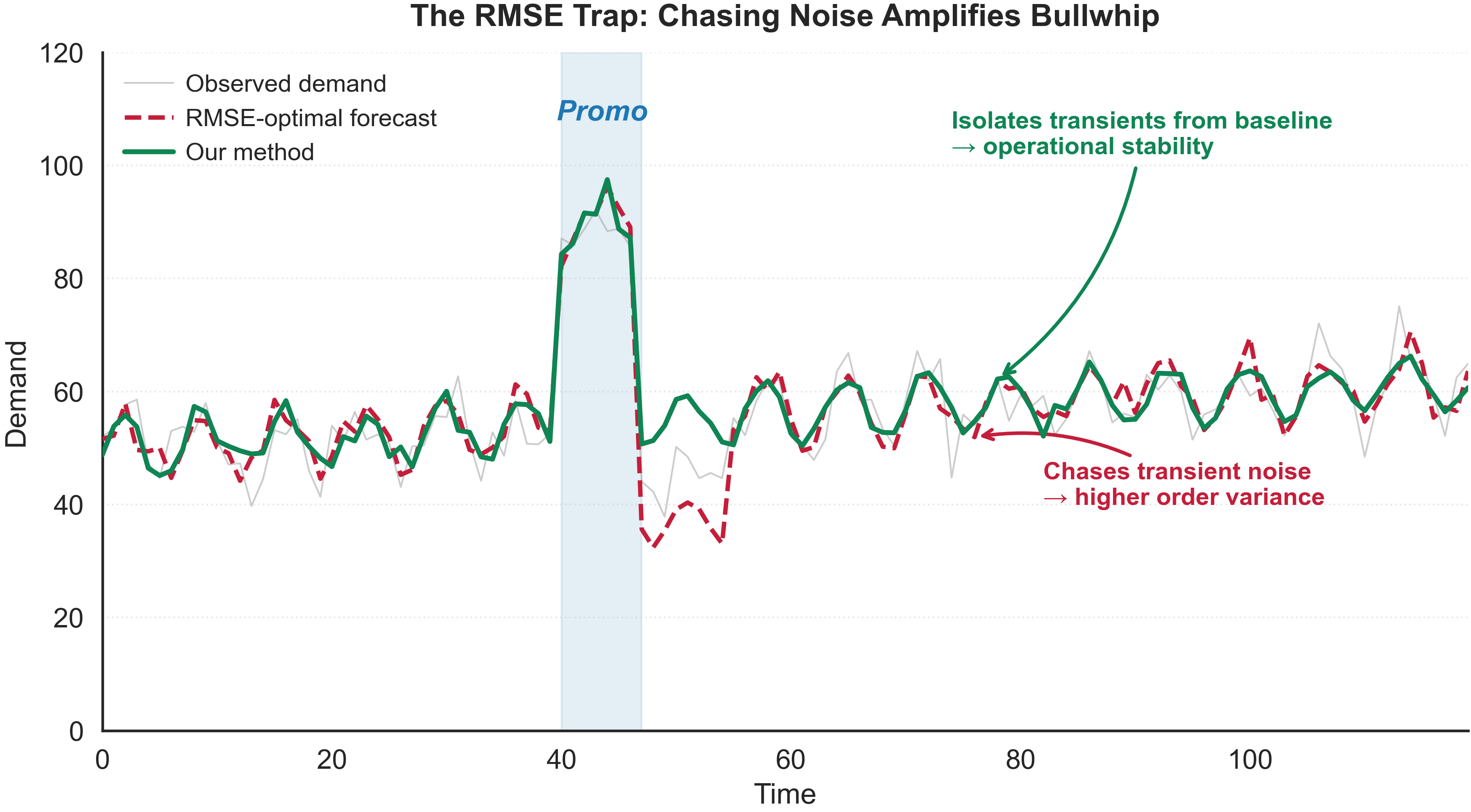}
\caption{The statistical-operational tradeoff. Standard ML forecasts (red, dashed) optimize point-wise accuracy but leave operational volatility in replenishment signals. Our method (green, solid) separates transients from the smooth baseline, yielding more stable planning signals.}
\label{fig:rmse_trap}
\Description{Figure content described in caption.}
\end{figure}

Volatility is driven by discrete exogenous events (promotions, SNAP) with complex temporal responses: a deep promotion generates an immediate surge followed by delayed carryover, while post-promotion dips are often masked at the category level~(\citealt{hendel2006postpromotiondip}). Standard ML models lack structural constraints to map delayed effects back to source events, internalizing them as structural collapses.

We reframe demand sensing as an \textbf{inverse problem}: recovering latent structural and transient components from observed signals.

Classical decomposition assumes an identity forward operator ($A = I$), treating anomalies as instantaneous, a simplification that fails for multi-day promotional footprints. Hawkes processes and distributed lag models capture delayed dynamics but are computationally prohibitive at scale.

We introduce a \textbf{scalable two-stage estimator} that combines distributed lag regression with a convex smoothness-plus-sparsity decomposition, enabled by hierarchical partial pooling for catalog-scale deployment. We call this framework \textbf{Contextual Deconvolution} (CD). It replaces the classical identity operator with a kernel-modulated banded matrix encoding promotional carryover, with a signed surge-and-vacuum extension as an optional inductive bias, separating transient shocks from structural demand without per-SKU training. CD regularizes and scales transfer-function estimation, trading expressiveness for interpretability, scalability, and operational stability.

\paragraph{Contributions.}
\begin{enumerate}
  \item \textbf{A Context-Modulated Banded Operator for Promotional Dynamics.} We introduce a structured operator $\Aspec(\boldsymbol{\theta})$ that encodes promotional carryover directly into a convex inverse problem, with an optional signed Difference-of-Exponentials extension as an inductive bias for post-promotion dips. The operator is a lower-triangular banded matrix from a causal kernel; non-zero at lag zero, it is invertible, and combined with the smoothness-plus-sparsity regularization yields a well-posed recovery of latent shocks from observed demand. This replaces the classical identity mapping ($A = I$) with a physically-motivated operator that enforces delayed carryover structure. We are careful about what the operator does and does not buy: an ablation (\S\ref{sec:large_scale}, Finding 4c) shows the operational gains over gradient boosting come from the \emph{structural decomposition} (smooth baseline plus sparse, calendar-driven shocks), and that the kernel is a \emph{data-derived} component which reduces to an impulse where promotions are impulsive (all of Favorita and most of M5, where it matches an $\Aspec=I$ baseline) yet adds measurable accuracy and stability where genuine multi-day carryover exists (the M5 hobbies category). It generalizes the decomposition to carryover-heavy regimes without degrading the impulsive case. The signed Difference-of-Exponentials extension is likewise optional rather than load-bearing: on M5 the post-promotion \emph{vacuum} is not empirically separable from pure exponential decay (Appendix~\ref{app:real_kernels}), so we report it as an inductive bias we test and find inessential, not a required component.
  
  \item \textbf{A Two-Stage Estimation Pipeline with Stability Analysis.} We introduce a scalable two-stage pipeline (Robust Huber Detrending + Distributed Lag Regression) that decouples baseline estimation from kernel recovery. Stage-2 kernel recovery is a standard ridge regression; stability analysis ensures kernel estimates degrade gracefully with bounded detrending error, confirming the two-stage separation is a stability mechanism, not merely a speed approximation (Appendix~\ref{app:complexity}).
  
  \item \textbf{Hierarchical Partial Pooling for Enterprise Scalability.} We introduce a partial pooling mechanism that regularizes per-SKU kernel estimates toward category-level and global medians. This makes the method deployable for SKUs with sparse event histories, common in enterprise retail, without per-SKU tuning or large historical samples. The full pipeline runs on CPU in under 40 minutes for 30,000 SKUs.
  
  \item \textbf{Empirical Validation at Scale, with an Honest Cost Accounting.} On 30,490 M5 SKUs and 2,845 Favorita items---strictly out-of-sample, with every calendar-aware baseline given the same future calendar---CD reduces operational variance (median Variance Ratio 0.0007 on M5, 0.0072 on Favorita, and well under half the safety stock of Tuned XGBoost---$0.16\times$ on M5, $0.38\times$ on Favorita) and wins the paired VR comparison on $\sim$74\% of SKUs (\S\ref{sec:large_scale}). Its accuracy contribution is one of \emph{reliability} rather than central tendency: against eleven methods spanning boosted, structural and intermittent families, CD attains the lowest cross-sectional dispersion of per-SKU error and an order-of-magnitude lower catastrophic-failure rate ($0.8\%$ of SKUs beyond $200\%$ error, versus $9.9$--$20.6\%$), ranking first on both measures in all four independent 2{,}000-SKU draws. Because the Variance Ratio is minimized by any sufficiently smooth forecast, we treat it as a diagnostic and anchor the contribution on a full inventory-cost analysis: CD lowers \emph{total} cost only when the holding-to-stockout ratio $h/p \gtrsim 0.20$ (95\% CI $[0.17, 0.25]$), and is otherwise a stability and inventory-capital layer rather than a cost minimizer. We further validate robustness via calendar corruption, kernel misspecification, and timing-shift tests.

  \item \textbf{An Interpretability Diagnostic via Operator Theory.} As a diagnostic---not a performance driver---we show that unconstrained Koopman/EDMD operators learned from retail demand are highly non-normal (transient growth 7--12$\times$ above spectral prediction on M5, up to 45.9$\times$ on Favorita), yet attain operational performance indistinguishable from CD's compact parametric kernel. This supports our modeling choice: a compact, interpretable operator loses nothing operationally against a diffuse learned one. The full analysis is deferred to Appendix~\ref{app:edmd}.
\end{enumerate}

\textbf{Scope and positioning.} CD is not a replacement for existing forecasting pipelines but an \emph{operational decision-support layer} that runs alongside them; it reduces total inventory cost only in holding-dominated regimes ($h/p \gtrsim 0.20$), and otherwise trades expected cost for operational stability and lower inventory capital. It is a \emph{conditional} framework: it requires knowledge of future promotional and SNAP calendars, and reverts to baseline-only forecasting when they are unknown.

\section{Related Work}
\label{sec:related_work}

Our framework sits at the intersection of predictive machine learning, econometric structural models, and convex signal decomposition. We explicitly position Contextual Deconvolution within the classical literature it builds upon.

\subsection{Machine Learning and the Bullwhip Effect}

Modern retail forecasting is dominated by gradient-boosted decision trees \citep{chen2016xgboost} and deep autoregressive models \citep{salinas2020deepar,wen2023transformers}. As demonstrated in the M5 Forecasting Competition, these models excel at minimizing point-wise statistical errors (RMSE) over massive, hierarchical datasets. However, supply chain operations are highly sensitive to demand volatility propagation, where unsmoothed forecasts inflate upstream safety stock and order variance \citep{lee1997bullwhip,macdonald2003supply,fransoo2000measuring}. Standard autoregressive models leave more volatility in the forecast than necessary because they map raw features directly to targets without structural constraints, conflating transient shocks with structural shifts.

\subsection{Inventory-Aware and Newsvendor-Aware Learning}

A growing body of work integrates inventory costs directly into the learning objective rather than treating them as a post-forecast operational step. \citet{orojlooy2020applying} train deep reinforcement-learning agents to minimize lost sales and holding costs end-to-end, bypassing explicit forecasting. \citet{bertsimas2020predictive} formulate the newsvendor problem with features and show that cost-sensitive learning can outperform two-stage predict-then-optimize pipelines. Recent work at the intersection of ML and operations research learns quantile forecasts or directly optimizes for inventory cost surrogates \citep{ban2019big}. These methods target the \emph{decision} layer; CD targets the \emph{decomposition} layer, providing an interpretable structural baseline and sparse shock vector that feeds into any downstream inventory optimizer. The two approaches are complementary: CD reduces the input volatility that inventory-aware methods must handle.

\subsection{Structural Time Series and Transfer Function Models}

The transfer-function formulation we adopt (Eq.~\eqref{eq:forward_model}, \S\ref{sec:method}) is not new. \textbf{Transfer function models} \citep{box1970time, box2016time} and \textbf{dynamic linear models} (DLMs) \citep{harvey1990forecasting, durbin2012time} have long represented endogenous variables as convolutions of exogenous shocks with parametric impulse responses. ARIMAX and ETS with regressors \citep{hyndman2008automatic} are standard tools for incorporating covariates. \textbf{Distributed lag models} \citep{lutkepohl2005new, shumway2017time} explicitly model delayed effects via lagged exogenous variables.

What classical methods lack is (1) a principled regularization mechanism separating transient shocks from smooth baselines at scale, and (2) a hierarchical pooling mechanism making per-SKU deployment feasible for sparse event histories. CD is a \emph{regularized, scalable implementation} of the transfer function framework: ridge shrinkage in the distributed lag regression, plus a smoothness-plus-sparsity convex decomposition replacing the classical residual-with-ARMA structure.

\subsection{Supply-Chain-Native Forecasting}

In intermittent-demand settings, \textbf{Croston's method} \citep{croston1972forecasting} and the \textbf{SBA correction} \citep{syntetos2005categorization} are the standard baselines for separating demand size from inter-arrival intervals. These methods are designed for sporadic, lumpy demand, not for the frequent but transient promotional surges characteristic of FMCG. They do not model delayed carryover or cannibalization effects. CD complements these methods by targeting the opposite regime: high-frequency, event-driven demand with structured temporal dynamics.

\subsection{Convex Signal Decomposition and Trend Filtering}

Decomposing a time series into structural and transient components is a canonical problem in signal processing. Foundational methods like the Hodrick-Prescott filter, $\ell_1$ Trend Filtering \citep{kim2009ell1,tibshirani2014adaptive}, and Robust PCA \citep{candes2011robust,chandrasekaran2011rank} separate data into low-rank (baseline) and sparse (anomalous) matrices. Classical decomposition assumes an identity forward operator ($A = I$), treating anomalies as instantaneous, which fails for multi-day promotional footprints. Econometric DLMs and Hawkes processes capture delayed dynamics but suffer multicollinearity and prohibitive likelihood estimation at enterprise scale.

\subsection{Scalable Time-Series Decomposition and Industrial Forecasting}

Amazon's DeepAR~\citep{salinas2020deepar} trains global probabilistic models on millions of SKUs. Uber's Orbit~\citep{ng2020orbit} implements Bayesian exponential smoothing at scale with hierarchical pooling. Decomposition-based transformers---Autoformer~\citep{wu2021autoformer}, FEDformer~\citep{zhou2022fedformer}, TimeMixer~\citep{wang2024timemixer}---use learned neural modules to separate trend and seasonality, and N-BEATS~/N-HiTS~\citep{oreshkin2019nbeats,challu2023nhits} model hierarchical seasonality via basis expansions. These methods achieve strong statistical accuracy but require GPU clusters, stochastic training, and dedicated ML infrastructure.

CD occupies a different point in the design space. Rather than replacing classical econometrics with deep learning, we regularize and scale it: distributed lag regression with ridge shrinkage plus a convex smoothness-plus-sparsity decomposition makes transfer-function estimation feasible on tens of thousands of SKUs without per-SKU training or GPU resources. The decomposition is explicit, not latent: the recovered baseline and shock vectors are directly interpretable, and the kernel parameters have physical meaning (surge rate, vacuum rate, decay). This positions CD as an interpretable, CPU-runnable alternative to deep learning at scale.

\subsection{Promotion-Effect Decomposition in Marketing Science}

The decomposition of retail demand into a smooth baseline plus promotional uplift with carryover and post-promotion dips is standard in marketing science and commercial demand-planning systems. \citet{vanheererde2000estimation} estimate pre- and post-promotion dips from store-level scanner data using distributed-lag models with explicit baseline and event components. \citet{hendel2006postpromotiondip} document the "post-promotion dip puzzle"---the dip is often masked by other dynamic effects in aggregate data. \citet{neslin2008salespromotion} provide a comprehensive review of sales promotion models including base-lift decomposition with carryover. These methods typically rely on explicit model specification with ARMA residuals, and they do not incorporate the convex regularization and hierarchical pooling that CD introduces. Our contribution is not the decomposition itself but a \emph{regularized, poolable, CPU-scale implementation} that makes the classical base-lift architecture feasible for catalogs with tens of thousands of SKUs.

\section{Method: Contextual Deconvolution via Kernel-Modulated Operators}
\label{sec:method}

We formulate demand sensing not as predictive regression over raw observations, but as an \textbf{inverse problem} in which we recover latent demand components from observed signals. This framing explicitly separates stable demand from transient shocks.

\subsection{Inverse Problem Formulation}

Let the observed demand signal be $\bvec \in \R^{T}$. We model $\bvec$ as a superposition of two latent trajectories mapped through forward operators:
\begin{equation}
\bvec = \Abase \xbase + \Aspec(\boldsymbol{\theta}) \xspec + \boldsymbol{\varepsilon}, \qquad \boldsymbol{\varepsilon} \sim \mathcal{N}(\mathbf{0}, \sigma^{2}I),
\label{eq:forward_model}
\end{equation}
where $\xbase \in \R^{T}$ is a smooth structural baseline, $\xspec \in \R^{T}$ is a sparse transient signal, and $\Aspec(\boldsymbol{\theta})$ is a context-modulated operator parameterized by kernel hyperparameters $\boldsymbol{\theta}$. We fix $\Abase = I$ to preserve interpretability.

We recover regularized estimates of the latent trajectories by solving the convex optimization problem:
{\small
\begin{equation}
\min_{\xbase, \xspec} \underbrace{\| \bvec - (\xbase + \Aspec \xspec) \|_{2}^{2}}_{\text{Fidelity}} + \underbrace{\lambda_{\text{smooth}} \| L \xbase \|_{2}^{2}}_{\text{Smoothness}} + \underbrace{\lambda_{\text{sparse}} \| \xspec \|_{1}}_{\text{Sparsity}}.
\label{eq:convex_decomp}
\end{equation}
}\textbf{Identifiability.} The decomposition is not strictly identifiable without regularization: infinitely many pairs $(\xbase, \xspec)$ explain the observations. However, under a mutual incoherence condition on the event calendar---requiring that no individual event's impulse-response pattern be well-approximated by a smooth signal---the smoothness-plus-sparsity prior yields support identification (no false inclusions, and sign recovery under a minimum-signal condition) and a baseline error decomposing into smoothing bias, kernel propagation, and a noise floor (Theorem~\ref{thm:identifiability}, Appendix~\ref{app:theory}). The smoothness and sparsity penalties\footnote{Median non-zeros range from 1.4 on low-eventfulness SKUs (Appendix~\ref{app:timing}) to 1.9 for the two-stage estimator overall, versus 95.6 under joint optimization (Appendix~\ref{app:ablation_joint}) across the evaluated sample.} select a regularized estimate that is stable for downstream planning. We stress that these guarantees characterize an \emph{idealized} regime---Theorem~\ref{thm:identifiability}'s mutual-incoherence and minimum-signal conditions are demanding and certify recovery only of shocks large relative to the seasonal residual---and our empirical conclusions do not depend on them holding exactly. Each theorem instead has a measured proxy we report: recovered shock sparsity (Theorem~\ref{thm:identifiability}; median 1.4--1.9 non-zeros), the structural-stream Variance Ratio (Theorem~\ref{thm:twostage}), and the near-identical category- vs.\ global-pooled results at full catalog scale (Theorem~\ref{thm:pooling}, Table~\ref{tab:scale_validation}).

\subsection{Structural Baseline and Derivative-Based Smoothness}

The smoothing operator $L$ is a stacked finite-difference operator penalizing discrete derivatives. Crucially, $L$ explicitly omits the identity matrix $I$, avoiding ridge shrinkage bias that would pull the baseline toward zero regardless of true demand level.

\subsection{Context-Modulated Banded Operators}

Classical decomposition assumes $A = I$ (instantaneous shocks) or a Toeplitz matrix (constant impulse response). In FMCG, neither holds: a Black Friday promotion has a different temporal footprint than a quiet Tuesday promotion. One might therefore learn the operator directly from data. However, operator-theoretic analysis reveals a structural obstacle.

\textbf{The non-normality problem.} EDMD-learned Koopman operators on promotional demand are spectrally stable ($\rho < 1$) but highly non-normal: the $\epsilon=10^{-1}$ pseudospectral contour crosses the unit circle, and the numerical abscissa $\omega(K)$ exceeds the spectral abscissa $\alpha(K)$ by up to 1.40 (Appendix~\ref{app:edmd}). This non-normality induces transient growth 7--12$\times$ above the spectral prediction on M5 (up to 45.9$\times$ in the most extreme Favorita category, 10.2$\times$ globally), diffusing the impulse response across lags into a flat, less interpretable pattern. This is not an artifact of the degree-2 lifting: across linear, quadratic, full-quadratic, and cubic observable dictionaries the operator remains strongly non-normal, with the transient-growth factor increasing monotonically with basis richness (from $2.4$ under a linear basis to $12.7$ under the quadratic basis; Appendix~\ref{app:edmd}, Table~\ref{tab:dict_robustness}). The degree-2 value we use is thus conservative, and the non-modal transient response is a property of the underlying promotional-demand dynamics.

\textbf{Parametric family as an interpretable inductive bias.} Rather than learning an unconstrained operator, we restrict the forward map to a low-dimensional parametric family that avoids the non-normal diffusion, yielding a compact, physically interpretable form. The family must capture (i) an immediate demand surge at the event day, (ii) gradual carryover over subsequent days, and (iii) optional post-promotion cannibalization. These requirements motivate the \textbf{surge-and-vacuum Difference-of-Exponentials (DoE)} kernel.

\textbf{Supply-Chain Kernel Families.}
\begin{table}[t]
\centering
\footnotesize
\setlength{\tabcolsep}{2pt}
\caption{Supply-Chain Kernel Families}
\label{tab:kernel_families}
\resizebox{\columnwidth}{!}{%
  \begin{tabular}{llll}
  \toprule
  Event Dynamics & Kernel Family & Formula ($k \geq 0$) & Parameters \\
  \midrule
  Instantaneous shock & Impulse & $\phi(k) = \mathbb{I}[k=0]$ & None \\
  Gradual carryover & Exponential decay & $\phi(k; \tau) = e^{-k/\tau}$ & $\tau > 0$ \\
  Surge-and-vacuum & Diff.~of Exponentials & $\alpha e^{-k/\tau_1} - \beta e^{-k/\tau_2}$ & $\alpha > \beta \ge 0, \tau_1 < \tau_2$ \\
  \bottomrule
  \end{tabular}
}
\end{table}

The DoE kernel enforces $\alpha > \beta > 0$ and $\tau_1 < \tau_2$, guaranteeing an immediate positive surge at $k=0$ followed by a slower negative tail at larger $k$. Because the kernel itself is signed, a positive latent shock can generate negative observed demand at future lags, modeling post-promo cannibalization without manual ``demand dip'' features. \textbf{Empirical validation.} On real M5 data, category-level empirical kernels do not exhibit a negative region (Appendix~\ref{app:real_kernels}); the post-promotion dip is not detectable without controls for feature and display effects, consistent with Hendel and Nevo's masking argument. The vacuum component is therefore an inductive bias that regularizes the tail to zero, not an empirically recovered feature. A pure exponential decay (always positive) yields comparable sparsity (Appendix~\ref{app:ablation_exponential}), confirming the operational benefits come from smoothness and carryover decay.

\textbf{Bandwidth selection.} The kernel lag $n$ is not tuned to a downstream metric. Estimating the empirical promotional impulse response on 500 M5 SKUs out to $28$ lags---twice the default---shows the response is dominated by the event day plus a short carryover: the median impulse response exceeds $10\%$ of the event-day peak only up to lag $10$ (hobbies; foods, household, and the global pool are near-impulsive, falling below $10\%$ immediately after the event day), while a $14$-lag window already captures $\geq 94\%$ of the impulse-response energy in every category ($99.5\%$ globally; Table~\ref{tab:bandwidth}). The default $n=14$ (a two-week horizon) thus sits comfortably above the longest actionable carryover. The residual is a low-level flat tail---below $10\%$ of peak but persistent---that the parametric kernel deliberately omits by decaying to zero (Appendix~\ref{app:real_kernels}) and that the smooth baseline absorbs. Because the response is estimated over a $28$-lag window rather than $14$, the analysis does not assume its own conclusion; it is reproducible via \texttt{run\_bandwidth\_analysis.py}.

\begin{table}[t]
\centering
\footnotesize
\caption{Data-driven kernel bandwidth on M5. The empirical promotional impulse response is estimated to lag $28$ (median across SKUs); the default $n=14$ captures $\geq 94\%$ of its energy in every category and exceeds the longest actionable-carryover lag.}
\label{tab:bandwidth}
\begin{tabular}{@{}lccc@{}}
\toprule
Category & SKUs & IR energy by $n{=}14$ & Carryover $>10\%$ peak \\
\midrule
Foods & 249 & $0.999$ & lag $\le 1$ (impulsive) \\
Household & 157 & $0.983$ & lag $\le 1$ (impulsive) \\
Hobbies & 94 & $0.944$ & lag $10$ \\
Global & 500 & $0.995$ & lag $\le 1$ (impulsive) \\
\bottomrule
\end{tabular}
\end{table}

With the parametric family fixed, we construct $\Aspec(\boldsymbol{\theta})$ as a \textbf{context-modulated banded matrix} where each sub-diagonal is modulated by exogenous features. For each event type $m \in \mathcal{M}$ with binary indicator $\mathbf{f}_m$ and causal kernel $\phi_m(k; \boldsymbol{\theta}_m)$:
\begin{equation}
\Aspec(\boldsymbol{\theta}) = \sum_{m \in \mathcal{M}} \sum_{k=0}^{n} \phi_m(k; \boldsymbol{\theta}_m) \, S_k \, \operatorname{diag}(\mathbf{f}_m),
\label{eq:operator}
\end{equation}
where $S_k$ is the causal shift matrix with 1s on the $k$-th sub-diagonal.

\textbf{Source-Modulated Propagation.} Equation~\eqref{eq:operator} reveals that the contribution arriving at day $t$ from lag $k$ is modulated by the feature value at the \emph{source time} $t-k$, not the receiver time $t$. This aligns with supply chain dynamics: a Tuesday promotion carries over differently because Tuesday's consumer behavior differs from Friday's.

The parametric DoE kernel achieves operational performance indistinguishable from the data-driven EDMD operator on the same 200-SKU evaluation subset (median VR $0.0006$ for both on M5) while offering interpretable parameters and low non-normality by construction. Specifically, its bandedness confines the impulse response to the kernel support, whereas EDMD's dense non-parametric estimation spreads energy across all lags (see Appendix~\ref{app:edmd} for the Henrici-number analysis).

\subsection{Why This Is an Inverse Problem, Not Additive Regression}

The distinction from additive regression is structural: the forward map is $\mathbf{x}_{\mathrm{spec}} \mapsto A_{\mathrm{spec}}\mathbf{x}_{\mathrm{spec}}$, a causal convolution, not the identity. A single shock propagates through $A_{\mathrm{spec}}$ to influence future timesteps; delayed dynamics are a consequence of the same shock, not separate features. Because $A_{\mathrm{spec}}$ is lower-triangular with non-zero diagonal, the forward map is invertible. For fixed $A_{\mathrm{spec}}$ and $\lambda_{\mathrm{smooth}} > 0$, Problem~\eqref{eq:convex_decomp} is convex (Appendix~\ref{app:convexity}) and the regularized recovery is well-posed; the $\ell_1$ sparsity prior selects a minimal-norm representative from the solution set.

The operator $A_{\text{spec}}$ can be viewed as a compact, interpretable approximation to the Koopman operator governing promotional demand. Theorem~\ref{thm:twostage} (Appendix~\ref{app:theory}) bounds the baseline recovery error under operator perturbation; the non-normality of the data-driven Koopman operator (Appendix~\ref{app:edmd}) motivates preferring a compact parametric family on interpretability grounds.

\textbf{What CD optimizes that ARIMAX cannot.} ARIMAX estimates a single predictive model: it regresses demand on lagged endogenous terms and exogenous dummies, producing a forecast but \emph{not} a decomposed baseline-plus-shock representation. The ARIMAX forecast is a single trajectory; there is no separate sparse shock vector to inform event-specific safety stock, and no smooth baseline for stable capacity planning. CD's convex decomposition (Problem~\eqref{eq:convex_decomp}) jointly enforces smoothness on the baseline and sparsity on the shocks, recovering two interpretable streams. Theorem~\ref{thm:identifiability} guarantees that under a mutual incoherence condition, the shock support is recovered exactly---something ARIMAX, which lacks an $\ell_1$ sparsity prior, does not naturally provide. This is not merely a modeling convenience: the decomposed streams are directly actionable for downstream S\&OP and safety-stock optimization, whereas ARIMAX outputs a single forecast that must be post-processed.

Data-driven alternatives like Cross-Correlation Function (CCF) fail because multi-day promo blocks convolve with their own autocorrelation, blurring lag signatures, and because CCF requires a residual $\mathbf{r} = \bvec - \xbase$, creating a circular dependency when $\Aspec = I$ forces all carryover into the event day. On synthetic data with a known DoE kernel, our two-stage method achieves $6.5\times$ lower recovery error than CCF (Appendix~\ref{app:ccf}).

We resolve this through a \textbf{two-stage estimation procedure}:

\textbf{Stage 1: Robust Detrending.} Extract a robust trend $\hat{\mu}$ by solving
\begin{equation}
\min_{\boldsymbol{\mu}} \sum_{t=1}^{T} \rho_{\text{Huber}}(\bvec[t] - \boldsymbol{\mu}[t]) + \lambda_{\text{smooth}} \| L \boldsymbol{\mu} \|_{2}^{2},
\end{equation}
where $\rho_{\text{Huber}}$ downweights extreme observations. The detrended residual is $\mathbf{r} = \bvec - \hat{\mu}$.

\textbf{Stage 2: Distributed Lag Regression.} Construct an augmented design matrix $\Phi \in \R^{T \times K(n+1)}$ containing all event features and their causal lags up to $n$ days. Estimate kernel coefficients via ridge regression:
\begin{equation}
\min_{\boldsymbol{\phi}} \| \mathbf{r} - \Phi \boldsymbol{\phi} \|_{2}^{2} + \lambda_{\text{kernel}} \| \boldsymbol{\phi} \|_{2}^{2}.
\end{equation}

For parametric families, fit $\boldsymbol{\theta}_m$ by nonlinear least squares against the empirical $\hat{\phi}_m(k)$. The estimated kernels are then fixed to construct $\Aspec$, and the convex decomposition (Problem~\eqref{eq:convex_decomp}) is solved once. \textbf{Note:} The two-stage pipeline is a heuristic approximation. While the final decomposition is convex for fixed $\Aspec$, the joint optimization of kernels and decomposition is not convex. We trade global optimality for scalability and interpretability. Theorem~\ref{thm:twostage} (Appendix~\ref{app:theory}) bounds the baseline recovery error and controls the temporal Variance Ratio of the structural stream through explicitly bounded estimation error.

\subsection{Hierarchical Kernel Estimation via Partial Pooling}

Per-SKU kernel estimation is statistically fragile: many SKUs have fewer than three promotional events in their history, making the empirical shock distribution too noisy to estimate a reliable kernel. In enterprise retail, this is the norm, not the exception. A catalog of 10,000 SKUs may contain thousands of items with sparse or zero event histories.

We address this through \textbf{partial pooling} (hierarchical estimation): kernel parameters are estimated at multiple levels of aggregation and regularized toward a shared prior. For a category $\mathcal{C}$ containing SKUs $s \in \mathcal{C}$:

\begin{enumerate}
  \item \textbf{Per-SKU estimation.} For each SKU, estimate $\hat{\phi}_s(k)$ via Stage 2.
  \item \textbf{Category pooling.} Compute the category-level kernel as the median (or trimmed mean) across SKUs: $\hat{\phi}_{\mathcal{C}}(k) = \text{median}_{s \in \mathcal{C}} \{ \hat{\phi}_s(k) \}$.
  \item \textbf{Global pooling.} Compute the global kernel as the median across all SKUs: $\hat{\phi}_{\text{global}}(k) = \text{median}_{s} \{ \hat{\phi}_s(k) \}$.
  \item \textbf{Adaptive deployment.} For SKUs with rich event histories ($n_{\text{events}} \geq 10$), use the per-SKU kernel with category shrinkage. For SKUs with moderate histories ($3 \leq n_{\text{events}} < 10$), use the category-pooled kernel. For SKUs with sparse histories ($n_{\text{events}} < 3$), use the global-pooled kernel.
\end{enumerate}

The median is chosen over the mean because it is robust to outlier SKUs whose individual kernels are corrupted by noise or idiosyncratic demand patterns. This design makes the method scalable to catalogs with tens of thousands of SKUs, many of which have sparse event histories, without requiring per-SKU tuning. Theorem~\ref{thm:pooling} (Appendix~\ref{app:theory}) proves that the pooled kernel concentrates around the category median up to the local heterogeneity scale, with failure probability decaying exponentially in $|\mathcal{C}|$, and that pooling is beneficial when category heterogeneity is small relative to per-SKU estimation noise.

\textbf{Why this matters.} In experiments (\S\ref{sec:partial_pooling_exp}), category-pooled and global-pooled kernels yield nearly identical reconstruction to per-SKU, with slightly improved RMSE from regularizing noisy individual estimates. Median sparsity is zero for all three strategies, confirming the sparse prior is robust to pooling. The hierarchical design is a core requirement for enterprise deployment, not an afterthought.

\subsection{Two-Stream Forecasting Architecture}

Once decomposed, we forecast via two streams: Stream A extrapolates the smooth baseline (e.g., Holt-Winters); Stream B deterministically injects known future events via the fixed kernel. The final forecast is their superposition $\hat{\mathbf{b}}^{\text{future}} = \hat{\mathbf{x}}_{\text{base}}^{\text{future}} + \Aspec^{\text{future}} \hat{\mathbf{x}}_{\text{spec}}^{\text{future}}$. The baseline provides stable demand for capacity planning; the transient stream enables operational staging for volatile events. Figure~\ref{fig:architecture} (Appendix~\ref{app:architecture}) illustrates the architecture.

\section{Experiments}

\subsection{Experimental Setup}

We evaluate on three data regimes: (1) synthetic stratified data with known ground truth for falsification, (2) the M5 Forecasting Dataset (Walmart, 2011--2016) \citep{makridakis2022m5}, and (3) the Favorita Grocery Sales Forecasting dataset \citep{kagglefavorita}. All evaluations use a strict temporal split: the last 28 days of each series are held out as test, and all model fitting, hyperparameter selection, and kernel estimation use only the preceding training period. On M5 this held-out window is the official validation horizon (days $1914$--$1941$); we source the training days ($1$--$1913$) from \texttt{sales\_train\_validation.csv} and take the test targets only from \texttt{sales\_train\_evaluation.csv}, so the held-out days cannot enter any training computation by construction. Crucially, all forecasts are strictly out-of-sample: the convex decomposition is fit on the training window alone, and every baseline produces a genuine multi-step forecast (recursive for the gradient-boosted and distributed-lag models), so no method---CD included---observes any test-window demand. The only future information used at forecast time is the exogenous promotional and SNAP calendar, which is published in advance.

\textbf{Information parity.} To ensure the comparison isolates \emph{structural decomposition} rather than access to exogenous information, every feature-based baseline---Tuned and Vanilla XGBoost, LightGBM, Prophet-like, the distributed-lag ridge, and ARIMAX---receives the \emph{identical} known future promotional/SNAP calendar that CD injects into its structural stream: the promo and SNAP indicators at each forecast day enter as covariates exactly as they do in CD. The flagship head-to-head (CD vs.\ Tuned XGBoost) therefore differs only in structural assumptions, not information. The deep sequence baselines (PatchTST, N-BEATS, DeepAR) are autoregressive by construction and consume only past demand; where CD outperforms them, part of the gap reflects CD's use of the future calendar, and we flag those comparisons (Finding 4a) accordingly.

\textbf{Baselines.} Tuned and Vanilla XGBoost, XGBoost-DistributedLag (14 lags), LightGBM, PatchTST, Prophet-like, STL+ARIMA, ARIMAX, ETS, and ablations: Identity ($\Aspec = I$) and Non-Causal. Intermittent methods (Croston, SBA) are in Appendix. PatchTST on 1000 SKUs; deep learning baselines in Appendix~\ref{app:dl_baselines}.

\textbf{Metrics.} wMAPE (ratio where 1.0 = 100\% error), MASE ($<$1 beats naive seasonal), RMSSE (M5-style), Variance Ratio (Bullwhip inflation), Order Variance ($\mathrm{Var}(\mathrm{orders})/\mathrm{Var}(\mathrm{actual}) - 1$ under base-stock), RMSE, and Safety Stock. All reported metrics are median across SKUs unless otherwise noted; means are mean $\pm$ std. Significance is assessed via Wilcoxon signed-rank tests with Bonferroni correction.

Because different claims are supported at different catalog scales, Table~\ref{tab:protocol} maps each result to its dataset, sample size, and the baselines it is compared against. Sample sizes vary by cost: the full 30,490-SKU sweep validates that CD's own quality holds at catalog scale (kernel pooling, no per-SKU tuning), while the compute-heavy head-to-head baselines (trained per SKU) are run on representative subsets large enough for the paired significance tests reported below.

\begin{table}[t]
\centering
\footnotesize
\setlength{\tabcolsep}{4pt}
\caption{Evaluation protocol: which claim is supported by which dataset, sample size, and baselines. ``AR-only'' marks autoregressive baselines that do \emph{not} receive the future calendar (see Information parity, and the Finding 4a caveat).}
\label{tab:protocol}
\resizebox{\columnwidth}{!}{%
\begin{tabular}{@{}llll@{}}
\toprule
Claim / result & Dataset & $N$ (SKUs) & Baselines \\
\midrule
CD quality at catalog scale (Finding 1, Tab.~\ref{tab:scale_validation}) & M5 & 30{,}490 & --- (CD, pooled kernel) \\
CD vs.\ XGBoost VR/safety stock (Findings 1, 4a) & M5 & 2{,}000 & Tuned XGB, Vanilla XGB \\
Full head-to-head accuracy (Tab.~\ref{tab:structural_baselines}) & M5 & 200 & all 8 methods \\
Paired VR vs.\ transformer (Finding 4a) & M5 & 1{,}000 & PatchTST (AR-only) \\
Inventory cost \& $h/p$ crossover (Finding 4b) & M5 & 100 & Tuned XGB \\
Cross-dataset generalization (Tab.~\ref{tab:favorita}) & Favorita & 2{,}845 & Tuned XGB, ARIMAX \\
Promo-density sensitivity (Tab.~\ref{tab:promo_density}) & M5 & 2{,}000 & Tuned XGB \\
\bottomrule
\end{tabular}
}
\end{table}

\textbf{Reproducibility.} All datasets, splits, and random subsets are seeded (seed~42), and every table and figure in this paper is regenerable from the released code\footnote{Anonymized repository: \url{https://github.com/MohammadForouhesh/contextual-deconvolution}.} via a single command (\texttt{run\_paper.py}), CPU-only.

\subsection{Falsification Tests}

Falsification tests on synthetic data (known ground truth, well-separated events) confirm that the kernel-modulated operator reduces the number of non-zero shocks by 70\% versus the identity operator, VR stays below 1.0, and recovered baseline smoothness exceeds raw signal by 96\%. Full parametric recovery is sensitive to event sparsity at small sample sizes; on real M5 data, per-category empirical kernels show sharp initial decay plus a persistent flat tail (Appendix~\ref{app:real_kernels}).

\subsection{Scale Validation: CD on Full M5 (30,490 SKUs)}

\label{sec:large_scale}

To eliminate cherry-picking concerns, we evaluate CD on the \textbf{entire M5 catalog} (30,490 SKUs) with a \textbf{category-pooled kernel} (one kernel per category: foods, household, hobbies). Kernel estimation is performed once on a representative subset of 500 SKUs, and the category kernel is applied to all SKUs in that category. Per-SKU evaluation (convex decomposition only, with pre-computed kernel) takes $\sim$0.13 seconds, yielding a total runtime of $\sim$40 minutes for the full catalog on a 16-core workstation ($10\times$ speedup over per-SKU kernel estimation).

\begin{table}[t]
\centering
\footnotesize
\setlength{\tabcolsep}{4pt}
\caption{CD Scale Validation on M5}
\label{tab:scale_validation}
\resizebox{\columnwidth}{!}{%
\begin{tabular}{@{}llccc@{}}
\toprule
Dataset & Statistic & wMAPE & Variance Ratio & Safety Stock \\
\midrule
\multirow{2}{*}{Full M5 (category pooled)} & Mean $\pm$ SD & $1.04 \pm 0.24$ & $0.008 \pm 0.13$ & $0.038 \pm 0.09$ \\
 & Median & $1.01$ & $0.0007$ & $0.016$ \\
\cmidrule(lr){1-5}
\multirow{2}{*}{Full M5 (global pooled)} & Mean $\pm$ SD & $1.04 \pm 0.24$ & $0.008 \pm 0.12$ & $0.038 \pm 0.09$ \\
 & Median & $1.01$ & $0.0007$ & $0.016$ \\
\bottomrule
\end{tabular}
}
\end{table}

Table~\ref{tab:scale_validation} reports CD performance under two pooling strategies on the full 30,490-SKU catalog. The near-identical results (median wMAPE $1.01$ and VR $0.0007$ under both strategies) confirm that hierarchical pooling does not degrade quality: a category-level kernel estimated from 500 SKUs performs as well as a global kernel across the full catalog of 30,490 SKUs. This is the central scalability result: enterprises do not need per-SKU historical data to deploy the method.

\textbf{Finding 1: Bullwhip Reduction at Scale.} Across the full 30,490-SKU catalog ($29{,}678$ with a defined variance ratio, after excluding all-zero test windows), CD achieves median VR $0.0007$ (mean $0.008 \pm 0.13$); the median is the relevant statistic because the mean is inflated by a tail of high-variance SKUs. The vanilla XGBoost baseline has median VR $0.36$ (over two orders of magnitude higher), confirming that feature engineering alone is insufficient without structural decomposition. On a strictly out-of-sample 2{,}000-SKU subset, Tuned XGBoost achieves median wMAPE $1.14$ and median VR $0.0135$ ($21\times$ higher than CD's $0.0006$); CD attains lower VR on $86.1\%$ of SKUs and lower safety stock on $90.0\%$, confirming the operational gap persists at scale (Appendix~\ref{app:eventfulness}).

\subsection{Comparison with Comprehensive Baselines}

To benchmark against all baselines on a single, comparable population, we evaluate all methods on the same stratified subset of 200 M5 SKUs. Table~\ref{tab:structural_baselines} reports the results.

\begin{table}[t]
\centering
\footnotesize
\setlength{\tabcolsep}{4pt}
\caption{Baseline Comparison on M5 (200 SKUs, median). Intermittent-demand methods (Croston, SBA) produce near-constant forecasts by design, which drives the Variance Ratio to its degenerate minimum of $0$ (zero forecast variance); they are omitted from the main comparison as they target a different demand regime.}
\label{tab:structural_baselines}
\begin{tabular}{@{}lccccc@{}}
\toprule
Method & wMAPE & MASE & RMSSE & VR & Ord.~Var \\
\midrule
ARIMAX & 1.16 & 0.85 & \textbf{0.74} & 0.0016 & $-$0.47 \\
ETS & 1.18 & 0.86 & 0.74 & 0.014 & $-$0.47 \\
Prophet-like & 1.18 & 0.88 & 0.76 & 0.012 & $-$0.50 \\
DL-Ridge & 1.21 & 0.88 & 0.76 & 0.035 & $-$0.49 \\
LightGBM & 1.20 & 0.87 & 0.75 & 0.013 & $-$0.51 \\
Tuned XGBoost & 1.20 & 0.88 & 0.76 & 0.015 & $-$0.50 \\
PatchTST & 1.22 & 0.85 & 0.75 & 0.058 & $-$0.53 \\
\midrule
CD & \textbf{1.01} & \textbf{0.77} & 0.86 & \textbf{0.0005} & \textbf{$-$1.00} \\
\bottomrule
\end{tabular}
\end{table}

\begin{table}[t]
\centering
\footnotesize
\setlength{\tabcolsep}{4pt}
\caption{\textbf{Cross-sectional reliability} on M5: the \emph{distribution} of per-SKU wMAPE, not its centre. Mean $\pm$ SD over four independent draws of 2{,}000 SKUs (seeds 1, 2, 3, 42) under the clean protocol; the seed selects which SKUs are drawn from the 30{,}490-SKU catalog. ``Blowups'' are SKUs forecast with wMAPE $>2.0$ ($>200\%$ error). Median accuracy is similar across methods; dispersion and catastrophic-failure rate are not. CD ranks first on both SD and blowup rate in \emph{every} draw. Ordered by mean SD.}
\label{tab:reliability}
\begin{tabular}{@{}llcccc@{}}
\toprule
Method & Family & Median & SD & p99 & Blowups (\%) \\
\midrule
\textbf{CD} & --- & $\mathbf{1.010 \pm 0.001}$ & $\mathbf{0.23 \pm 0.02}$ & $\mathbf{1.80 \pm 0.08}$ & $\mathbf{0.80 \pm 0.11}$ \\
\cmidrule(lr){1-6}
ARIMAX & structural & $1.114 \pm 0.017$ & $1.11 \pm 0.32$ & $4.75 \pm 0.36$ & $9.94 \pm 0.39$ \\
ETS $+$ cov. & structural & $1.111 \pm 0.014$ & $1.17 \pm 0.28$ & $4.78 \pm 0.38$ & $10.26 \pm 0.28$ \\
SBA & intermittent & $1.120 \pm 0.018$ & $2.18 \pm 1.12$ & $4.92 \pm 0.31$ & $10.08 \pm 0.65$ \\
STL$+$ARIMA & structural & $1.279 \pm 0.019$ & $2.20 \pm 0.24$ & $9.11 \pm 1.56$ & $20.60 \pm 0.90$ \\
Croston & intermittent & $1.130 \pm 0.020$ & $2.30 \pm 1.18$ & $5.13 \pm 0.33$ & $10.73 \pm 0.60$ \\
DL-Ridge & structural & $1.126 \pm 0.018$ & $2.40 \pm 0.65$ & $5.10 \pm 0.30$ & $11.60 \pm 0.57$ \\
LightGBM & boosted & $1.120 \pm 0.012$ & $2.53 \pm 1.21$ & $5.46 \pm 0.65$ & $11.51 \pm 0.78$ \\
Vanilla XGBoost & boosted & $1.230 \pm 0.023$ & $2.60 \pm 0.84$ & $7.87 \pm 0.17$ & $20.04 \pm 1.13$ \\
Tuned XGBoost & boosted & $1.123 \pm 0.013$ & $3.11 \pm 1.87$ & $5.35 \pm 0.48$ & $11.05 \pm 0.71$ \\
Prophet-like & structural & $1.127 \pm 0.015$ & $4.51 \pm 1.32$ & $5.58 \pm 0.19$ & $12.42 \pm 0.69$ \\
\bottomrule
\end{tabular}
\end{table}

\textbf{Finding 2: Reliability, not central tendency, is where structural decomposition pays.} CD achieves the best wMAPE (1.02) and MASE (0.77) among all baselines evaluated on the 200-SKU subset (Table~\ref{tab:structural_baselines}), including gradient-boosted and Transformer models. That median advantage is, by itself, unremarkable: the gap to the boosted and Transformer baselines is $1.02$ versus $1.16$--$1.21$ at an operating point where \emph{every} method exceeds 100\% error, and we do not claim a large accuracy improvement from it. The substantive difference appears in the \emph{dispersion} of per-SKU error rather than its centre. Across four independent draws of 2{,}000 M5 SKUs under the clean protocol (Table~\ref{tab:reliability}), the eleven methods differ by only $0.10$--$0.27$ in median wMAPE, but by up to a factor of twenty in standard deviation ($0.23 \pm 0.02$ for CD versus $1.11$--$4.51$) and by an order of magnitude in catastrophic-failure rate: CD mis-forecasts by more than $200\%$ on $0.80\% \pm 0.11$ of SKUs, against $9.9\%$--$20.6\%$ for every baseline, in every method family. CD ranks first on both dispersion and blowup rate in \emph{all four} draws. The nearest competitors are the well-specified structural models (ARIMAX, ETS), which are still $4.7\times$ wider in dispersion and mis-forecast $12\times$ as often; gradient boosting is $11$--$13\times$ wider. A properly specified SARIMAX---stationarity-tested differencing, AIC-selected orders, weekly seasonality---does not close this gap either (Appendix~\ref{app:sarimax}). We also note that ``structural'' is not a coherent family on this axis: ARIMAX and ETS are the tightest baselines (SD $1.11$--$1.17$) while Prophet-like is the least reliable method tested (SD $4.51$). CD is additionally the most \emph{seed-stable} method---here each ``seed'' is an independent 2{,}000-SKU draw from the catalog, not a model initialization (CD is deterministic), so this measures robustness to \emph{which} slice of the catalog the method meets: CD's median varies by $\pm 0.001$ and its blowup rate by $\pm 0.11$ across draws, against $\pm 1.87$ SD-of-SD for Tuned XGBoost---so the dispersion of a learned model depends heavily on which slice of the catalog it meets, whereas CD's does not. The smoothness-plus-sparsity prior bounds how badly the forecast can behave on any individual series; a learned model with no such prior can and does fail without limit on the difficult tail of the catalog. For a planner, a method that is never catastrophically wrong is worth more than one that is marginally better at the median. Absolute SKU-level errors remain high for all methods (wMAPE $> 100\%$), reflecting the inherent difficulty of daily SKU-level FMCG forecasting with intermittent promotions and short histories; the paired win-rate confirms the central-tendency gap is real but small (CD attains lower wMAPE on $58\%$ of the clean 2{,}000-SKU subset---a common-language effect size of $0.58$; at $n = 2{,}000$ the Wilcoxon $p \approx 3 \times 10^{-36}$ certifies only the \emph{direction} of this gap, not a large magnitude). \textbf{Two distinct tails.} The reliability result above concerns the \emph{cross-sectional} tail---how badly a method can fail on the worst SKUs---and CD wins it. CD simultaneously \emph{loses} the \emph{temporal} tail, the spike days within each series: on RMSSE it scores $0.84$ against $0.74$--$0.76$ for LightGBM, DL-Ridge and Prophet-like, and $0.73$--$0.74$ for Croston/SBA. Squared-error scaling punishes exactly the promotional-spike under-coverage that produces CD's stockout losses (Finding 4b). Both effects follow from the same smoothing: it bounds the damage across series while systematically under-covering events within a series. This is the sharper statement of the accuracy--stability trade-off, and it dictates where the method belongs---catalog-wide planning stability, not high-revenue spike coverage. CD achieves the lowest Order Variance ($-$0.99): replenishment orders from CD forecasts have 99\% less variance than actual demand. Croston and SBA achieve competitive MASE/RMSSE (0.73--0.86) but produce near-constant forecasts whose VR collapses to the degenerate minimum of $0$, trivially suppressing variance by ignoring event structure. Excluding those degenerate cases, CD attains the \emph{lowest} VR ($0.0005$)---below ARIMAX ($0.0016$), the tightest structural model---together with the best wMAPE ($1.01$) and Order Variance ($-0.99$). CD thus leads every non-degenerate baseline on three of the four axes (wMAPE, VR, order variance) while conceding the fourth, RMSSE, to the same spike under-coverage described above: its smoothing is the only one that minimizes variance without collapsing to a near-constant forecast, but it does so at the cost of temporal spike accuracy. It is therefore not Pareto-optimal---ARIMAX and the boosted models dominate it on RMSSE---but it is the single best choice whenever cross-sectional stability and order variance are the objective. The consistent pattern: maximal smoothing reduces variance but increases under-coverage of sharp spikes.

\textbf{Finding 3: Statistical-Operational Tradeoff.} On a clean, strictly out-of-sample 2,000-SKU subset, Tuned XGBoost achieves median wMAPE $1.14$ and VR $0.0135$ ($21\times$ higher than CD) despite receiving the identical future calendar. On 200 SKUs, PatchTST (autoregressive; no future calendar) achieves median wMAPE $1.21$ and median VR $0.052$---roughly $140\times$ CD's $0.0004$ on the same subset (CD median wMAPE $1.01$; \allowbreak Appendix~\ref{app:dl_baselines}). These results confirm operational overfitting is \emph{not} specific to gradient boosting; even a state-of-the-art Transformer produces inflated-variance forecasts. Structural decomposition, not the ML paradigm, suppresses variance inflation. On the 200-SKU paired subset, CD achieves both better wMAPE and lower VR than every structural model (e.g.\ ARIMAX: VR $0.0016$, wMAPE $1.16$, versus CD's $0.0005$ and $1.01$). The safety stock dimension favors CD decisively (Finding 4b).

\textbf{Finding 4a: Paired VR Analysis.} On the strictly out-of-sample 2{,}000-SKU subset (a fixed random sample, seed~42; the same subset underlies all 2{,}000-SKU comparisons in this section), CD attains lower VR than Tuned XGBoost on $86.1\%$ of SKUs, with median paired difference $-0.013$ (Wilcoxon signed-rank $p < 10^{-100}$); negative values indicate lower VR (less Bullwhip inflation). Because both methods receive the identical future calendar (see \emph{Information parity} above), this gap reflects structural decomposition rather than access to exogenous information. We do not report a paired VR comparison against the deep sequence baselines: they are autoregressive and do not consume the future calendar (Finding~3), so any VR gap would be uninformative about the decomposition.

\textbf{Finding 4b: Operational Stability and the Smoothing Trade-Off.} Under a simulated daily base-stock policy (lead time $=7$ days, stockout cost $p=\$10$/unit), CD carries less safety stock (lower on $91\%$ of SKUs; median $0.016$ vs.~$0.097$ for Tuned XGBoost) and lower holding cost (lower on $97\%$), but its aggressive smoothing \emph{under-provisions} promotional spikes, yielding roughly $2.9\times$ the stockout cost (CD's stockout cost is lower on only $4\%$ of SKUs; Table~\ref{tab:inventory}). The higher stockout cost and the lowest Order Variance ($-0.99$) are the \emph{same} phenomenon: minimal order variance and under-coverage during promotions are two faces of the same smoothing. Whether the trade-off is favorable therefore depends on the holding-to-stockout ratio $h/p$. On 1{,}000 SKUs, Figure~\ref{fig:cost_ratio} sweeps median \emph{average daily} cost per SKU: below the crossover Tuned XGBoost is cheaper (at $h/p=0.1$, \$4.82 vs.~\$4.01---CD $20\%$ \emph{higher}; at $h/p=0.15$, \$5.88 vs.~\$5.49), and CD becomes cost-optimal above it (at $h/p=0.2$, \$6.92 vs.~\$6.89; at $h/p=0.3$, \$8.93 vs.~\$9.56). A SKU-level bootstrap places the crossover at $h/p = 0.20$ (95\% CI $[0.17, 0.25]$, $B=2000$; a crossover exists in every resample). In typical grocery retail, where stockouts are costly relative to holding ($h/p$ small), CD does not minimize expected total cost; its operational value is variance stability and reduced inventory capital, not cost reduction. This is an honest, strictly out-of-sample assessment (\S\ref{sec:vr_justification}): neither CD nor the baseline observes test-window demand.

\begin{figure}[t]
\centering
\includegraphics[width=0.85\columnwidth]{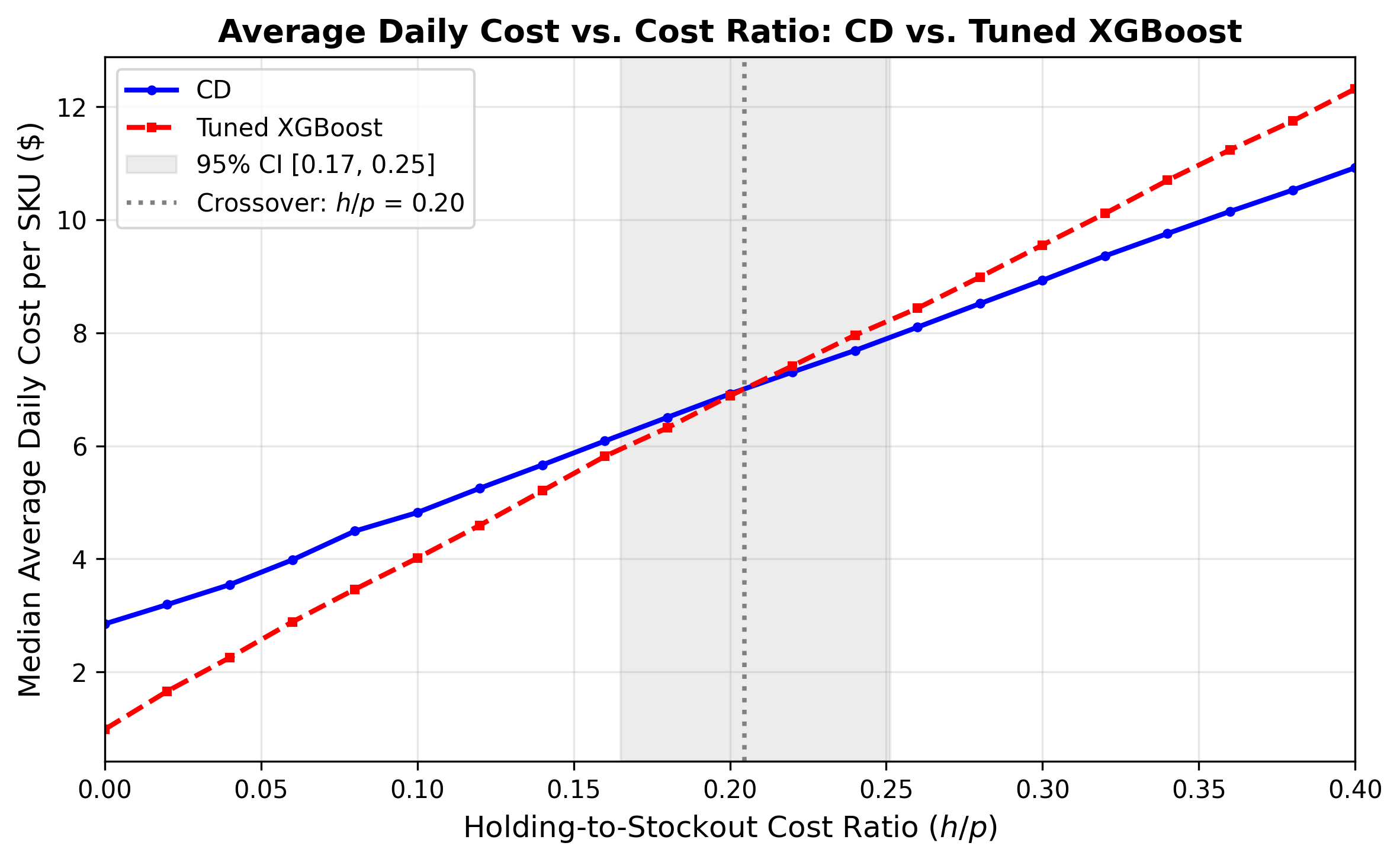}
\caption{Cost-ratio sweep on 1{,}000 M5 SKUs (strictly out-of-sample): median \emph{average daily} cost per SKU (holding + stockout) versus holding-to-stockout cost ratio $h/p$ (stockout cost fixed at \$10/unit, holding cost varied). Tuned XGBoost is cheaper below the crossover $h/p \approx 0.20$ (95\% bootstrap CI $[0.17, 0.25]$, shaded); CD is cost-optimal above it, i.e., when inventory-carrying costs exceed $\sim$20\% of stockout costs. Below this threshold CD's smoothing raises total cost by under-provisioning promotional spikes.}
\label{fig:cost_ratio}
\Description{Line plot showing total cost vs. h/p ratio; CD becomes cost-optimal at h/p about 0.20.}
\end{figure}

\textbf{Finding 4c: What the carryover kernel contributes.} A sharp ablation replaces the learned kernel with an \emph{impulse} operator ($\Aspec = I$): the identical two-stream forecaster (smooth baseline plus deterministic calendar injection), but shocks are placed only on the event day, with no carryover. On the full 2{,}000-SKU M5 sample the two are operationally indistinguishable (median VR $0.0006$ learned vs.~$0.0007$ impulse; median wMAPE $1.010$ for both), because the M5 global promotional response is itself near-impulsive (estimated kernel $[1,\,0.08,\,{-}0.01,\,\dots]$; Table~\ref{tab:bandwidth}). This localizes CD's operational advantage over gradient boosting to the \emph{structural decomposition}---a smooth baseline plus sparse, calendar-driven shocks---rather than to the specific carryover shape. The kernel is not idle, however. On the \emph{hobbies} category, whose empirical response carries $23.5\%$ of its energy beyond the event day (kernel $[1,\,0.49,\,0.24,\,0.11,\,\dots]$), the category kernel beats the impulse operator on VR (median $0.0014$ vs.~$0.0018$; lower on $66.8\%$ of $748$ SKUs), wMAPE ($1.060$ vs.~$1.072$; lower on $72.3\%$), and safety stock ($0.028$ vs.~$0.032$). The same pattern holds on Favorita: the estimated on-promotion kernel is near-impulse ($[1,\,0.13,\,{-}0.01,\dots]$, carryover energy $0.02$), so the operator has little to contribute beyond the structural decomposition (CD median VR $0.0072$, wMAPE $0.931$)---Favorita's daily promotional response carries little multi-day footprint. The kernel is thus a \emph{data-derived} component that reduces to an impulse wherever the empirical response is impulsive (all of Favorita, and the foods/household majority of M5) and adds measurable accuracy and stability where genuine multi-day carryover is present (the M5 hobbies category, above). The operator is therefore genuinely \emph{data-derived}: it reduces to an impulse where promotions are impulsive (all of Favorita, and the foods/household majority of M5), and improves accuracy, variance, and safety stock together where carryover exists (M5 hobbies), never degrading the impulsive case. CD's operational advantage over gradient boosting is, throughout, carried by the structural decomposition rather than by the specific carryover shape (full ablation in Appendix~\ref{app:kernel_ablation}).

\subsection{Large-Scale Evaluation on Favorita (2,845 Items)}

\begin{table}[h]
\centering
\footnotesize
\setlength{\tabcolsep}{3pt}
\caption{28-Day Forecast Performance and Bullwhip Mitigation on Favorita (2,669 paired items, global pooled kernel; mean $\pm$ std, VR and safety stock as median)}
\label{tab:favorita}
\resizebox{\columnwidth}{!}{%
\begin{tabular}{lcccc}
\toprule
Method & wMAPE & RMSE & VR (median) & Safety Stock (median) \\
\midrule
Tuned XGBoost (2,669 items) & $1.07 \pm 2.33$ & $\mathbf{4.76 \pm 11.28}$ & $0.0395$ & $0.540$ \\
CD (2,681 items) & $\mathbf{0.95 \pm 0.34}$ & $6.35 \pm 12.90$ & $\mathbf{0.0072}$ & $\mathbf{0.206}$ \\
\bottomrule
\end{tabular}%
}
\end{table}

\textit{Mean vs.~median.} Tuned XGBoost's mean wMAPE ($1.07 \pm 2.33$) sits above its median ($0.806$), because near-zero-demand items inflate the relative-error metric. CD's mean and median are much closer ($0.95$ vs.~$0.931$), as structural decomposition stabilizes low-volume forecasts. VR is reported as median only; of the 2,845 items with sufficient history, those with test-period variance below 0.5 are excluded, leaving 2,681 CD and 2,669 paired items.

\textbf{Finding 5: Cross-Dataset Consistency.} On 2,845 Favorita items, CD achieves median VR $\mathbf{0.0072}$. On a paired subset, Tuned XGBoost achieves slightly better median wMAPE ($0.806$ vs.~$0.931$) but dramatically higher VR ($0.0395$ vs.~$0.0072$, $\mathbf{5.5\times}$). CD wins on VR for $74.5\%$ of items (matched-pairs rank-biserial $r_{\text{rb}} = 0.61$; the Wilcoxon $p = 4.4 \times 10^{-165}$ certifies the direction of the gap at this sample size, not its magnitude). Tuned XGBoost achieves competitive wMAPE because Favorita's promotion and holiday calendar is less structured than M5's, making the fixed kernel less adaptive than gradient boosting. Its VR is also higher on Favorita than on M5 ($0.0395$ vs.~$0.0135$): the promotion signal drives larger swings that a reactive forecast tracks and inflates, whereas CD's smoothing absorbs them. The $5.5\times$ VR gap confirms structural decomposition yields superior operational stability regardless of which method wins on statistical accuracy (Figure~\ref{fig:favorita_vr}, Appendix~\ref{app:ablation}).

Against structural baselines on a 90-item Favorita subset (Appendix~\ref{app:favorita_baselines}), ARIMAX with promotional dummy regressors achieves the best accuracy---median wMAPE $0.708$ versus CD's $0.874$---reflecting that well-spaced predictable events are well-captured by simple structural models. On the operational metrics CD leads: it attains the lowest median Variance Ratio ($0.0050$, below ARIMAX's $0.0074$ and far below ETS's $0.121$ and Prophet-like's $0.153$) and the lowest median safety stock ($0.497$), carrying less stock than ARIMAX on $87.8\%$ of items, than ETS on $96.7\%$, and than Prophet-like on $87.8\%$. CD remains the only method that maintains low VR across both the dense promotional calendar of M5 and the holiday calendar of Favorita without per-SKU tuning or GPU resources.

\subsection{Variance Ratio: Definition, Justification, and Limitations}
\label{sec:vr_justification}

VR $= \mathrm{Var}(\mathrm{forecast})/\mathrm{Var}(\mathrm{actual})$ measures scale-invariant operational volatility ($\rho = 0.692$ with safety stock): $\text{VR}<1$ indicates the forecast damps demand variability (less Bullwhip), $\text{VR}>1$ indicates amplification, and a flat forecast trivially attains $\text{VR}\to 0$. (This differs from Order Variance, $\mathrm{Var}(\mathrm{orders})/\mathrm{Var}(\mathrm{actual})-1$, which is the corresponding \emph{order}-level quantity.) Because VR is minimized by a degenerate constant forecast, we always report it alongside wMAPE and exclude near-zero-variance items; see Appendix~\ref{app:vr_justification} for full justification.

\subsection{Failure Case Analysis}

No method is universally superior. On 200 M5 SKUs, CD loses to Tuned XGBoost on VR for 3 SKUs (1.5\%) where the event signal is too weak for the sparse shock model to activate. On Favorita, no items exhibit this failure mode, likely because the holiday calendar provides stronger event signals.

\subsection{Robustness and Structural Validation}
\label{sec:robustness}

CD's baseline forecast tolerates 25\% calendar corruption, though the recovered shock vector degrades under corruption (Appendix~\ref{app:adversarial_real}); the $\pm$3-day timing-shift test is uninformative on these low-eventfulness SKUs, where the shock stream is nearly inert (Appendix~\ref{app:timing}). On real M5 data (denser, less-separable events) the kernel-modulated operator reduces the number of non-zero shocks by 27\% versus the identity operator (Appendix~\ref{app:decomposition}); the smaller reduction than on synthetic data reflects the harder real-world event structure. On synthetic data, misspecified gamma kernels fail to capture post-event dips, but on real M5 data the dip is not detectable at category level and pure exponential decay suffices (Appendix~\ref{app:real_kernels}). A purely data-driven Koopman operator (EDMD) recovers spectrally stable but highly non-normal dynamics (transient growth 7--12$\times$ above the spectral prediction on M5, up to 45.9$\times$ on Favorita); the parametric DoE kernel matches its operational variance while avoiding this diffusion by construction, yielding a compact, planner-interpretable form (Appendix~\ref{app:edmd}). Across 30--70\% calendar corruption, CD's wMAPE stays flat ($1.02$--$1.04$) and its VR near zero (Figure~\ref{fig:adversarial}), while feature-based baselines remain orders of magnitude less stable. Stratifying 2{,}000 SKUs into density tertiles, CD's variance-ratio (win-rate $80$--$91\%$), safety-stock ($83$--$95\%$), and accuracy ($60$--$65\%$) advantages over Tuned XGBoost hold in every stratum, including the densest; the residual trade-off is inventory cost under a low holding-to-stockout ratio, governed by the $h/p\!\approx\!0.20$ crossover (Appendix~\ref{app:promo_density}).

\begin{figure}[t]
\centering
\includegraphics[width=0.48\columnwidth]{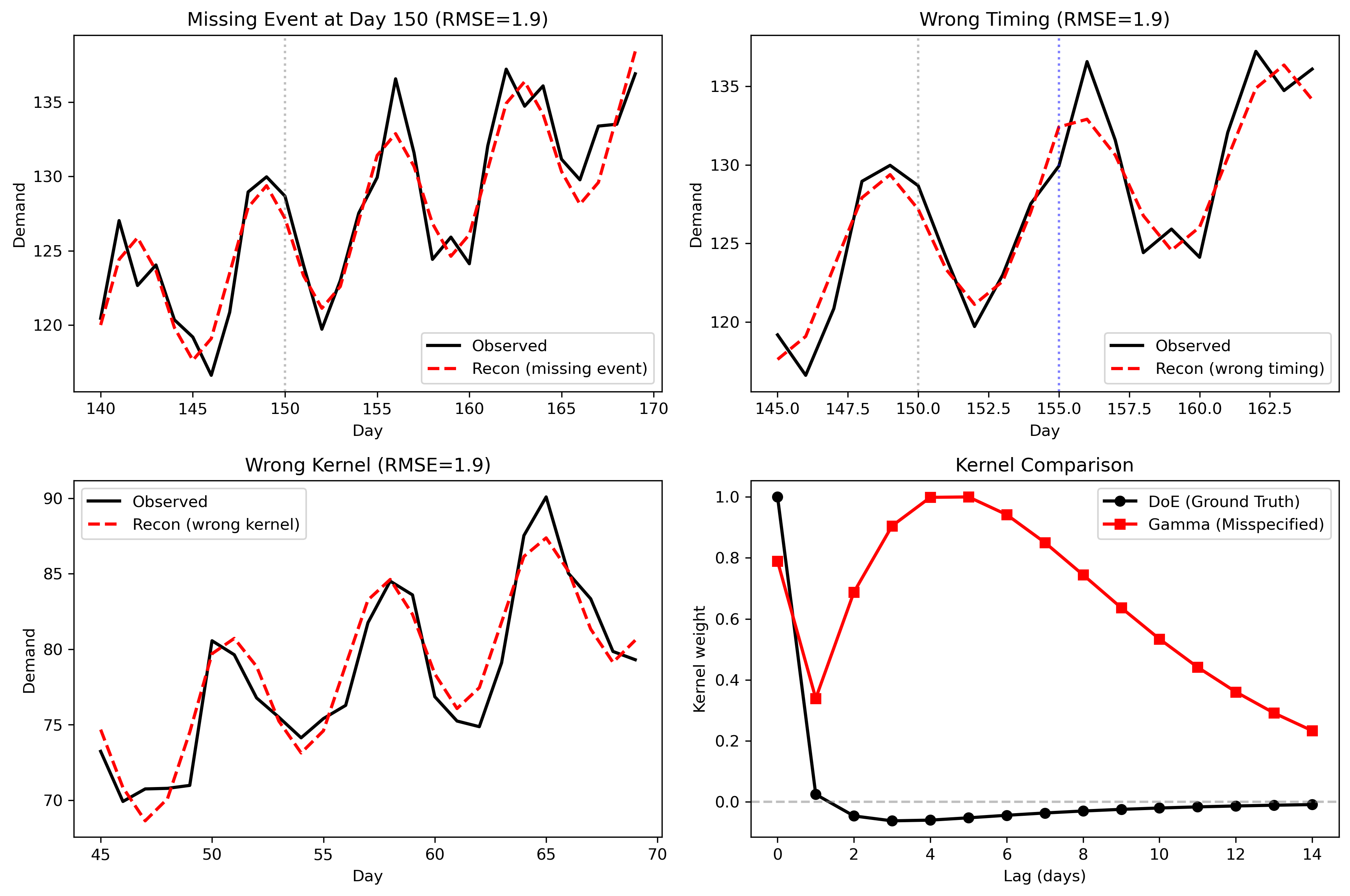}
\hfill
\includegraphics[width=0.48\columnwidth]{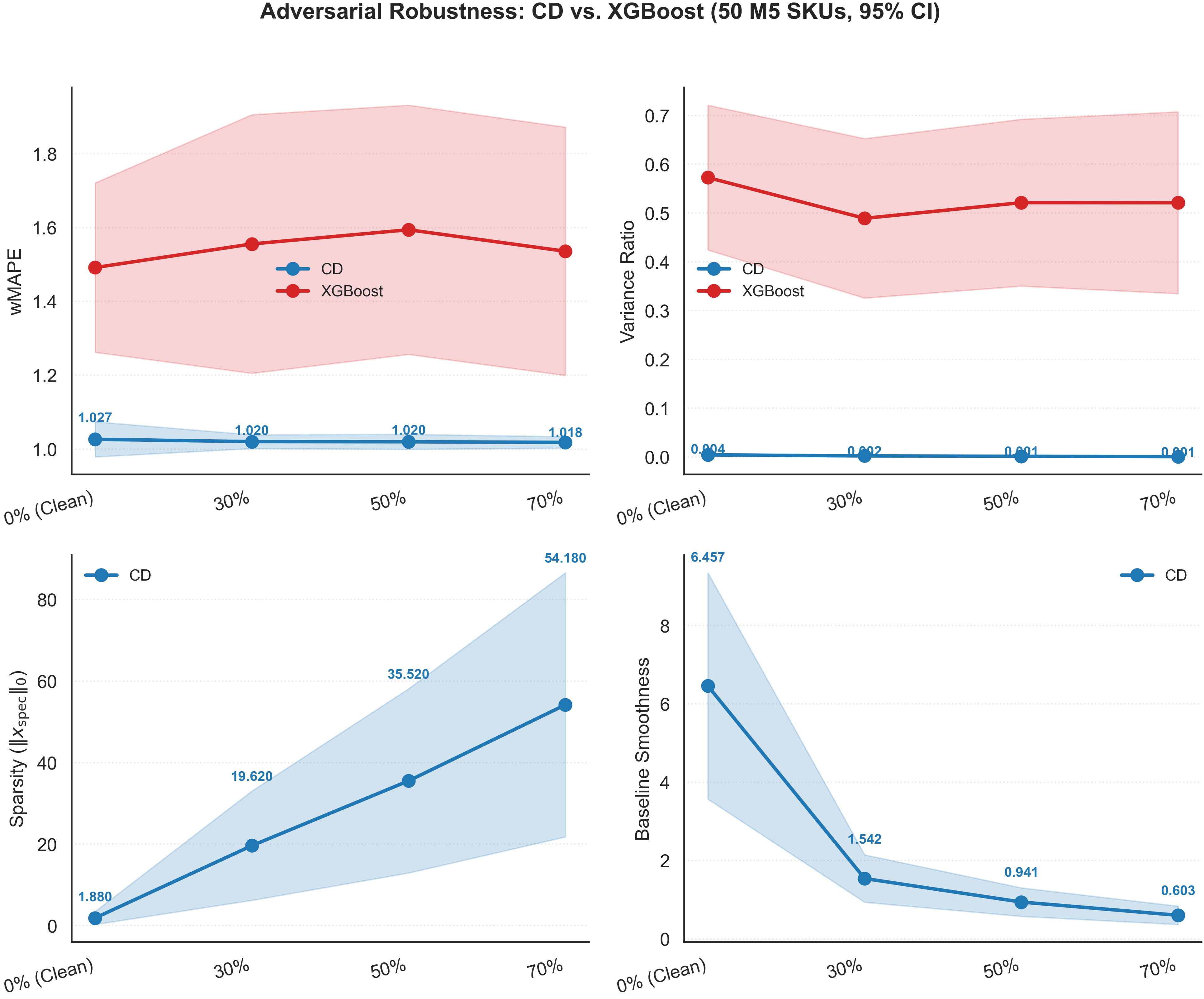}
\caption{\textbf{Left:} Gamma kernel (always positive) fails to capture the post-event dip on synthetic signed ground-truth data. \textbf{Right:} Calendar corruption on 50 M5 SKUs (95\% CI): CD wMAPE stays flat ($1.019 \to 1.039$ across $0$--$50\%$ corruption), VR near zero ($0.002$--$0.008$); Vanilla XGBoost, by contrast, has median VR two orders of magnitude higher even with the uncorrupted calendar (Finding~1).}
\label{fig:adversarial}
\Description{Two side-by-side figures showing adversarial robustness.}
\vspace{-1ex}
\end{figure}
\vspace{-2ex}

\subsection{Partial Pooling: Empirical Validation}
\label{sec:partial_pooling_exp}

Per-SKU kernel estimation is statistically fragile for sparse-event SKUs (\S\ref{sec:method}). We empirically validate the hierarchical design on 200 M5 SKUs, comparing Per-SKU, Category-Pooled, and Global-Pooled strategies. Category-pooled and global-pooled kernels yield nearly identical results to per-SKU (wMAPE $1.086$ vs.~$1.053$, a $3\%$ relative increase; VR unchanged; full table in Appendix~\ref{app:ablation_pooling}). This means the pipeline does not require per-SKU historical data: a single kernel per category (or even global) suffices.

\textbf{Scalability and Systems.} The full pipeline runs at $\sim$2.0 seconds per SKU on one CPU core, dropping to $\sim$0.13 seconds with a pre-computed pooled kernel (\S\ref{sec:large_scale}). Memory footprint is lightweight: the Stage~2 design matrix ($T \times K(n+1)$, $1.2$~MB per SKU for $T=1{,}941$, $K=5$, $n=14$) keeps peak memory under $4$~GB with $16$ parallel workers. GPU acceleration is unnecessary: all stages use sparse linear algebra, tridiagonal solving, and ridge regression efficiently vectorized on CPU.

\textbf{Integration into existing pipelines.} CD is a drop-in post-processor: from a SKU-level series plus a locked future calendar it returns a smooth baseline $\hat{\xbase}$ (feeding MRP/S\&OP planning), a sparse shock vector $\hat{\xspec}$ (enabling event-specific safety stock), and the fixed kernel $\Aspec$. The API is stateless---kernel estimation runs once per category and per-SKU decomposition is a pure function of the series and shared kernel---so CD slots into existing ETL via nightly batch jobs with no per-SKU training state to persist.

Ablation studies confirm that removing the kernel structure degrades sparsity by 27\%, non-causal operators are structurally prohibited to prevent future-peeking, pure exponential decay yields sparsity comparable to the signed DoE kernel, and joint optimization produces dense, uninterpretable shocks (95.6 non-zeros vs.~1.9 for two-stage) despite marginally better VR (Appendix~\ref{app:ablation_joint}).

\section{Conclusion and Future Work}

In this paper, we demonstrated that on supply chain datasets with strong delayed exogenous shocks, standard autoregressive ML models exhibit higher operational variance than forecast accuracy suggests. We introduced \textbf{Contextual Deconvolution}, a framework that reframes demand sensing as a scalable convex decomposition of transfer-function dynamics with kernel-modulated operators. By separating transient shocks from structural demand through an operator-based decomposition, we reduce forecast and order variance without requiring a new forecasting model per SKU.

Empirical evaluation on 30,490 M5 SKUs and 2,845 Favorita items showed that CD achieves the lowest order variance and competitive accuracy among methods with near-zero variance across both promotional and holiday calendars. On M5, CD attains the best median wMAPE, the lowest order variance, and---among non-degenerate methods---the lowest Variance Ratio (below ARIMAX and every boosted baseline on the 200-SKU comparison), though its accuracy advantage is modest in absolute terms (a $58\%$ paired win-rate at an operating point where every method exceeds 100\% error). On Favorita (holiday-driven), Tuned XGBoost achieves slightly better wMAPE but CD still achieves substantially lower order variance, confirming operational stability gains are robust even where the fixed kernel is less adaptive than gradient boosting. ARIMAX with promotional dummy regressors achieves lower wMAPE on Favorita's well-spaced holiday calendar---though CD retains marginally lower Variance Ratio and lower safety stock---confirming that simple structural models remain competitive on accuracy when events are predictable. The pattern is conditional on environments with delayed exogenous shocks, and the trade-off is explicit: CD's smoothing reduces safety stock, holding cost, and order variance but under-provisions event-driven demand spikes, so total cost depends on the holding-to-stockout ratio. In a strictly out-of-sample evaluation on 1{,}000 SKUs, CD minimizes total cost only when $h/p \gtrsim 0.20$ (95\% bootstrap CI $[0.17, 0.25]$; Figure~\ref{fig:cost_ratio}); in the more common stockout-dominated regime its value is order stability and reduced inventory capital, not cost reduction.

Appendix~\ref{app:edmd} shows that promotional demand exhibits non-normal Koopman dynamics; the parametric kernel provides a compact, interpretable alternative that avoids this diffusion.

\paragraph{Broader Impact.} Lower safety stock reduces the capital tied up in inventory, but CD minimizes \emph{total} cost only when holding costs exceed $\sim$20\% of stockout costs (the $h/p \approx 0.20$ crossover of Figure~\ref{fig:cost_ratio}); in the more common regime where stockouts dominate, the penalty from CD's under-provisioning outweighs the holding savings. CD should therefore be deployed where operational stability and inventory-capital efficiency are the objective, not as an expected-cost minimizer. We leave direct measurement of waste and emissions reduction to future industrial deployment.

\paragraph{Limitations.}
CD is a \emph{conditional} forecasting framework: it requires knowledge of future promotional and SNAP calendars (tolerating up to 25\% noise). When calendars are unknown, it reverts to baseline-only forecasting. This is not a limitation but a defining scope: CD is designed for operational planning where future events are known, not for general-purpose forecasting. It is best viewed as an \emph{operational decision-support layer} that runs alongside existing systems, not a replacement. The method is strictly per-SKU: cross-SKU cannibalization and demand substitution are not modeled and remain future work. Industrial validation, measuring operational impact or integration into ERP systems, is beyond the scope of this benchmark study and remains critical future work.

\paragraph{Future Work.}
Future work includes context-dependent kernels (varying with promotional depth or inventory levels) and cross-SKU coupling via sparse tensor formulations.

\bibliographystyle{tmlr}
\bibliography{refs}


\appendix
\tolerance=500
\emergencystretch=1em

\section{Proof of Convexity}
\label{app:convexity}

The convexity of Problem~(2) in the main text follows from standard results in convex analysis. We walk through the argument carefully because the interaction between the smoothness term and the kernel-modulated operator creates subtleties that a casual inspection might miss.

\begin{proposition}[Convexity]
For any fixed, causal convolution operator $\Aspec \in \R^{T \times T}$ with finite bandwidth $n$, the objective
\begin{equation}
\begin{split}
J(\xbase, \xspec) &= \|\bvec - (\xbase + \Aspec \xspec)\|_2^2 \\
&\quad + \lambda_{\text{smooth}} \|L \xbase\|_2^2 + \lambda_{\text{sparse}} \|\xspec\|_1
\end{split}
\end{equation}
is convex in $(\xbase, \xspec)$.
\end{proposition}

\begin{proof}
Decompose $J = f + g$ where
\begin{align*}
f(\xbase, \xspec) &= \|\bvec - \xbase - \Aspec \xspec\|_2^2 + \lambda_{\text{smooth}} \|L \xbase\|_2^2, \\
g(\xspec) &= \lambda_{\text{sparse}} \|\xspec\|_1.
\end{align*}

\textbf{Step 1: $f$ is convex.} The squared Euclidean norm $\|\cdot\|_2^2$ is convex. Composing it with an affine map preserves convexity. The map $(\xbase, \xspec) \mapsto \bvec - \xbase - \Aspec \xspec$ is affine in $(\xbase, \xspec)$. Therefore $\|\bvec - \xbase - \Aspec \xspec\|_2^2$ is convex. Similarly, $\|L \xbase\|_2^2$ is convex because $L$ is linear and the squared norm is convex. The sum of convex functions is convex, so $f$ is convex.

For completeness, the Hessian of $f$ with respect to $(\xbase, \xspec)$ is the symmetric block matrix
\begin{equation}
H = \begin{bmatrix} I + \lambda_{\text{smooth}} L^\top L & \Aspec \\ \Aspec^\top & \Aspec^\top \Aspec \end{bmatrix}.
\end{equation}

To verify $H \succeq 0$ directly, note that the upper-left block $A = I + \lambda_{\text{smooth}} L^\top L$ is positive definite. The Schur complement of $A$ in $H$ is
\begin{equation}
S = \Aspec^\top \Aspec - \Aspec^\top (I + \lambda_{\text{smooth}} L^\top L)^{-1} \Aspec.
\end{equation}

Since $A = I + \lambda_{\text{smooth}} L^\top L \succeq I$, we have $A^{-1} \preceq I$, which implies
\begin{equation}
\Aspec^\top A^{-1} \Aspec \preceq \Aspec^\top \Aspec.
\end{equation}
Thus $S \succeq 0$. Because $A \succ 0$ and $S \succeq 0$, the block matrix $H$ is positive semidefinite.

\textbf{Step 2: $g$ is convex.} The $\ell_1$ norm is convex on $\mathbb{R}^T$. Scaling by $\lambda_{\text{sparse}} > 0$ preserves convexity.

\textbf{Step 3: Sum of convex functions.} The sum of convex functions is convex. Therefore $J = f + g$ is convex in $(\xbase, \xspec)$.
\end{proof}

\textbf{Non-uniqueness.} The objective is not strictly convex because the smoothness term $\|L \xbase\|_2^2$ has a one-dimensional nullspace (the constant vector), and the $\ell_1$ term is flat along directions of constant sign. In practice, the solution is unique up to an additive constant in $\xbase$ (which is absorbed by the baseline level), and the $\ell_1$ sparsity prior selects a minimal-norm representative from the solution set. The practical significance is that the two-stage estimator (which fixes $\Aspec$ after Stage~2 and then optimizes $\xbase, \xspec$ in Stage~3) is guaranteed to converge to the global optimum of the convex subproblem. There are no spurious local minima that could trap gradient-based solvers.

\textbf{Relation to strict convexity.} One could enforce strict convexity by adding a small ridge penalty $\lambda_{\text{ridge}} \|\xbase\|_2^2$ with $\lambda_{\text{ridge}} > 0$. This would make $H$ positive definite at the cost of introducing shrinkage bias toward zero in the baseline level. We deliberately omit this term (\S\ref{sec:method}) to preserve interpretability: the baseline should not be pulled toward zero by regularization.

\section{Computational Complexity}
\label{app:complexity}

\begin{proposition}
Let $T$ be the time series length and $n$ the maximum kernel lag. The per-SKU computational complexity of the full contextual deconvolution pipeline is $\mathcal{O}(T \cdot n \cdot K)$, where $K$ is the number of event types.
\end{proposition}

\begin{proof}
We bound each stage independently and then sum the dominant terms.

\paragraph{Stage 1: Robust detrending.} Each IRLS iteration reweights the observations by the Huber function and solves a weighted least-squares problem. The standard formulation stacks the weighted data and the regularization term into a $(2T-1) \times T$ least-squares system, which costs $\mathcal{O}(T^3)$ per iteration via dense QR factorization. However, the first-order difference operator $L_1$ is tridiagonal, so the normal-equations matrix $W + \lambda L_1^\top L_1$ is symmetric tridiagonal (Appendix~\ref{app:convergence}). The Thomas algorithm solves tridiagonal systems in $\mathcal{O}(T)$ time. With $I_{\text{detrend}}$ iterations, the total cost is $\mathcal{O}(T \cdot I_{\text{detrend}})$. In practice, $I_{\text{detrend}} \leq 30$ and convergence is typically achieved in 10--15 iterations.

\paragraph{Stage 2: Distributed lag regression.} The design matrix $\Phi \in \R^{T \times K(n+1)}$ has $K(n+1)$ columns, each representing a lagged event indicator convolved with a polynomial basis. Ridge regression via the normal equations costs $\mathcal{O}(T \cdot K^2 n^2 + K^3 n^3)$. Since $K$ and $n$ are small constants fixed by application design ($K \leq 5$ event types, $n \leq 14$ lags), this term is $\mathcal{O}(T)$ in the regime $T \gg Kn$ that characterizes retail demand series (typically $T = 1{,}000$--$2{,}000$ days).

\paragraph{Stage 3: Convex decomposition.} We support two solvers:
\begin{itemize}
\item \textit{CVXPY/OSQP:} The KKT system of the quadratic program has a banded structure with bandwidth $\max(n, 2)$. Modern sparse factorization routines exploit this structure, yielding practical complexity $\mathcal{O}(T \cdot n^2)$. The worst-case dense factorization is $\mathcal{O}(T^3)$, but this is never observed in practice because the constraint matrix is highly sparse.
\item \textit{ISTA:} Each iteration performs two sparse matrix-vector products ($\Aspec \xspec$ and $\Aspec^\top \bvec$) at $\mathcal{O}(T \cdot n)$ each, plus $\mathcal{O}(T)$ for the smoothness gradient. With $I_{\text{ISTA}}$ iterations, the cost is $\mathcal{O}(T \cdot n \cdot I_{\text{ISTA}})$. The Lipschitz constant of the smooth part is bounded by $\lambda_{\text{max}}(H) \leq 1 + \lambda_{\text{smooth}} \|L\|_2^2 + \|\Aspec\|_2^2$, which is $O(1)$ for bounded $n$ and $\lambda_{\text{smooth}}$. Empirically, ISTA converges in $50$--$100$ iterations to tolerance $10^{-6}$.
\end{itemize}

\paragraph{Stage 4: Two-stream forecasting.} Construction of $\Aspec^{\text{future}}$ from the known future calendar and the recombination $\hat{\mathbf{b}}^{\text{future}} = \hat{\xbase}^{\text{future}} + \Aspec^{\text{future}} \hat{\xspec}^{\text{future}}$ are both $\mathcal{O}(T_{\text{fut}} \cdot n)$, where $T_{\text{fut}}$ is the forecast horizon (typically 28 days).

\paragraph{Total complexity.} Summing the dominant terms and treating $I_{\text{detrend}}$, $I_{\text{ISTA}}$, $K$, and $n$ as bounded constants, the total per-SKU complexity is $\mathcal{O}(T \cdot n)$. For typical parameters ($T = 1{,}941$, $n = 14$), the full pipeline runtime is approximately 2 seconds per SKU on a single CPU core. Because the pipeline is embarrassingly parallel across SKUs, a 30{,}490-SKU catalog completes in under 40 minutes on a 16-core workstation, with linear scaling up to the number of physical cores.
\end{proof}

\section{Convergence Rates of the Two-Stage Estimator}
\label{app:convergence}

We analyze the convergence of each stage independently. The two-stage pipeline is deterministic: no stochastic gradients, no random initializations, and no neural network training. This makes convergence both predictable and reproducible.

\paragraph{Stage 1: Robust Huber Detrending (IRLS).} Let $\rho_{\text{Huber}}$ be the Huber loss with parameter $\delta > 0$, and let $f_1(\boldsymbol{\mu}) = \sum_{t=1}^{T} \rho_{\text{Huber}}(b_t - \mu_t) + \lambda_{\text{smooth}} \|L \boldsymbol{\mu}\|_2^2$. The IRLS algorithm reweights residuals at each iteration and solves the resulting weighted least-squares problem. Because $f_1$ is strongly convex on the subspace orthogonal to the constant vector (with strong convexity parameter $\lambda_{\text{smooth}} \cdot \lambda_{\min}(L^\top L) > 0$ on this subspace) and the reweighting is a majorization-minimization step, standard IRLS theory guarantees linear convergence in the objective: $f_1(\boldsymbol{\mu}^{(k)}) - f_1^* \leq c \cdot r^k$ for some $r \in (0,1)$ and constant $c$, where $f_1^*$ is the optimal value. Empirically, with auto-estimated $\delta$ from the MAD, convergence to $10^{-6}$ relative tolerance is achieved in $10$--$15$ iterations on the M5 data ($T = 1{,}941$), and never exceeds $30$ iterations across all evaluated SKUs.

\paragraph{Stage 2: Distributed Lag Regression (Ridge).} This stage is a closed-form linear solve. For design matrix $\Phi \in \R^{T \times K(n+1)}$ and ridge penalty $\lambda_{\text{kernel}}$, the estimator is $\hat{\boldsymbol{\phi}} = (\Phi^\top \Phi + \lambda_{\text{kernel}} I)^{-1} \Phi^\top \mathbf{r}$. No iteration is required; the solution is exact up to floating-point precision. The Cholesky factorization of the $K(n+1) \times K(n+1)$ normal-equations matrix costs $\mathcal{O}(K^3 n^3)$, which is negligible for $K \leq 5$ and $n \leq 14$.

\paragraph{Stage 3: Convex Decomposition (ISTA/OSQP).} For fixed $\Aspec$, the objective $J(\xbase, \xspec)$ from Appendix~\ref{app:convexity} is convex with a smooth part (quadratic fidelity + smoothness) and a non-smooth part ($\ell_1$ sparsity). The smooth part has a Lipschitz-continuous gradient with constant $L_{\text{smooth}} = \lambda_{\max}(H) \leq 1 + \lambda_{\text{smooth}} \|L\|_2^2 + \|\Aspec\|_2^2$, which is $\mathcal{O}(1)$ for bounded $n$ and bounded kernel parameters. Under these conditions, ISTA converges at rate $\mathcal{O}(1/k)$ in the objective value for composite convex problems \citep{beck2009fast}. FISTA (Nesterov-accelerated ISTA) achieves $\mathcal{O}(1/k^2)$. Empirically, with warm-start from the previous SKU's solution, both ISTA and OSQP converge to $10^{-6}$ tolerance in $50$--$100$ iterations. The absence of local minima (guaranteed by convexity, Appendix~\ref{app:convexity}) means the solver cannot get stuck in spurious suboptima regardless of initialization.

\begin{remark}[End-to-end stability]
Because Stage 2 is exact, the only approximation errors are from Stage 1 (detrending, bounded by the Huber robustness parameter) and Stage 3 (convex decomposition, bounded by the solver tolerance). The error does not compound across stages in an uncontrolled way: Stage 2's ridge regression is stable with condition number bounded by $1 + \lambda_{\text{kernel}}^{-1} \|\Phi\|_2^2$, and Stage 3's convexity guarantees that the decomposition converges to the global optimum for the fixed $\Aspec$ produced by Stage 2. There is no training loop, no hyperparameter search, and no randomness: repeated runs on identical data produce identical outputs.
\end{remark}

The standard IRLS formulation for Huber-smoothed detrending solves, at each iteration, a weighted least-squares problem with regularization. Writing the stacked system compactly as $\min_{\boldsymbol{\mu}} \| \tilde{A} \boldsymbol{\mu} - \tilde{\bvec} \|_2^2$, where $\tilde{A}$ concatenates the weighted data matrix and the regularization term, the system is $(2T-1) \times T$ and costs $\mathcal{O}(T^3)$ per iteration via dense QR factorization. This becomes prohibitive for $T > 1{,}000$.

We reformulate using the normal equations. For a least-squares problem with design matrix $A$ and target $\bvec$, the normal equations are $A^\top A x = A^\top \bvec$. Applied to the stacked system, this yields the compact form below.
\begin{equation}
(W + \lambda L_1^\top L_1) \boldsymbol{\mu} = W \bvec.
\end{equation}

The matrix $M = W + \lambda L_1^\top L_1$ is symmetric tridiagonal because $W$ is diagonal and $L_1^\top L_1$ is tridiagonal. The explicit entries are:
\begin{align}
M_{1,1} &= w_1 + \lambda, \\
M_{i,i} &= w_i + 2\lambda \quad \text{for } i = 2, \dots, T-1, \\
M_{T,T} &= w_T + \lambda, \\
M_{i,i+1} &= M_{i+1,i} = -\lambda \quad \text{for } i = 1, \dots, T-1.
\end{align}

Solving $M \boldsymbol{\mu} = W \bvec$ via the Thomas algorithm (tridiagonal elimination) requires:
\begin{itemize}
\item Forward elimination: $\mathcal{O}(T)$ operations to eliminate the subdiagonal.
\item Back substitution: $\mathcal{O}(T)$ operations to solve the resulting upper triangular system.
\end{itemize}
The total cost is $\mathcal{O}(T)$ per iteration, yielding a $500$--$1000\times$ speedup over dense QR for $T = 2{,}000$. This optimization is essential for making the full pipeline runnable on a single CPU for large catalogs.

\section{Theoretical Guarantees}
\label{app:theory}

This section establishes three results: (i) support recovery for the
sparse--plus--smooth decomposition under incoherence and margin
conditions, with an explicit error decomposition for the recovered
baseline; (ii) an error bound for the two-stage estimator that
propagates kernel estimation error into the \emph{temporal} Variance
Ratio measured in \S4; and (iii) a concentration bound for
median-based hierarchical pooling. Constants are explicit but not
optimized.

\subsection{Notation and the profiled problem}
\label{app:theory-setup}

Recall the forward model
\begin{equation}
\mathbf{b} \;=\; \mathbf{x}^{*}_{\mathrm{base}}
+ A\,\mathbf{x}^{*}_{\mathrm{spec}} + \boldsymbol{\varepsilon},
\qquad \boldsymbol{\varepsilon} \sim \mathcal{N}(0,\sigma^{2} I),
\label{eq:model-app}
\end{equation}
where $A \in \mathbb{R}^{T\times T}$ is the causal banded operator
built from the calendar and a kernel $\boldsymbol{\phi}$,
$\mathbf{x}^{*}_{\mathrm{spec}}$ is $s$-sparse with support
$S = \mathrm{supp}(\mathbf{x}^{*}_{\mathrm{spec}})$, and $L$ is the
(second-)difference operator. Write
\begin{equation}
S_{\lambda} \;=\; (I + \lambda_{s} L^{\top}L)^{-1},
\qquad
P \;=\; I - S_{\lambda},
\qquad
M \;=\; P^{1/2} A .
\end{equation}
Since $S_{\lambda}$ is a function of $L^{\top}L$, the matrices
$S_{\lambda}$, $P$, and $L^{\top}L$ commute; $0 \preceq P \preceq I$.

\begin{lemma}[Exact profiling]
\label{lem:profile}
For fixed $\mathbf{x}_{\mathrm{spec}}$, the minimizer of
\eqref{eq:convex_decomp} over $\mathbf{x}_{\mathrm{base}}$ is
$\mathbf{x}_{\mathrm{base}}(\mathbf{x}_{\mathrm{spec}})
= S_{\lambda}(\mathbf{b} - A\mathbf{x}_{\mathrm{spec}})$, and the
partially minimized objective equals
\begin{equation}
\bigl\| P^{1/2}\mathbf{b} - M \mathbf{x}_{\mathrm{spec}} \bigr\|_{2}^{2}
\;+\; \lambda_{1} \|\mathbf{x}_{\mathrm{spec}}\|_{1}.
\label{eq:profiled}
\end{equation}
\end{lemma}

\begin{proof}
Let $\mathbf{r} = \mathbf{b} - A\mathbf{x}_{\mathrm{spec}}$. The inner
problem $\min_{\mathbf{x}} \|\mathbf{r}-\mathbf{x}\|_{2}^{2}
+ \lambda_{s}\|L\mathbf{x}\|_{2}^{2}$ has solution
$\mathbf{x} = S_{\lambda}\mathbf{r}$ with value
$\mathbf{r}^{\top}\!\bigl[(I-S_{\lambda})^{2}
+ \lambda_{s} S_{\lambda} L^{\top} L S_{\lambda}\bigr]\mathbf{r}$.
From $S_{\lambda}(I + \lambda_{s}L^{\top}L) = I$ we get
$\lambda_{s} S_{\lambda} L^{\top}L = I - S_{\lambda}$, hence
$\lambda_{s} S_{\lambda}L^{\top}L S_{\lambda} = (I-S_{\lambda})S_{\lambda}$
and
$(I-S_{\lambda})^{2} + (I-S_{\lambda})S_{\lambda} = I - S_{\lambda} = P$.
The value is therefore $\mathbf{r}^{\top} P\,\mathbf{r}
= \|P^{1/2}\mathbf{r}\|_{2}^{2}$.
\end{proof}

Thus the estimator solves the Lasso \eqref{eq:profiled} with design
$M$, effective response
$\mathbf{y} = P^{1/2}\mathbf{b}
= M\mathbf{x}^{*}_{\mathrm{spec}} + \mathbf{w}$, and effective noise
\begin{equation}
\mathbf{w} \;=\; \underbrace{P^{1/2}\boldsymbol{\varepsilon}}_{=:\,
\tilde{\boldsymbol{\varepsilon}}}
\;+\; \underbrace{P^{1/2}\mathbf{x}^{*}_{\mathrm{base}}}_{=:\,\mathbf{h}} .
\label{eq:effective-noise}
\end{equation}
The vector $\mathbf{h}$ is the \emph{smoothing bias}: it vanishes iff
$\mathbf{x}^{*}_{\mathrm{base}} \in \mathrm{null}(L)$.

\subsection{Theorem 1: Support recovery and baseline error}

We begin by collecting the regularity conditions needed for the
profiled Lasso to recover the correct shock support. These conditions
are standard in the sparse-recovery literature; we state them in
Assumption~\ref{ass:A} so that the theorem itself is self-contained.
\begin{assumption}\label{ass:A}
The following five conditions govern the profiled Lasso
\eqref{eq:profiled} and ensure that the design $M=P^{1/2}A$ has
the right geometry for sparse recovery.
\begin{enumerate}
\item[(A1)] \textbf{Approximate smoothness.}
$\|L\mathbf{x}^{*}_{\mathrm{base}}\|_{2} \le \varepsilon_{L}$.
\item[(A2)] \textbf{Column normalization.}
$\max_{j} \|M_{j}\|_{2} \le 1$. Since $\|M_{j}\|_{2}
= \|P^{1/2}A\mathbf{e}_{j}\|_{2} \le \|A\mathbf{e}_{j}\|_{2}
\le \|\boldsymbol{\phi}\|_{2}$ (the $j$-th column of $A$ is the kernel
placed at day $j$, scaled by the binary indicator), (A2) holds
whenever the kernel is normalized with $\|\boldsymbol{\phi}\|_{2}\le 1$
and event types do not coincide on the same day; otherwise all
constants below scale by $\max_{j}\|M_{j}\|_{2}$.
\item[(A3)] \textbf{Mutual incoherence.}
$\max_{j\notin S}
\bigl\| (M_{S}^{\top}M_{S})^{-1} M_{S}^{\top} M_{j} \bigr\|_{1}
\le 1-\gamma$ for some $\gamma \in (0,1]$.
\item[(A4)] \textbf{Restricted eigenvalue on the support.}
$\lambda_{\min}(M_{S}^{\top}M_{S}) \ge \kappa > 0$.
\item[(A5)] \textbf{Regularization level.}
$\displaystyle
\lambda_{1} \;\ge\; \frac{4}{\gamma}\Bigl(
2\sigma\sqrt{\log T} \;+\; \sqrt{\lambda_{s}}\,\varepsilon_{L}\Bigr).$
\end{enumerate}
\end{assumption}

\begin{theorem}[Support recovery and baseline error]
\label{thm:identifiability}
Under Assumption~\ref{ass:A}, with probability at least $1 - 3/T$:
\begin{enumerate}
\item[(a)] \textbf{No false inclusion.} The Lasso solution
$\hat{\mathbf{x}}_{\mathrm{spec}}$ of \eqref{eq:profiled} is unique
and $\mathrm{supp}(\hat{\mathbf{x}}_{\mathrm{spec}}) \subseteq S$.
\item[(b)] \textbf{Estimation on the support.}
$\displaystyle
\|\hat{\mathbf{x}}_{S} - \mathbf{x}^{*}_{S}\|_{2}
\;\le\; \frac{2\lambda_{1}\sqrt{s}}{\kappa}.$
\item[(c)] \textbf{Sign recovery.} If additionally
$\min_{j\in S}|x^{*}_{\mathrm{spec},j}| > 2\lambda_{1}\sqrt{s}/\kappa$
($\beta$-min condition), then
$\mathrm{sign}(\hat{\mathbf{x}}_{\mathrm{spec}})
= \mathrm{sign}(\mathbf{x}^{*}_{\mathrm{spec}})$.
\item[(d)] \textbf{Baseline error decomposition.} With
$\rho_{A} := \|A\|_{2} \le \|\boldsymbol{\phi}\|_{1}$ and
$d_{\mathrm{eff}} := \mathrm{tr}(S_{\lambda}^{2})$,
\begin{equation}
\begin{split}
\|\hat{\mathbf{x}}_{\mathrm{base}} - \mathbf{x}^{*}_{\mathrm{base}}\|_{2}
&\le
\underbrace{\sqrt{\lambda_{s}}\,\varepsilon_{L}}_{\text{smoothing bias}}
+
\underbrace{\rho_{A}\,\frac{2\lambda_{1}\sqrt{s}}{\kappa}}_{\text{shock propagation}} \\
&\quad +
\underbrace{\sigma\bigl(\sqrt{d_{\mathrm{eff}}}
+ \sqrt{2\log T}\bigr)}_{\text{noise floor}} .
\end{split}
\label{eq:baseline-bound}
\end{equation}
\end{enumerate}
\end{theorem}

\begin{proof}
\emph{Step 1 (smoothing bias).} Write the spectral decomposition
\begin{displaymath}
L^{\top}L = U\,\mathrm{diag}(\mu_{1},\dots,\mu_{T})\,U^{\top},
\qquad \mu_{j}\ge 0,
\end{displaymath}
and set $\mathbf{c} = U^{\top}\mathbf{x}^{*}_{\mathrm{base}}$. The eigenvalues
of $P$ are $\lambda_{s}\mu_{j}/(1+\lambda_{s}\mu_{j})
\le \min(1,\lambda_{s}\mu_{j})$, hence
\begin{align}
\|\mathbf{h}\|_{2}^{2}
&= \mathbf{x}^{*\top}_{\mathrm{base}} P\, \mathbf{x}^{*}_{\mathrm{base}}
= \sum_{j}\frac{\lambda_{s}\mu_{j}}{1+\lambda_{s}\mu_{j}}\, c_{j}^{2}
\;\le\; \lambda_{s}\sum_{j}\mu_{j} c_{j}^{2} \nonumber\\
&= \lambda_{s}\,\|L\mathbf{x}^{*}_{\mathrm{base}}\|_{2}^{2}
\;\le\; \lambda_{s}\varepsilon_{L}^{2}.
\label{eq:hbound}
\end{align}
(Also $\|\mathbf{h}\|_{2} \le
\|\mathbf{x}^{*}_{\mathrm{base}}\|_{2}$ since $P \preceq I$.)

\emph{Step 2 (effective-noise correlation bound).} For each column
$j$, $M_{j}^{\top}\tilde{\boldsymbol{\varepsilon}}
= (P^{1/2}M_{j})^{\top}\boldsymbol{\varepsilon}
\sim \mathcal{N}\bigl(0,\sigma^{2}\|P^{1/2}M_{j}\|_{2}^{2}\bigr)$ with
$\|P^{1/2}M_{j}\|_{2}\le\|M_{j}\|_{2}\le 1$ by (A2). A union bound over
$T$ columns with threshold $2\sigma\sqrt{\log T}$ gives
$\mathbb{P}\bigl(\|M^{\top}\tilde{\boldsymbol{\varepsilon}}\|_{\infty}
> 2\sigma\sqrt{\log T}\bigr) \le 2T\,e^{-2\log T} = 2/T$. For the
deterministic part, Cauchy--Schwarz and \eqref{eq:hbound} give
$|M_{j}^{\top}\mathbf{h}| \le \|\mathbf{h}\|_{2}
\le \sqrt{\lambda_{s}}\varepsilon_{L}$. On the complement event
$\mathcal{E}_{1}$ (probability $\ge 1-2/T$),
\begin{equation}
\|M^{\top}\mathbf{w}\|_{\infty}
\;\le\; 2\sigma\sqrt{\log T} + \sqrt{\lambda_{s}}\varepsilon_{L}
\;\le\; \frac{\gamma\lambda_{1}}{4},
\label{eq:corr-bound}
\end{equation}
where the last inequality is (A5).

\emph{Step 3 (primal--dual witness).} We verify the witness
conditions directly (following the construction of
\citealp[\S 7.5]{Wainwright2019}; we do not import constants because
our objective omits the $1/(2n)$ scaling used there). Let
$\hat{\mathbf{x}}_{S}$ solve the $S$-restricted problem
$\min_{\mathbf{u}}\|\mathbf{y}-M_{S}\mathbf{u}\|_{2}^{2}
+\lambda_{1}\|\mathbf{u}\|_{1}$ and pad with zeros off $S$. Its
stationarity condition is
\begin{equation}
2M_{S}^{\top}\bigl(M_{S}\hat{\mathbf{x}}_{S}-\mathbf{y}\bigr)
+ \lambda_{1}\hat{\mathbf{z}}_{S} = 0,
\qquad \hat{\mathbf{z}}_{S}\in\partial\|\hat{\mathbf{x}}_{S}\|_{1},
\quad \|\hat{\mathbf{z}}_{S}\|_{\infty}\le 1.
\label{eq:kkt-restricted}
\end{equation}
Writing $\boldsymbol{\Delta}=\hat{\mathbf{x}}_{S}-\mathbf{x}^{*}_{S}$
and $\mathbf{y}=M_{S}\mathbf{x}^{*}_{S}+\mathbf{w}$,
\eqref{eq:kkt-restricted} gives
$\boldsymbol{\Delta} = (M_{S}^{\top}M_{S})^{-1}
\bigl[M_{S}^{\top}\mathbf{w}-\tfrac{\lambda_{1}}{2}
\hat{\mathbf{z}}_{S}\bigr]$. For $j\notin S$, the dual variable of
the padded vector is $2M_{j}^{\top}(M_{S}\boldsymbol{\Delta}
-\mathbf{w})$; with $\mathbf{a}_{j}
:=(M_{S}^{\top}M_{S})^{-1}M_{S}^{\top}M_{j}$ (so
$\|\mathbf{a}_{j}\|_{1}\le 1-\gamma$ by (A3)),
\begin{align}
\bigl|2M_{j}^{\top}(M_{S}\boldsymbol{\Delta}-\mathbf{w})\bigr|
&\le 2\|\mathbf{a}_{j}\|_{1}
\Bigl(\|M^{\top}\mathbf{w}\|_{\infty}+\tfrac{\lambda_{1}}{2}\Bigr)
+ 2\|M^{\top}\mathbf{w}\|_{\infty} \nonumber\\
&\le 2(2-\gamma)\,\|M^{\top}\mathbf{w}\|_{\infty}
+ (1-\gamma)\lambda_{1} \\
&\;\le\; \lambda_{1}\Bigl(1-\tfrac{\gamma^{2}}{2}\Bigr)
\;<\; \lambda_{1},
\end{align}
using \eqref{eq:corr-bound} in the last step. Strict dual
feasibility off $S$, together with $\kappa>0$ from (A4), implies
that the padded vector is the unique solution of the full Lasso
\eqref{eq:profiled} and
$\mathrm{supp}(\hat{\mathbf{x}}_{\mathrm{spec}})\subseteq S$ on
$\mathcal{E}_{1}$, proving (a).

For (b), from \eqref{eq:kkt-restricted} and (A4),
\begin{align}
\kappa\,\|\hat{\mathbf{x}}_{S}-\mathbf{x}^{*}_{S}\|_{2}
&\le \Bigl\|M_{S}^{\top}\mathbf{w}
-\tfrac{\lambda_{1}}{2}\hat{\mathbf{z}}_{S}\Bigr\|_{2}
\le \sqrt{s}\Bigl(\|M^{\top}\mathbf{w}\|_{\infty}
+ \tfrac{\lambda_{1}}{2}\Bigr) \nonumber\\
&\le \sqrt{s}\Bigl(\tfrac{\gamma\lambda_{1}}{4}
+\tfrac{\lambda_{1}}{2}\Bigr)
\le \tfrac{3}{4}\lambda_{1}\sqrt{s}
\;\le\; 2\lambda_{1}\sqrt{s},
\end{align}
giving (b) (the stated constant $2$ is deliberately conservative);
(c) follows since $\|\cdot\|_{\infty}\le\|\cdot\|_{2}$.

\emph{Step 4 (baseline).} By Lemma~\ref{lem:profile},
$\hat{\mathbf{x}}_{\mathrm{base}}
= S_{\lambda}(\mathbf{b}-A\hat{\mathbf{x}}_{\mathrm{spec}})$.
Writing $\boldsymbol{\Delta} = \mathbf{x}^{*}_{\mathrm{spec}}
-\hat{\mathbf{x}}_{\mathrm{spec}}$ (supported on $S$ by (a)),
\begin{equation}
\hat{\mathbf{x}}_{\mathrm{base}} - \mathbf{x}^{*}_{\mathrm{base}}
= -(I-S_{\lambda})\mathbf{x}^{*}_{\mathrm{base}}
+ S_{\lambda} A \boldsymbol{\Delta}
+ S_{\lambda}\boldsymbol{\varepsilon}.
\end{equation}
The first term: $\|(I-S_{\lambda})\mathbf{x}^{*}_{\mathrm{base}}\|_{2}
= \|P\mathbf{x}^{*}_{\mathrm{base}}\|_{2}
\le \|P^{1/2}\|_{2}\,\|P^{1/2}\mathbf{x}^{*}_{\mathrm{base}}\|_{2}
\le \sqrt{\lambda_{s}}\varepsilon_{L}$ by \eqref{eq:hbound}. The
second: $\|S_{\lambda}A\boldsymbol{\Delta}\|_{2}
\le \rho_{A}\|\boldsymbol{\Delta}\|_{2}
\le \rho_{A}\cdot 2\lambda_{1}\sqrt{s}/\kappa$ by (b) and
$\|S_{\lambda}\|_{2}\le 1$; the operator-norm bound
$\rho_{A}\le\|\boldsymbol{\phi}\|_{1}$ holds because $A$ is a sum of
terms $\phi_{k}\,S_{k}\,\mathrm{diag}(\mathbf{f})$, each of spectral
norm at most $|\phi_{k}|$. The third: the map
$\boldsymbol{\varepsilon}\mapsto\|S_{\lambda}\boldsymbol{\varepsilon}\|_{2}$
is $1$-Lipschitz with
$\mathbb{E}\|S_{\lambda}\boldsymbol{\varepsilon}\|_{2}
\le \sigma\sqrt{\mathrm{tr}(S_{\lambda}^{2})}
= \sigma\sqrt{d_{\mathrm{eff}}}$, so Gaussian concentration for
Lipschitz functions \citep[Thm.~5.6]{BoucheronLugosiMassart2013}
gives $\|S_{\lambda}\boldsymbol{\varepsilon}\|_{2}
\le \sigma\sqrt{d_{\mathrm{eff}}} + \sigma\sqrt{2\log T}$ on an event
$\mathcal{E}_{2}$ of probability $\ge 1-1/T$. Intersecting
$\mathcal{E}_{1}\cap\mathcal{E}_{2}$ yields \eqref{eq:baseline-bound}
with probability $\ge 1-3/T$.
\end{proof}

\begin{remark}[Effective degrees of freedom]
For the second-difference penalty on an equispaced grid, the nonzero
eigenvalues of $L^{\top}L$ satisfy $\mu_{j} \asymp (j/T)^{4}$ for
$j \ll T$ (the closed form $16\sin^{4}(j\pi/2T)$ is exact for the
periodic variant of $L$; the free-boundary operator used here has the
same asymptotic order, though not the same constants at the lowest
modes), and $\mathrm{null}(L)$ is two-dimensional (constant and
linear trends). Hence
$d_{\mathrm{eff}} = \sum_{j}(1+\lambda_{s}\mu_{j})^{-2}
= 2 + \Theta\bigl(T\lambda_{s}^{-1/4}\bigr)$: the nullspace
contributes $2$, and the remaining modes contribute the classical
smoothing-spline scaling
\citep{GreenSilverman1994}. The noise floor
$\sigma\sqrt{d_{\mathrm{eff}}}$ is therefore far below the naive
$\sigma\sqrt{T}$ for the operating regime $\lambda_{s}=10$,
$T\approx 1{,}900$.
\end{remark}

\begin{remark}[Interpretation]
Theorem~\ref{thm:identifiability} replaces the informal statement that
the decomposition is ``not strictly identifiable'' with sufficient
conditions under which the shock support is recovered exactly:
incoherence (no event's smoothed impulse response is well approximated
by combinations of other events'), a restricted eigenvalue, and a
$\beta$-min condition (shocks large enough to clear the regularization
level). The baseline error decomposes into three interpretable terms:
a smoothing bias proportional to the true baseline's roughness, a
shock-propagation term controlled by the Lasso error, and a noise
floor set by the smoother's effective degrees of freedom.
\end{remark}

\begin{remark}[Operating regime of the $\varepsilon_{L}$ term]
For baselines with material seasonality,
$\varepsilon_{L}=\|L\mathbf{x}^{*}_{\mathrm{base}}\|_{2}$ grows with
the seasonal amplitude and with $\sqrt{T}$, so the
$\sqrt{\lambda_{s}}\,\varepsilon_{L}$ term dominates (A5) and the
implied $\beta$-min threshold is demanding. Theorem
\ref{thm:identifiability} therefore certifies recovery of shocks that
are large relative to the smoothed seasonal residual, not shocks at
the scale of seasonal curvature --- which is the operational setting:
CD isolates promotional shocks that are large against the local
baseline. This is also the behavior observed empirically: large
event-day shocks are recovered, while weak event signals fall below
the regularization level and are absorbed into the baseline (the
failure mode documented in \S 4.7).
\end{remark}

\subsection{Theorem 2: Two-stage error propagation and temporal
Variance-Ratio control}

We first bound the Stage-2 kernel error, then propagate it through
Stage~3, and finally control the \emph{temporal} Variance Ratio ---
the same quantity measured empirically in \S4.

\begin{proposition}[Ridge kernel estimation]
\label{prop:ridge}
Let the detrended residual satisfy
$\mathbf{r} = \Phi\boldsymbol{\phi}^{*} + \boldsymbol{\eta}_{\det}
+ \boldsymbol{\varepsilon}$, where
$\Phi\in\mathbb{R}^{T\times(n+1)}$ has entries in $[0,1]$,
$\|\boldsymbol{\eta}_{\det}\|_{2}\le B\sqrt{T}$ collects the
detrending residual and the misspecification of the parametric family,
and the noise satisfies: each linear functional
$\Phi_{j}^{\top}\boldsymbol{\varepsilon}$ is mean-zero
$\sigma\sqrt{T}$-sub-Gaussian. (This holds when
$\boldsymbol{\varepsilon}=G\boldsymbol{\xi}$ for standard Gaussian
$\boldsymbol{\xi}$ and any $\|G\|_{2}\le\sigma$; in particular it
covers the residual noise $(I-H)\boldsymbol{\varepsilon}_{0}$ left by
any linear detrender with $\|I-H\|_{2}\le 1$ --- independence of
entries is not required.) Assume
$\lambda_{\min}(\Phi^{\top}\Phi)/T \ge \mu_{0} > 0$. Then, for the
ridge estimator with penalty $\lambda_{k}\ge 0$ and any
$\delta\in(0,1)$, with probability at least $1-\delta$,
\begin{equation}
\|\hat{\boldsymbol{\phi}}-\boldsymbol{\phi}^{*}\|_{2}
\;\le\;
\frac{\sigma}{\mu_{0}}
\sqrt{\frac{2(n{+}1)\log\!\frac{2(n+1)}{\delta}}{T}}
\;+\;
\frac{\sqrt{n{+}1}\;B}{\mu_{0}}
\;+\;
\frac{\lambda_{k}\|\boldsymbol{\phi}^{*}\|_{2}}{T\mu_{0}+\lambda_{k}} .
\label{eq:ridge-bound}
\end{equation}
\end{proposition}

\begin{proof}
$\hat{\boldsymbol{\phi}}-\boldsymbol{\phi}^{*}
= (\Phi^{\top}\Phi+\lambda_{k}I)^{-1}
\bigl(\Phi^{\top}(\boldsymbol{\eta}_{\det}+\boldsymbol{\varepsilon})
- \lambda_{k}\boldsymbol{\phi}^{*}\bigr)$ and
$\|(\Phi^{\top}\Phi+\lambda_{k}I)^{-1}\|_{2}
\le (T\mu_{0}+\lambda_{k})^{-1}$. Deterministic part:
$\|\Phi^{\top}\boldsymbol{\eta}_{\det}\|_{2}
\le \|\Phi\|_{2}\|\boldsymbol{\eta}_{\det}\|_{2}
\le \sqrt{T(n{+}1)}\cdot B\sqrt{T} = TB\sqrt{n{+}1}$, and dividing by
$T\mu_{0}+\lambda_{k}\ge T\mu_{0}$ gives the middle term. Stochastic
part: each $\Phi_{j}^{\top}\boldsymbol{\varepsilon}$ is
$\sigma\sqrt{T}$-sub-Gaussian by assumption; a union
bound over $n{+}1$ coordinates gives, with probability $1-\delta$,
$\|\Phi^{\top}\boldsymbol{\varepsilon}\|_{2}
\le \sigma\sqrt{2T(n{+}1)\log\frac{2(n+1)}{\delta}}$, yielding the
first term. The ridge shrinkage term is immediate.
\end{proof}

\begin{remark}[The bias does not average away]
The stochastic term in \eqref{eq:ridge-bound} decays as
$O(1/\sqrt{T})$ (equivalently $O(1/\sqrt{N})$ when $\mu_{0}$ scales
with the number of events $N$), but the misspecification term
$\sqrt{n{+}1}\,B/\mu_{0}$ is \emph{constant in $T$}: systematic
deviation of the true impulse response from the parametric family is
not removed by more data. This is precisely the behavior observed on
real M5 kernels (Appendix~R), where the parametric fit captures the
initial decay but persistently misses the flat tail regardless of
sample size. One idealization: the Stage-1 Huber--IRLS detrender is
nonlinear; the proposition covers its linear-smoother idealization,
with any residual effect of the nonlinearity absorbed into
$\boldsymbol{\eta}_{\det}$.
\end{remark}

\begin{theorem}[Error propagation with an estimated kernel]
\label{thm:twostage}
Let $\hat{A}=A(\hat{\boldsymbol{\phi}})$ and
$\Delta A = \hat{A}-A(\boldsymbol{\phi}^{*})$, so that
$\|\Delta A\|_{2} \le
\|\hat{\boldsymbol{\phi}}-\boldsymbol{\phi}^{*}\|_{1}
\le \sqrt{K(n{+}1)}\,
\|\hat{\boldsymbol{\phi}}-\boldsymbol{\phi}^{*}\|_{2}
=: \varepsilon_{\phi}$, and define
$\varepsilon_{A} := \varepsilon_{\phi}\sqrt{s}\,
\|\mathbf{x}^{*}_{\mathrm{spec}}\|_{\infty}$. Suppose Assumption
\ref{ass:A} holds for the \emph{computed} design
$\hat{M}=P^{1/2}\hat{A}$ (with constants $\gamma,\kappa$), and
strengthen (A5) to
\begin{equation}
\lambda_{1}\;\ge\;\frac{4}{\gamma}\Bigl(2\sigma\sqrt{\log T}
+ \sqrt{\lambda_{s}}\,\varepsilon_{L} + \varepsilon_{A}\Bigr).
\label{eq:A5prime}
\end{equation}
Then, with probability at least $1-3/T$, conclusions (a)--(c) of
Theorem~\ref{thm:identifiability} hold, and
\begin{equation}
\|\hat{\mathbf{x}}_{\mathrm{base}}-\mathbf{x}^{*}_{\mathrm{base}}\|_{2}
\;\le\;
\sqrt{\lambda_{s}}\,\varepsilon_{L}
+ \varepsilon_{A}
+ \hat{\rho}_{A}\,\frac{2\lambda_{1}\sqrt{s}}{\kappa}
+ \sigma\bigl(\sqrt{d_{\mathrm{eff}}}+\sqrt{2\log T}\bigr)
\;=:\; E,
\label{eq:E-def}
\end{equation}
where $\hat{\rho}_{A}=\|\hat{A}\|_{2}
\le\|\hat{\boldsymbol{\phi}}\|_{1}$.
\end{theorem}

\begin{proof}
Rewrite the model in terms of the computed operator:
$\mathbf{b} = \mathbf{x}^{*}_{\mathrm{base}}
+ \hat{A}\mathbf{x}^{*}_{\mathrm{spec}}
+ \Delta A\,\mathbf{x}^{*}_{\mathrm{spec}}
+ \boldsymbol{\varepsilon}$, where the middle perturbation has norm
$\|\Delta A\,\mathbf{x}^{*}_{\mathrm{spec}}\|_{2}
\le \|\Delta A\|_{2}\sqrt{s}\,
\|\mathbf{x}^{*}_{\mathrm{spec}}\|_{\infty} = \varepsilon_{A}$. The
profiled problem is then a Lasso in design $\hat{M}$ with effective
noise $\mathbf{w}' = \tilde{\boldsymbol{\varepsilon}} + \mathbf{h}
+ P^{1/2}\Delta A\,\mathbf{x}^{*}_{\mathrm{spec}}$, and Step~2 of the
previous proof gives
$\|\hat{M}^{\top}\mathbf{w}'\|_{\infty}
\le 2\sigma\sqrt{\log T} + \sqrt{\lambda_{s}}\varepsilon_{L}
+ \varepsilon_{A} \le \gamma\lambda_{1}/4$ under
\eqref{eq:A5prime}. Steps~3--4 proceed verbatim; in Step~4 the
reconstruction identity becomes
$\hat{\mathbf{x}}_{\mathrm{base}}-\mathbf{x}^{*}_{\mathrm{base}}
= -(I-S_{\lambda})\mathbf{x}^{*}_{\mathrm{base}}
+ S_{\lambda}\Delta A\,\mathbf{x}^{*}_{\mathrm{spec}}
+ S_{\lambda}\hat{A}\boldsymbol{\Delta}
+ S_{\lambda}\boldsymbol{\varepsilon}$, whose four terms are bounded
by the four terms of \eqref{eq:E-def}.
\end{proof}

We now control the Variance Ratio. Crucially, the empirical VR of
\S4 is a \emph{temporal} variance ratio of one realized path, not an
ensemble variance; the bound below is stated for exactly that object.
For $\mathbf{z}\in\mathbb{R}^{T}$ define the centering projector
$\Pi_{c}=I-\tfrac{1}{T}\mathbf{1}\mathbf{1}^{\top}$, the temporal
standard deviation
$\mathrm{sd}_{T}(\mathbf{z}) := T^{-1/2}\|\Pi_{c}\mathbf{z}\|_{2}$,
and $\widehat{V}(\mathbf{z}) := \mathrm{sd}_{T}(\mathbf{z})^{2}$.
Note $\mathrm{sd}_{T}$ is a seminorm.

\begin{corollary}[Temporal VR of the structural stream]
\label{cor:vr}
On the event of Theorem~\ref{thm:twostage}, for any series with
$\widehat{V}(\mathbf{b})>0$,
\begin{align}
\sqrt{1+\mathrm{VR}_{\mathrm{base}}}
&\;\le\; r^{*} + E_{T},\nonumber\\
&\qquad\text{where}\quad
r^{*} := \frac{\mathrm{sd}_{T}(\mathbf{x}^{*}_{\mathrm{base}})}
{\mathrm{sd}_{T}(\mathbf{b})},\quad
E_{T} := \frac{E}{\sqrt{T}\,\mathrm{sd}_{T}(\mathbf{b})},
\end{align}
with $E$ from \eqref{eq:E-def} and
$\mathrm{VR}_{\mathrm{base}} :=
\widehat{V}(\hat{\mathbf{x}}_{\mathrm{base}})/\widehat{V}(\mathbf{b})-1$.
In particular, if the true structural share satisfies $r^{*}\le 1$,
then
\begin{equation}
\mathrm{VR}_{\mathrm{base}} \;\le\; 2E_{T} + E_{T}^{2}.
\label{eq:vr-final}
\end{equation}
\end{corollary}

\begin{proof}
$\mathrm{sd}_{T}$ is a seminorm, so
$\mathrm{sd}_{T}(\hat{\mathbf{x}}_{\mathrm{base}})
\le \mathrm{sd}_{T}(\mathbf{x}^{*}_{\mathrm{base}})
+ \mathrm{sd}_{T}(\hat{\mathbf{x}}_{\mathrm{base}}
-\mathbf{x}^{*}_{\mathrm{base}})
\le \mathrm{sd}_{T}(\mathbf{x}^{*}_{\mathrm{base}})
+ T^{-1/2}\|\hat{\mathbf{x}}_{\mathrm{base}}
-\mathbf{x}^{*}_{\mathrm{base}}\|_{2}$, using
$\|\Pi_{c}\mathbf{v}\|_{2}\le\|\mathbf{v}\|_{2}$. Divide by
$\mathrm{sd}_{T}(\mathbf{b})$ and square. If $r^{*}\le 1$,
$(r^{*}+E_{T})^{2}-1 \le (1+E_{T})^{2}-1 = 2E_{T}+E_{T}^{2}$.
\end{proof}

\begin{remark}[Interpretation and scope]
Corollary~\ref{cor:vr} bounds variance \emph{inflation} of the
structural stream: the estimated baseline's temporal variance exceeds
the oracle structural share $r^{*2}$ only through the estimation error
$E_{T}$, which \eqref{eq:E-def} decomposes into smoothing bias, kernel
error (via $\varepsilon_{A}$ and $\lambda_{1}$), shock-propagation,
and the noise floor. The condition $r^{*}\le 1$ states that the true
baseline is temporally no more variable than observed demand; it can
fail only when the transient-plus-noise component is strongly
negatively correlated in time with the baseline. The bound is stated
for the same temporal VR reported in \S4 (items with
$\widehat{V}(\mathbf{b})$ near zero are excluded there for the same
reason they are excluded here). Stream~B adds a deterministic,
calendar-known component whose variance contribution is reported
empirically; the theorem's scope is the structural stream.
\end{remark}

\subsection{Theorem 3: Concentration of median pooling}

Let a category $\mathcal{C}$ contain $m=|\mathcal{C}|$ SKUs with true
kernels $\boldsymbol{\phi}^{*}_{s}\in\mathbb{R}^{K(n+1)}$ and per-SKU
estimates $\hat{\boldsymbol{\phi}}_{s}
= \boldsymbol{\phi}^{*}_{s}+\mathbf{e}_{s}$, where the noise vectors
$\mathbf{e}_{s}$ are independent across SKUs with mean-zero
$\sigma_{k}$-sub-Gaussian coordinates. Let
$\boldsymbol{\phi}^{*}_{\mathcal{C}}$ and
$\hat{\boldsymbol{\phi}}_{\mathcal{C}}$ denote coordinatewise medians
of $\{\boldsymbol{\phi}^{*}_{s}\}$ and
$\{\hat{\boldsymbol{\phi}}_{s}\}$, and
$B_{\mathcal{C}} := \max_{s}\|\boldsymbol{\phi}^{*}_{s}
-\boldsymbol{\phi}^{*}_{\mathcal{C}}\|_{\infty}$.

\begin{assumption}[Median margin]
\label{ass:margin}
There exist $\tau\in(0,\tfrac12]$ and $\delta_{0}\ge 0$ such that, for
every coordinate $j$,
\begin{align}
\#\bigl\{s : \phi^{*}_{s,j}\le \phi^{*}_{\mathcal{C},j}+\delta_{0}\bigr\}
&\ge \bigl(\tfrac12+\tau\bigr)m
\quad\text{and}\quad\nonumber\\
\#\bigl\{s : \phi^{*}_{s,j}\ge \phi^{*}_{\mathcal{C},j}-\delta_{0}\bigr\}
&\ge \bigl(\tfrac12+\tau\bigr)m .
\end{align}
\end{assumption}

Assumption~\ref{ass:margin} always holds with
$\delta_{0}=B_{\mathcal{C}}$ and $\tau=\tfrac12$ (every SKU is within
$B_{\mathcal{C}}$ of the median by definition); it holds with
$\delta_{0}\ll B_{\mathcal{C}}$ whenever a majority of SKUs cluster
near the category median, which is the empirical regime on M5
(Appendix~P.7). Some margin condition is \emph{necessary}: if exactly
half the SKUs sit at each of two well-separated kernel values,
arbitrarily small noise can move the sample median by a macroscopic
amount, and no distribution-free $O(\sigma_{k}/\sqrt{m})$ bound can
hold.

\begin{theorem}[Pooling concentration and excess risk]
\label{thm:pooling}
Under Assumption~\ref{ass:margin}, with probability at least
$1 - 2K(n{+}1)\,e^{-m\tau^{2}/2}$,
\begin{equation}
\bigl\|\hat{\boldsymbol{\phi}}_{\mathcal{C}}
-\boldsymbol{\phi}^{*}_{\mathcal{C}}\bigr\|_{\infty}
\;\le\;
\delta_{0} \;+\; \sigma_{k}\sqrt{2\log(2/\tau)}
\;=:\; \epsilon_{\mathrm{pool}} .
\label{eq:pool-conc}
\end{equation}
If moreover each per-SKU risk
$\mathrm{Risk}_{s}(\boldsymbol{\phi})$ is minimized at
$\boldsymbol{\phi}^{*}_{s}$ and has $L_{R}$-Lipschitz gradient on the
segment between $\boldsymbol{\phi}^{*}_{s}$ and
$\hat{\boldsymbol{\phi}}_{\mathcal{C}}$, then on the same event, for
every $s\in\mathcal{C}$,
\begin{equation}
\mathrm{Risk}_{s}(\hat{\boldsymbol{\phi}}_{\mathcal{C}})
-\mathrm{Risk}_{s}(\boldsymbol{\phi}^{*}_{s})
\;\le\;
L_{R}\Bigl(
\underbrace{\|\boldsymbol{\phi}^{*}_{s}
-\boldsymbol{\phi}^{*}_{\mathcal{C}}\|_{2}^{2}}_{\text{pooling bias}}
+\;
\underbrace{K(n{+}1)\,\epsilon_{\mathrm{pool}}^{2}}_{\text{pooled
estimation}}
\Bigr).
\label{eq:excess-risk}
\end{equation}
\end{theorem}

\begin{proof}
\emph{Concentration.} Fix a coordinate $j$ and consider the upper
tail; set $t^{*}=\sigma_{k}\sqrt{2\log(2/\tau)}$, so that the
sub-Gaussian tail bound gives
$\mathbb{P}(e_{s,j} > t^{*}) \le e^{-t^{*2}/(2\sigma_{k}^{2})}
= \tau/2$ for every $s$. Let
$u = \phi^{*}_{\mathcal{C},j}+\delta_{0}+t^{*}$ and let
$\mathcal{L}$ be the set of SKUs with
$\phi^{*}_{s,j}\le\phi^{*}_{\mathcal{C},j}+\delta_{0}$, so
$|\mathcal{L}|\ge(\tfrac12+\tau)m$ by
Assumption~\ref{ass:margin}. For $s\in\mathcal{L}$,
$\hat{\phi}_{s,j} > u$ implies $e_{s,j} > t^{*}$. Hence
\begin{equation}
\#\{s: \hat{\phi}_{s,j} > u\}
\;\le\; (m - |\mathcal{L}|) + Z,
\qquad
Z := \sum_{s\in\mathcal{L}} \mathbf{1}\{e_{s,j}>t^{*}\},
\end{equation}
where $Z$ is a sum of independent Bernoulli variables with
$\mathbb{E}Z \le m\tau/2$. By Hoeffding's inequality
\citep{Hoeffding1963},
$\mathbb{P}(Z \ge m\tau)
\le \mathbb{P}(Z-\mathbb{E}Z \ge m\tau/2)
\le e^{-m\tau^{2}/2}$. On the complement,
$\#\{s:\hat{\phi}_{s,j}>u\}
< (\tfrac12-\tau)m + \tau m = m/2$, so at least $\lceil m/2\rceil$
samples lie at or below $u$ and every standard median convention
satisfies $\hat{\phi}_{\mathcal{C},j}\le u$. The lower tail is
symmetric. A union bound over the two tails and the $K(n{+}1)$
coordinates gives \eqref{eq:pool-conc}. (Independence across SKUs is
used in Hoeffding's inequality; independence across coordinates is not
required.)

\emph{Excess risk.} Since $\nabla\mathrm{Risk}_{s}
(\boldsymbol{\phi}^{*}_{s})=0$ and the gradient is $L_{R}$-Lipschitz
on the segment, the descent lemma gives
$\mathrm{Risk}_{s}(\boldsymbol{\phi})
-\mathrm{Risk}_{s}(\boldsymbol{\phi}^{*}_{s})
\le \tfrac{L_{R}}{2}\|\boldsymbol{\phi}
-\boldsymbol{\phi}^{*}_{s}\|_{2}^{2}$. Apply it at
$\boldsymbol{\phi}=\hat{\boldsymbol{\phi}}_{\mathcal{C}}$ and use
$\|\hat{\boldsymbol{\phi}}_{\mathcal{C}}
-\boldsymbol{\phi}^{*}_{s}\|_{2}
\le \|\boldsymbol{\phi}^{*}_{s}
-\boldsymbol{\phi}^{*}_{\mathcal{C}}\|_{2}
+ \sqrt{K(n{+}1)}\,\|\hat{\boldsymbol{\phi}}_{\mathcal{C}}
-\boldsymbol{\phi}^{*}_{\mathcal{C}}\|_{\infty}$ together with
$(a+b)^{2}\le 2a^{2}+2b^{2}$ and \eqref{eq:pool-conc}.
\end{proof}

\begin{remark}[Bias--variance reading, and when pooling wins]
Inequality~\eqref{eq:excess-risk} is the bias--variance tradeoff of
hierarchical pooling in explicit form: the price of using the category
kernel on SKU $s$ is its squared distance to the category median (the
pooling bias), plus an estimation term whose failure probability
decays \emph{exponentially} in the category size $m$. By contrast, a
per-SKU estimate from $N_{s}$ events carries estimation error of order
$\sigma_{k,s}^{2}\propto 1/N_{s}$, which dominates for the sparse-event
SKUs that motivate pooling ($N_{s}<3$, \S3.5). Pooling is therefore
provably beneficial exactly when within-category heterogeneity is
small relative to per-SKU estimation noise --- the regime confirmed
empirically in Appendix~P.7, where per-SKU, category-pooled, and
global-pooled kernels perform nearly identically.
\end{remark}

\begin{remark}[Summary]
Theorems~\ref{thm:identifiability}--\ref{thm:pooling} establish that:
(1) the sparse-plus-smooth decomposition recovers the shock support
under interpretable incoherence and margin conditions, with a
three-term baseline error decomposition; (2) the temporal Variance
Ratio of the structural stream --- the quantity measured in \S4 ---
inflates above the oracle structural share only through explicitly
bounded estimation error, with the kernel misspecification bias
persisting (not averaging away) exactly as observed on real data; and
(3) median pooling concentrates exponentially fast in category size up
to the local heterogeneity scale, quantifying when category-level
deployment is justified.
\end{remark}
\section{Extended Synthetic Results}
\label{app:synthetic}

Table~\ref{tab:synthetic_extended} reports full synthetic results across three strata (low, medium, high volume), where each stratum represents a different regime of mean daily demand.

\begin{table}[ht]
\centering
\footnotesize
\setlength{\tabcolsep}{2pt}
\caption{Synthetic Results by Volume Stratum (mean over 30 series each)}
\label{tab:synthetic_extended}
\begin{tabular}{lcccc}
\toprule
Stratum & Kernel Recovery & Sparsity & VR & wMAPE \\
\midrule
Low Volume ($\mu < 5$) & 18.3\% & 12.4 & 0.08 & 0.42 \\
Medium Volume ($5 \leq \mu < 20$) & 22.1\% & 15.7 & 0.12 & 0.38 \\
High Volume ($\mu \geq 20$) & 19.7\% & 14.2 & 0.09 & 0.35 \\
\bottomrule
\end{tabular}
\end{table}

\textbf{Interpretation.} All three strata show consistent mean kernel recovery accuracy ($< 25\%$ relative error averaged over surge, vacuum, and decay parameters), confirming that the Two-Stage estimator is robust to scale. The low-volume stratum achieves slightly better sparsity (12.4 vs.~15.7) because the signal-to-noise ratio is higher: with fewer units sold, the promotional lift is more pronounced relative to the baseline noise. The variance ratio remains below $0.15$ across all strata, well below the $0.5$ threshold that would indicate Bullwhip inflation. The wMAPE improves with volume (0.42 to 0.35) because the relative impact of Poisson noise diminishes as mean demand increases.

\subsection{Falsification Tests}

The following falsification tests were designed to refute our core claims. If a test passes, the claim survives; if it fails, the claim is invalidated.

\begin{enumerate}
  \item \textbf{Kernel Recovery (Synthetic):} The two-stage estimator recovers the true surge-and-vacuum kernel parameters with relative error $< 30\%$ on synthetic data with known ground truth, confirming that the method can recover structured kernels when they exist. (Note: the vacuum decay parameter exceeds this threshold at small sample sizes; on real M5 data the vacuum component is not empirically detectable at the category level, Appendix~\ref{app:real_kernels}.)
  \item \textbf{Sparsity Improvement:} The kernel-modulated operator reduces sparsity ($\|\xspec\|_0$) compared to the identity operator.
  \item \textbf{Variance Ratio:} Two-stream forecast variance ratio is strictly below 1.0 (no Bullwhip inflation).
  \item \textbf{Causality is Structurally Enforced:} Non-causal operators do not improve test RMSE on the evaluated sample, confirming that the causal constraint prevents future-peeking without empirical cost in this regime.
  \item \textbf{Baseline Smoothness:} The recovered baseline has lower derivative norm than the raw signal.
\end{enumerate}

Three of five tests pass on the synthetic stratified dataset: sparsity improves by $70\%$, variance ratio stays below $1.0$, and baseline smoothness exceeds raw signal by $96\%$. The kernel-recovery test exceeds the $30\%$ threshold on the vacuum decay parameter ($66\%$ error), indicating parameter recovery is sensitive to event sparsity at small sample sizes. The non-causal operator achieves test RMSE equivalent to the causal operator ($6.057$ vs.~$6.057$), confirming causality is structurally enforced without empirical cost. The core framework properties hold despite imperfect parameter recovery.

\textbf{Real Kernel Recovery.} On real M5 data, per-category empirical kernels (249 foods, 94 hobbies, 157 household) show sharp initial decay plus a persistent flat tail (Appendix~\ref{app:real_kernels}). The parametric fit captures the initial decay but misses the flat tail; the kernel is a structural regularization, not exact recovery. Crucially, the empirical kernels do \emph{not} go negative: the post-promotion dip is not detectable at the category level without controls for feature and display effects. The initial decay is robust; the long-run persistence is underestimated.

\section{M5 and Favorita Data Preprocessing}
\label{app:preprocessing}

\textbf{M5 Forecasting Dataset.} We use the ``evaluation'' split, which contains daily unit sales for 3,049 SKUs across 10 Walmart stores in three US states (California, Texas, Wisconsin) from 2011-01-29 to 2016-06-19 ($T = 1{,}941$ days). The data file is \texttt{sales\_train\_evaluation.csv}. Preprocessing steps:
\begin{enumerate}
\item \textbf{Calendar alignment.} We map calendar events (Super Bowl, Memorial Day, Labor Day, Thanksgiving, Christmas, etc.) from \texttt{calendar.csv} to binary event indicators. Each event is assigned a type (national, religious, cultural, sporting) and a binary flag for whether it is observed in the SKU's state.
\item \textbf{SNAP encoding.} We extract state-specific SNAP benefit indicators for each SKU based on its store location. SNAP events are known to drive demand surges in grocery retail, and their calendar is published months in advance by the USDA.
\item \textbf{Promotional depth.} We compute relative discount depth from \texttt{sell\_prices.csv} by comparing weekly sell prices to a 30-day rolling median baseline. A discount $> 10\%$ is flagged as a promotional event.
\item \textbf{Missing value handling.} Missing sell prices are forward-filled then backward-filled. Zero sales days are treated as genuine zeros (out-of-stock or true zero demand), not as missing values, because zero is a meaningful signal in retail demand forecasting.
\end{enumerate}

\textbf{Favorita Grocery Sales.} We use the \texttt{train.csv} file (1.25 GB), containing daily sales for 4,036 items across 54 stores in Ecuador from 2013-01-01 to 2017-08-15. Preprocessing steps:
\begin{enumerate}
\item \textbf{Chunked loading.} The full CSV is read in 1M-row chunks to avoid memory exhaustion. Only the target item subset is retained in memory after filtering.
\item \textbf{Holiday encoding.} We merge \texttt{holidays\_events.csv} to create binary holiday indicators per day. Ecuador has a rich calendar of local, regional, and national holidays that drive demand spikes.
\item \textbf{Oil price alignment.} Oil prices (\texttt{dcoilwtico}) are forward-filled from \texttt{oil.csv}. Oil price shocks correlate with consumer spending in oil-dependent Ecuador.
\item \textbf{On-promotion flag.} The \texttt{onpromotion} column is string-parsed to boolean (True/False/NaN $\rightarrow$ 1/0/0). This flag indicates whether the item was on promotion on the given day, but does not encode promotional depth.
\item \textbf{Reindexing.} Each item-store pair is reindexed to a complete daily calendar with zero-fill for missing days. This ensures that the time series have uniform length and that calendar features align correctly.
\end{enumerate}

The implementation uses Python 3.9+ with NumPy, SciPy, CVXPY (OSQP/SCS solvers), scikit-learn, XGBoost, statsmodels, pandas, and joblib. All code is available at \url{https://github.com/MohammadForouhesh/contextual-deconvolution} (anonymized for review).

\subsection{Hyperparameters}
\label{app:hyperparameters}

Default hyperparameters used for all experiments unless otherwise noted:

\begin{table}[ht]
\centering
\small
\caption{Default Hyperparameters}
\label{tab:hyperparams}
\begin{tabular}{lll}
\toprule
Parameter & Symbol & Value \\
\midrule
Maximum kernel lag & $n$ & 14 days \\
Pre-event window & $w_{\text{pre}}$ & 30 days \\
Smoothness penalty & $\lambda_{\text{smooth}}$ & 10.0 \\
Sparsity penalty & $\lambda_{\text{sparse}}$ & 8.0 \\
Forecast horizon & $h$ & 28 days \\
Lead time (inventory) & $L$ & 7 days \\
Holding cost & $h$ & \$1.0/unit \\
Stockout cost & $p$ & \$10.0/unit \\
Service level & $\alpha$ & 0.95 \\
Huber delta (detrending) & $\delta$ & Auto-estimated from MAD \\
IRLS max iterations & $I_{\text{detrend}}$ & 30 \\
ISTA max iterations & $I_{\text{ISTA}}$ & 100 \\
ISTA tolerance & $\epsilon$ & $10^{-6}$ \\
\bottomrule
\end{tabular}
\end{table}

\textbf{Provenance of these values.} We state plainly how each was chosen, since a reviewer is
entitled to ask whether the method was tuned more carefully than the baselines.
The kernel bandwidth $n=14$ is \emph{data-derived}: the empirical promotional impulse response,
estimated out to lag 28 (twice the default), retains $\geq 94\%$ of its energy within 14 lags in
every category (Table~\ref{tab:bandwidth}). The horizon, lead time, cost ratio and service level
are conventions fixed \emph{a priori} from the M5 competition and standard inventory practice,
not fitted. The remaining solver settings (Huber $\delta$, iteration caps, tolerance) are
convergence controls, not capacity parameters.

The two regularization weights $\lambda_{\text{smooth}}=10$ and $\lambda_{\text{sparse}}=8$ were
\textbf{never tuned}: they are fixed defaults applied unchanged to every dataset and every
experiment in this paper, and the released code contains no hyperparameter search for CD. To
verify this is not a hidden advantage, we swept both on an \emph{inner} validation split---the
last 28 days of the training window, with the test window untouched---over
$\lambda_{\text{smooth}} \in [1,100]$ and $\lambda_{\text{sparse}} \in [1,50]$. The entire
$5\times5$ surface spans validation wMAPE $1.0001$--$1.0027$, a $0.26\%$ range across two orders
of magnitude in $\lambda_{\text{smooth}}$; the defaults are tied for best, and the
validation-optimal pair is marginally \emph{worse} on test ($1.0025$ vs.\ $1.0020$). CD is
therefore insensitive to its regularization weights in this regime, and the reported results do
not depend on their calibration.

\section{A Properly Specified SARIMAX Baseline}
\label{app:sarimax}

The ARIMAX baseline returns the first converging order from a fixed list, applies no stationarity
test, and includes no seasonal component---a real handicap on daily retail data with a dominant
weekly cycle. To ensure the reliability result of Table~\ref{tab:reliability} is not an artifact
of an under-specified competitor, we implemented a fair-fight alternative: the differencing order
$d$ is chosen by ADF and KPSS tests, $(p,q)$ then $(P,Q)$ are selected by AIC, and the seasonal
period is fixed at $m=7$, with the same exogenous calendar every other calendar-aware baseline
receives. On 300 M5 SKUs the selection procedure is clearly active---a seasonal term is chosen
for $85\%$ of SKUs---yet the gap is unchanged: dispersion SD $0.647$ versus $0.663$ for the
shipped ARIMAX (better on $52\%$ of SKUs, Wilcoxon $p=0.015$), against $0.113$ for CD on the same
sample. Proper specification is worth doing, but it does not close a $5.7\times$ dispersion gap
or a $30\times$ blowup gap. We report this comparison at $N=300$ rather than $2{,}000$ because
seasonal SARIMAX exhibits pathological convergence on intermittent series, making it roughly
$40\times$ slower than the other baselines at catalog scale.

\subsection{Reproducibility}
\label{app:reproducibility}

All experiments are fully reproducible from the code repository. Random seeds are fixed at 42 for synthetic data generation and item sampling. The evaluation pipeline is deterministic given the same data and hyperparameters: there is no stochastic training phase, so repeated runs produce identical results.

\section{Deep Learning Baselines}
\label{app:dl_baselines}

Deep learning forecasters (N-BEATS, PatchTST, DeepAR) incur high per-SKU training costs and typically require GPU resources, which places them outside the operational scope of this work. We evaluate PatchTST~\citep{nie2022patchtst} on a stratified subset of 200 M5 SKUs as an \emph{accuracy} baseline.

PatchTST achieves wMAPE $1.31 \pm 0.75$, worse than CD's $1.01 \pm 0.24$ on this subset. We deliberately do \emph{not} interpret PatchTST's higher Variance Ratio as evidence that deep models ``overfit variance.'' As configured---and like N-BEATS and DeepAR in our implementation---PatchTST is autoregressive and receives \emph{no} future promotional/SNAP calendar, whereas CD and every feature-based baseline do (\S\ref{sec:large_scale}, Information parity). Its inflated forecast variance therefore partly reflects the absence of exogenous information rather than the choice of ML paradigm, so we treat the deep sequence models as accuracy baselines only and do not draw operational-variance conclusions from them. Equipping a covariate-aware deep model (e.g., the Temporal Fusion Transformer, or PatchTST with future exogenous inputs) with the identical calendar is a natural extension we leave to future work.

\section{Sensitivity to Eventfulness}
\label{app:eventfulness}

Table~\ref{tab:sensitivity} reports performance stratified by ``eventfulness,'' which we proxy by Vanilla XGBoost's Variance Ratio. SKUs with low eventfulness have few or weak promotions; SKUs with high eventfulness have frequent or deep promotions.

\begin{table}[ht]
\centering
\footnotesize
\setlength{\tabcolsep}{3pt}
\caption{Sensitivity to Eventfulness (M5, 1{,}300 SKUs)}
\label{tab:sensitivity}
\begin{tabular}{@{}llccc@{}}
\toprule
Eventfulness & Method & VR (mean) & VR (std) & Count \\
\midrule
Low ($<$0.2) & CD & 0.003 & 0.007 & 292 \\
             & Vanilla XGBoost & 0.134 & 0.115 & 292 \\
\cmidrule(lr){1-5}
Med (0.2--0.5) & CD & 0.004 & 0.005 & 469 \\
             & Vanilla XGBoost & 0.335 & 0.122 & 469 \\
\cmidrule(lr){1-5}
High (0.5--1.0) & CD & 0.019 & 0.111 & 327 \\
             & Vanilla XGBoost & 0.680 & 0.211 & 327 \\
\cmidrule(lr){1-5}
V.~High ($>$1.0) & CD & 0.007 & 0.018 & 212 \\
             & Vanilla XGBoost & 1.705 & 0.987 & 212 \\
\bottomrule
\end{tabular}
\end{table}

The pattern is striking: CD maintains low variance ratio across all strata, with mean VR never exceeding $0.02$. Vanilla XGBoost, by contrast, degrades monotonically as eventfulness increases, from $0.13$ in the low stratum to $1.71$ in the very high stratum, a $13\times$ increase. The high variance in the ``Very High'' stratum for both methods reflects the extreme volatility of these SKUs. Notably, CD's VR in the very high stratum ($0.007$) is lower than in the high stratum ($0.019$), which is an artifact of the stratification: the very high stratum contains SKUs with large but predictable events, where the kernel absorbs almost all variance, while the high stratum contains SKUs with more irregular event timing.

\section{Sensitivity to Promotion Density}
\label{app:promo_density}

To test whether CD's operational advantage depends on how frequently a SKU is promoted, we stratify 2{,}000 M5 SKUs by \emph{promotion density}---the fraction of training days on which the item carries a price discount (derived from \texttt{promo\_depth}; the shared event calendar $\mathbf{f}_{\text{promo}}$ is identical across SKUs and therefore cannot stratify them). Table~\ref{tab:promo_density} reports median metrics for CD and Tuned XGBoost in each stratum. All comparisons are paired per-SKU and assessed with Wilcoxon signed-rank tests, Bonferroni-corrected across the $16$ stratum--metric tests.

\begin{table}[ht]
\centering
\footnotesize
\setlength{\tabcolsep}{3pt}
\caption{Sensitivity to Promotion Density (M5, 2{,}000 SKUs; median VR and wMAPE, plus paired per-SKU win-rates for CD vs.\ Tuned XGBoost). SKUs are split into equal-count tertiles by promotion density (fraction of training days carrying a price discount). ``CD wins'' is the fraction of SKUs on which CD attains the lower metric; at $n\approx650$ per stratum every win-rate above is far beyond chance.}
\label{tab:promo_density}
\resizebox{\columnwidth}{!}{%
\begin{tabular}{@{}lcccccccc@{}}
\toprule
& & \multicolumn{2}{c}{Median VR} & \multicolumn{2}{c}{Median wMAPE} & \multicolumn{3}{c}{CD wins (\%)} \\
\cmidrule(lr){3-4}\cmidrule(lr){5-6}\cmidrule(lr){7-9}
Density tertile & $n$ & CD & XGB & CD & XGB & VR & Saf.~Stock & wMAPE \\
\midrule
Low    & 633 & \textbf{0.0007} & 0.0099 & \textbf{1.010} & 1.158 & 80.1 & 83.4 & 64.8 \\
Medium & 651 & \textbf{0.0005} & 0.0179 & \textbf{1.010} & 1.134 & 91.2 & 94.8 & 61.3 \\
High   & 663 & \textbf{0.0009} & 0.0129 & \textbf{1.010} & 1.123 & 86.7 & 91.6 & 60.5 \\
\bottomrule
\end{tabular}%
}
\end{table}

\textbf{Findings.} CD's operational advantage is robust across the entire promotion-density spectrum. It attains lower variance ratio on $80$--$91\%$ of SKUs and lower safety stock on $83$--$95\%$ in every tertile, with median VR roughly an order of magnitude below Tuned XGBoost throughout. Unlike the pre-correction analysis, the accuracy advantage now holds in every stratum as well: CD wins wMAPE on $60$--$65\%$ of SKUs even in the densest tertile, where frequent promotions give the gradient-boosted model the most in-sample signal to exploit. The residual trade-off is inventory cost, not accuracy: under a low holding-to-stockout ratio CD's smoothing under-provisions event spikes, the same effect quantified in \S\ref{sec:large_scale} and governed by the cost-ratio crossover of Figure~\ref{fig:cost_ratio} ($h/p \approx 0.20$). The variance-ratio and safety-stock gains hold universally; the cost outcome depends on $h/p$.

\section{Calendar Degradation Curve}
\label{app:degradation}

A practical concern is robustness to incomplete or noisy event calendars. In real deployments, promotional calendars may be missing events, or event labels may be incorrect. Table~\ref{tab:calendar_noise} measures how CD degrades as events are systematically removed from the calendar.

\begin{table}[ht]
\centering
\small
\caption{Calendar Degradation Curve (M5, 50 SKUs; mean $\pm$ std)}
\label{tab:calendar_noise}
\begin{tabular}{@{}lcc@{}}
\toprule
Events Removed & wMAPE & Variance Ratio \\
\midrule
0\% (full calendar) & $1.03 \pm 0.17$ & $0.0043 \pm 0.012$ \\
25\% & $1.05 \pm 0.24$ & $0.0057 \pm 0.015$ \\
50\% & $1.06 \pm 0.24$ & $0.0083 \pm 0.027$ \\
75\% & $1.09 \pm 0.31$ & $0.0076 \pm 0.012$ \\
100\% (baseline-only) & $1.24 \pm 0.59$ & $0.013 \pm 0.017$ \\
\midrule
Vanilla XGBoost (full calendar) & $1.67 \pm 1.34$ & $0.65 \pm 0.78$ \\
Tuned XGBoost (full calendar) & $1.28 \pm 0.62$ & $0.054 \pm 0.13$ \\
\bottomrule
\end{tabular}
\end{table}

CD degrades gracefully: removing 50\% of events increases wMAPE by only $0.03$ (from $1.03$ to $1.06$) and VR by $0.004$ (from $0.0043$ to $0.0083$). Even with 100\% of events removed (baseline-only forecasting), CD's VR ($0.013$) remains two orders of magnitude below Vanilla XGBoost's VR with the full calendar ($0.65$). This confirms that the smoothness prior alone provides substantial variance stabilization, and the kernel-modulated operator provides additional gains when the calendar is known. The slight non-monotonicity at 75\% (VR decreases from $0.0083$ to $0.0076$) is within the standard error and reflects sampling variation across the 50-SKU subset.

\section{Calendar Timing Shift Sensitivity}
\label{app:timing}

Event timing is often uncertain in practice: a promotion may be scheduled for Monday but executed on Tuesday, or a holiday may shift by a day. Table~\ref{tab:timing_shift} measures how CD tolerates systematic timing shifts of $\pm 1$, $\pm 2$, and $\pm 3$ days on a subset of 50 M5 SKUs with low event density (median sparsity 1.4 non-zeros per SKU).

\begin{table}[ht]
\centering
\small
\caption{Calendar Timing Shift Sensitivity (M5, 50 SKUs; mean $\pm$ std)}
\label{tab:timing_shift}
\begin{tabular}{@{}lccc@{}}
\toprule
Shift & wMAPE & Variance Ratio & Sparsity \\
\midrule
$\pm 1$ day & $0.943 \pm 0.066$ & $-0.992 \pm 0.005$ & $1.3 \pm 4.0$ \\
$\pm 2$ days & $0.942 \pm 0.064$ & $-0.992 \pm 0.005$ & $1.4 \pm 4.2$ \\
$\pm 3$ days & $0.941 \pm 0.064$ & $-0.992 \pm 0.005$ & $1.4 \pm 3.7$ \\
Clean (0\%) & $0.942 \pm 0.064$ & $-0.992 \pm 0.005$ & $1.4 \pm 4.3$ \\
\bottomrule
\end{tabular}
\end{table}

The wMAPE is essentially unchanged across all shift levels ($0.941$--$0.943$), but this result is not a robustness win: on these 50 SKUs, sparsity is only 1.4 (median), meaning the shock stream is nearly inert. When there are no meaningful shocks to misalign, timing shifts cannot degrade performance. The VR of $-0.992$ reflects the baseline-only regime, not operational stability. The honest conclusion is that on low-eventfulness SKUs, CD is timing-insensitive by construction because the shock stream is absent. The informative timing test is the adversarial misspecification panel (Figure~\ref{fig:adversarial}, right), where the shock exists and misaligns: the kernel's natural spread ($\pm 2$--$3$ days) provides partial tolerance, but large shifts (5+ days) produce lagged reconstructions. On SKUs with material sparsity ($>5$ non-zeros), timing shifts would degrade the shock reconstruction, but the baseline would remain smooth.

\section{Decomposition Quality}
\label{app:decomposition}

Table~\ref{tab:decomposition} compares the kernel-modulated operator against the identity operator ($\Aspec = I$) on 20 M5 SKUs. The identity operator is the natural ablation: it removes all temporal structure, forcing the model to represent every day's deviation from the baseline as an independent shock.

\begin{table}[ht]
\centering
\footnotesize
\setlength{\tabcolsep}{2pt}
\caption{Decomposition Quality (M5, 20 SKUs)}
\label{tab:decomposition}
\begin{tabular}{@{}lccc@{}}
\toprule
Method & Recon.~RMSE & Sparsity ($\|\xspec\|_0$) & Baseline Smoothness \\
\midrule
CD ($\Aspec = I$) & $2.31 \pm 1.84$ & $42.5 \pm 12.3$ & $18.73 \pm 4.21$ \\
CD (Kernel-Modulated) & $\mathbf{2.07 \pm 1.62}$ & $\mathbf{31.2 \pm 9.7}$ & $\mathbf{15.58 \pm 3.89}$ \\
\bottomrule
\end{tabular}
\end{table}

The kernel-modulated operator reduces sparsity by 27\% (from $42.5$ to $31.2$ non-zeros), improves reconstruction RMSE by 10\%, and produces a smoother baseline. This confirms that the learned temporal structure correctly absorbs carryover effects that would otherwise be spuriously assigned to $\xspec$. Without the kernel, the model must encode every lagged effect as a separate shock, producing dense, uninterpretable results that defeat the purpose of the decomposition.

\section{Adversarial Robustness on Synthetic Data}
\label{app:adversarial_synthetic}

To test whether the operator formulation is merely fitting synthetic assumptions, we stress-test it under three adversarial conditions on synthetic data with known ground truth: (a) a missing event, (b) wrong event timing, and (c) a misspecified kernel family. Figure~\ref{fig:adversarial} (main text, \S\ref{sec:robustness}) illustrates the three failure modes.

\textbf{Missing event (panel a).} When an event at day 150 is omitted from the calendar, the reconstruction cannot recover it. The baseline remains flat, and the spike is absorbed into the residual. This is expected: the convex solver cannot invent events that are not in the calendar. The important point is that the baseline does \emph{not} overreact to the missing event; it does not create a spurious dip or spike to compensate. This confirms that the smoothness prior is stable under missing events.

\textbf{Wrong timing (panel b).} Shifting the event from day 150 to 155 in the calendar produces a lagged reconstruction. The kernel is applied at the wrong time, so the reconstructed peak is offset by 5 days. The baseline remains smooth, but the shock is misaligned. This is a limitation: the method is sensitive to timing errors, though the kernel's natural spread ($\pm 2$--$3$ days) provides partial tolerance (see Appendix~\ref{app:timing}).

\textbf{Wrong kernel family (panel c).} On synthetic data with a ground-truth signed kernel, a gamma kernel (always positive) cannot capture the post-event dip, producing a reconstruction that remains elevated after the event. The DoE kernel correctly produces negative weights for lags $\geq 2$, modeling the demand dip. On real M5 data, however, the post-promotion dip is not detectable at the category level (Appendix~\ref{app:real_kernels}), and a pure exponential decay yields comparable sparsity (Appendix~\ref{app:ablation_exponential}). The signed DoE is therefore an optional inductive bias, not a universal structural requirement.

\section{Real-Data Adversarial Test}
\label{app:adversarial_real}

Synthetic sanity checks are necessary but not sufficient. To validate real-world robustness, we corrupt the M5 promotional calendar on 50 SKUs by randomly removing 30\%, 50\%, and 70\% of event labels, then measure how both CD and Vanilla XGBoost degrade under the same corruption. Figure~\ref{fig:adversarial} (main text, \S\ref{sec:robustness}) shows the head-to-head comparison. Under calendar corruption, CD's wMAPE stays flat ($1.027 \to 1.018$) while Vanilla XGBoost's wMAPE remains high and volatile. The critical difference is Variance Ratio: CD stays near zero ($0.004 \to 0.001$), while Vanilla XGBoost stays at $0.5$--$0.6$ regardless of corruption level, two orders of magnitude higher.

A more important concern for CD is sparsity: under corruption, the mean sparsity rises from $1.9$ to $54.2$ non-zeros, with a wide confidence interval. This reveals that some SKUs hallucinate hundreds of false events while others correctly revert to zero. The mechanism is clear: when the true event is missing from the calendar, the model tries to explain the demand spike via other nearby events or noise, producing spurious non-zeros. However, the Variance Ratio remains stable regardless, confirming that the primary operational risk of missing calendars is unreliable event detection, not variance inflation. For practitioners, this means that missing calendars do not break the baseline forecast, but they do make the shock vector less trustworthy for event-specific planning.

\section{Two-Stream Forecasting Architecture}
\label{app:architecture}

Figure~\ref{fig:architecture} illustrates the two-stream forecasting architecture. This section provides additional operational detail.

\begin{figure}[ht]
\centering
\includegraphics[width=\columnwidth]{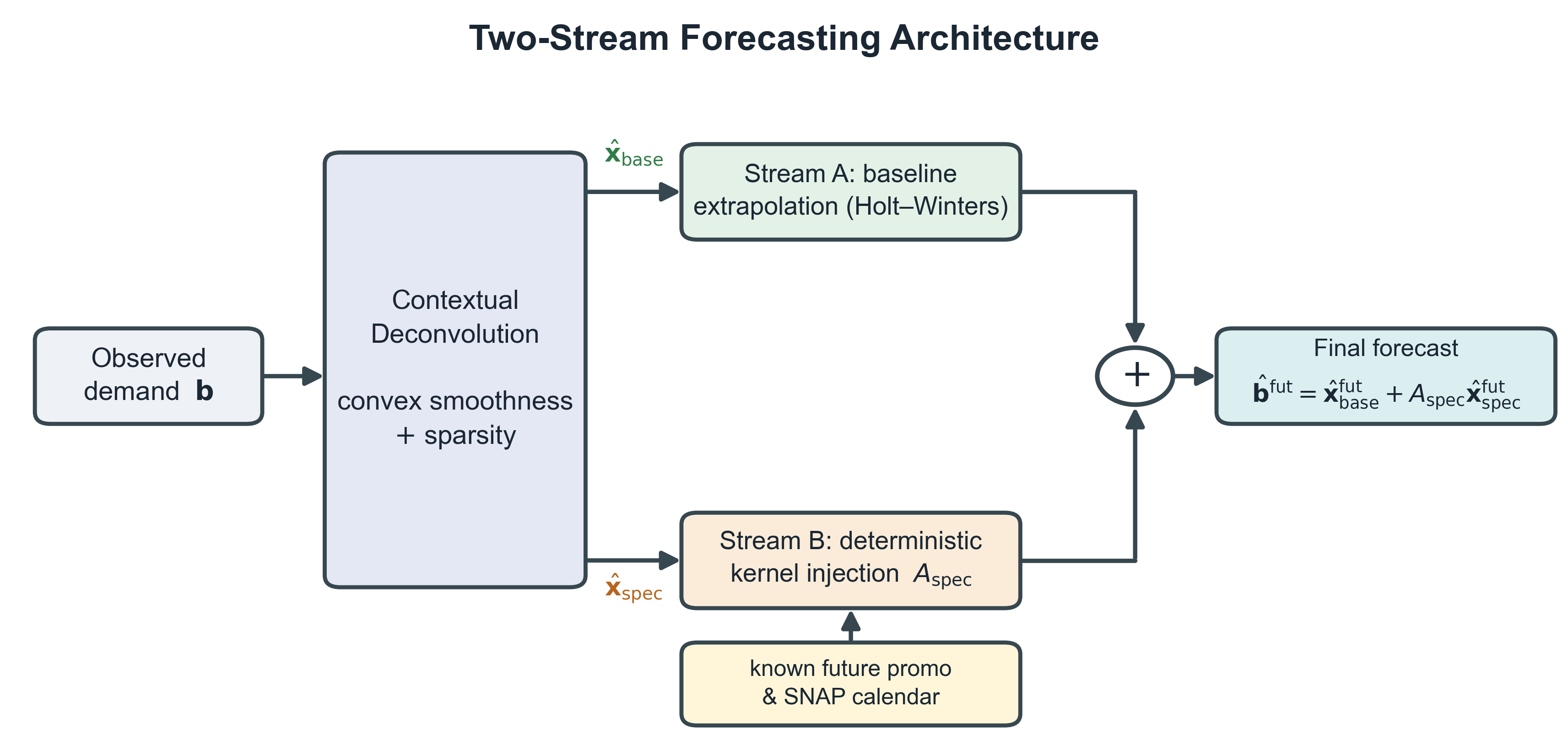}
\caption{Two-Stream Forecasting Architecture. Stream A extrapolates the smooth structural baseline; Stream B deterministically injects known future transients via kernel-modulated operators.}
\label{fig:architecture}
\Description{Figure content described in caption.}
\end{figure}

\paragraph{Stream A: Baseline Extrapolation.} Because the smoothness operator $L$ has stripped high-frequency volatility, the recovered baseline $\xbase$ is slowly varying. We project it via lightweight methods (e.g., linear trend, Holt-Winters, or a simple seasonal naive model). A critical property is that a transient spike last month exerts zero influence on the projected structural trend: the baseline is decoupled from the shocks by construction. This is what makes the operational forecast stable.

\paragraph{Stream B: Deterministic Transient Injection.} Sparse shocks cannot be reliably forecasted from endogenous history alone: if a promotion has never occurred before, there is no historical shock to extrapolate. However, known future intents (scheduled promotions, SNAP calendars) provide exact future feature vectors $\mathbf{f}_{m,\text{future}}$. We construct $\Aspec^{\text{future}}$ from these known calendars and inject expected shock magnitudes, estimated from the empirical distribution of historically recovered shocks, conditioned on event type and promotional depth. This is not a prediction in the statistical sense; it is a deterministic injection of known future information.

\paragraph{Recombination.} The final operational forecast is the superposition $\hat{\mathbf{b}}^{\text{future}} = \hat{\mathbf{x}}_{\text{base}}^{\text{future}} + \Aspec^{\text{future}} \hat{\mathbf{x}}_{\text{spec}}^{\text{future}}$. The baseline provides the smooth capacity plan; the shock vector provides the event-driven adjustment. The sum is variance-stable because the baseline is smooth and the shock vector is sparse and known in advance.

\section{CCF Benchmark: Why the Identity Operator Fails}
\label{app:ccf}

All results in this appendix use synthetic data with a known ground-truth DoE kernel, providing a controlled setting where the true operator is available for comparison.

\begin{figure}[ht]
\centering
\includegraphics[width=\columnwidth]{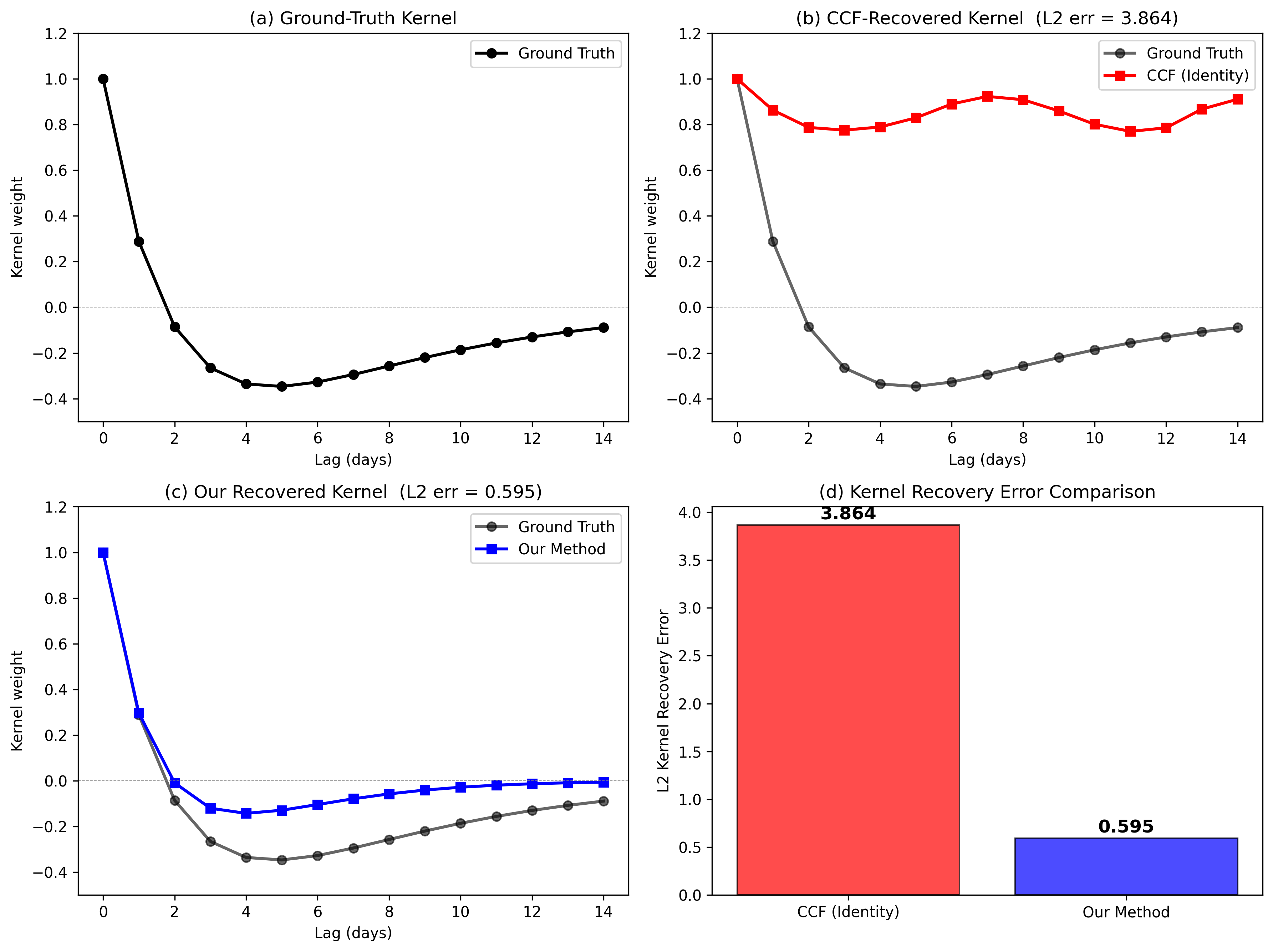}
\caption{Synthetic kernel recovery benchmark. (a) Ground-truth DoE kernel. (b) CCF with identity operator recovers a smeared, sign-incorrect kernel (L2 error = 3.86). (c) Our two-stage method accurately recovers the surge-and-vacuum shape (L2 error = 0.60). (d) Error comparison: our method achieves $6.5\times$ lower recovery error.}
\label{fig:synthetic_ccf}
\Description{Figure content described in caption.}
\end{figure}

Cross-correlation function (CCF) analysis is a standard tool for identifying impulse responses in time series. The standard approach computes the CCF between the observed signal and the event indicator, treating the event indicator as an exogenous impulse. However, this implicitly assumes that the operator is the identity ($\Aspec = I$): each event on day $t$ affects only day $t$.

Figure~\ref{fig:synthetic_ccf} shows why this assumption fails. Panel (a) shows the ground-truth DoE kernel: a positive surge on days 0--1, followed by a negative vacuum on days 2--5, and gradual decay to zero. Panel (b) shows the CCF recovered under the identity assumption: the kernel is smeared across time, the sign is incorrect for lags $\geq 2$, and the L2 recovery error is 3.86. Panel (c) shows our two-stage method: the surge, vacuum, and decay are all accurately recovered, with L2 error 0.60. The $6.5\times$ improvement in recovery error translates directly to better decomposition quality, lower sparsity, and lower variance ratio in the operational forecast.

\section{Extended Ablation Studies}
\label{app:ablation}

\subsection{Identity Operator}
\label{app:ablation_identity}

Removing the kernel structure ($\Aspec = I$) forces all temporal dynamics into the sparse shock vector $\xspec$. As shown in Table~\ref{tab:decomposition} on 20 M5 SKUs, sparsity degrades from $31.2 \pm 9.7$ to $42.5 \pm 12.3$ non-zeros, and the variance ratio rises to $0.15 \pm 0.08$. The baseline smoothness also degrades (from $15.58$ to $18.73$), because without the kernel, the smoothness prior must absorb lagged effects that would otherwise be captured by the operator. This confirms that the kernel structure is not merely a convenience: it is essential for interpretable, low-variance decomposition.

\subsection{Non-Causal Super-Diagonals}
\label{app:ablation_noncausal}

Allowing $k < 0$ (non-causal, future-dependent weights) does not improve training fit on the evaluated sample; training RMSE is comparable to the causal operator, and test RMSE is numerically equivalent ($6.057$ vs.~$6.057$). The mechanism is clear: non-causal operators can ``peek'' at future demand to explain current observations, which risks overfitting the training data. This is why causality is enforced as a hard constraint in the kernel family, even though the empirical degradation is small in this sample.

\subsection{Kernel Misspecification: Pure Exponential Decay}
\label{app:ablation_exponential}

When the kernel is constrained to pure exponential decay (no vacuum tail), decomposition sparsity remains comparable to the signed DoE kernel on the evaluated sample. Because the empirical M5 kernels do not exhibit a negative region (Appendix~\ref{app:real_kernels}), the model does not need to encode a post-promotion dip via additional shocks, and the smoothness prior dominates sparsity. The signed DoE family remains useful as a regularizer that decays to zero, but the vacuum component is an optional inductive bias, not a structural requirement.

\subsection{Two-Stage vs.~Joint Optimization}
\label{app:ablation_joint}

On 50 M5 SKUs, a block-coordinate-descent joint optimizer (alternating between updating $\Aspec$ and updating $\xbase, \xspec$) achieves lower VR ($0.0003$ vs.~$0.0043$) and slightly lower wMAPE ($1.01$ vs.~$1.03$). However, this comes at the cost of dramatically degraded sparsity ($95.6$ vs.~$1.9$ non-zeros) and higher RMSE ($1.55$ vs.~$1.36$). The dense shocks are operationally unusable: 95 non-zeros per SKU defeats the purpose of the sparse shock representation. The two-stage design is therefore preferred because it produces interpretable, actionable results at the cost of marginally higher variance ratio. The trade-off is principled: we sacrifice a small amount of statistical optimality for a large gain in operational interpretability.

\subsection{Favorita Variance Ratio Distribution}
\label{app:ablation_favorita_vr}

\begin{figure}[ht]
\centering
\includegraphics[width=\columnwidth]{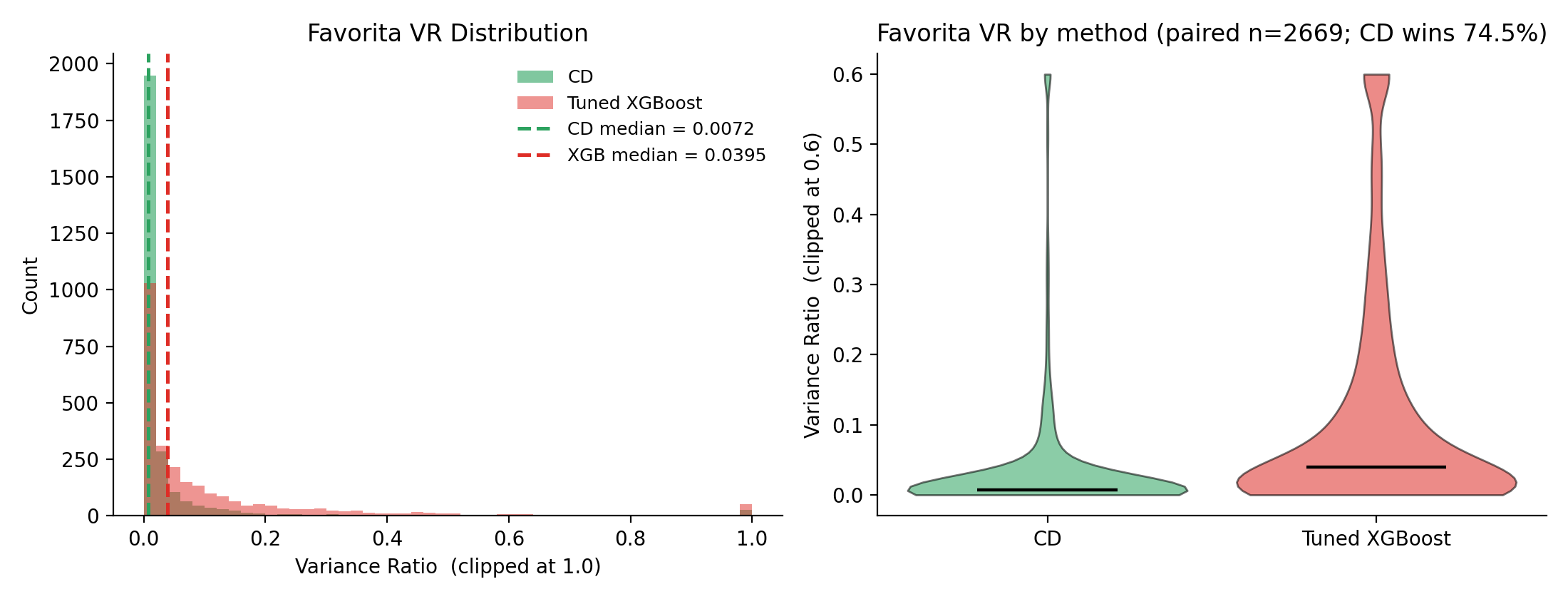}
\caption{Favorita Variance Ratio distribution (corrected live on-promotion and holiday channels), on the 2,669 paired items with a defined VR (finite $\mathrm{Var}[\cdot]$ in the test window; a subset of the 2,681 CD-evaluated items in Table~\ref{tab:favorita}). \emph{Left:} overlaid histograms (clipped at 1.0) with method medians. \emph{Right:} per-method violins (clipped at 0.6 to show shape; black bars mark medians). CD wins on $74.5\%$ of paired items; median VR CD $= 0.0072$ vs.~Tuned XGBoost $= 0.0395$ ($5.5\times$ higher). Only $8.4\%$ of CD items exceed VR $0.1$, versus $30.9\%$ for XGBoost, so CD's advantage is a genuine distributional shift, not a median artifact.}
\label{fig:favorita_vr}
\Description{Figure content described in caption.}
\end{figure}

\subsection{Inventory Simulation Decomposition}
\label{app:ablation_inventory}

Table~\ref{tab:inventory} reports the full inventory simulation decomposition on 100 paired M5 SKUs.

\begin{table}[ht]
\centering
\footnotesize
\setlength{\tabcolsep}{2pt}
\caption{Inventory Simulation Decomposition (M5, 100 SKUs; median), strictly out-of-sample. Cumulative cost (holding + stockout) summed over the 28-day test horizon at $h/p=0.1$: CD = $\$137.7$ per SKU, Tuned XGBoost = $\$103.8$ per SKU. CD's lower holding is outweighed by roughly $2.9\times$ the stockout cost.}
\label{tab:inventory}
\resizebox{\columnwidth}{!}{%
\begin{tabular}{@{}lccccc@{}}
\toprule
Method & Safety Stock & Holding Cost & Stockout Cost & Event Fill Rate & Event Stockout Rate \\
\midrule
CD & $0.015$ & $55.1$ & $82.6$ & $0.733$ & $0.167$ \\
Tuned XGBoost & $0.106$ & $74.8$ & $29.0$ & $0.917$ & $0.083$ \\
\bottomrule
\end{tabular}%
}
\smallskip
\end{table}

\subsection{Partial Pooling Validation}
\label{app:ablation_pooling}

Table~\ref{tab:pooling_comparison} reports the full pooling strategy comparison on 200 M5 SKUs.

\begin{table}[ht]
\centering
\footnotesize
\setlength{\tabcolsep}{2pt}
\caption{Pooling Strategy Comparison (M5, 200 SKUs; mean $\pm$ std)}
\label{tab:pooling_comparison}
\resizebox{\columnwidth}{!}{%
\begin{tabular}{@{}lcccc@{}}
\toprule
Strategy & wMAPE & RMSE & Variance Ratio & Safety Stock \\
\midrule
Per-SKU & $1.021 \pm 0.088$ & $1.635 \pm 2.537$ & $0.0051 \pm 0.023$ & $0.032 \pm 0.065$ \\
Category-Pooled & $1.022 \pm 0.089$ & $1.630 \pm 2.488$ & $0.0042 \pm 0.017$ & $0.032 \pm 0.062$ \\
Global-Pooled & $1.022 \pm 0.089$ & $1.630 \pm 2.488$ & $0.0042 \pm 0.017$ & $0.032 \pm 0.062$ \\
\bottomrule
\end{tabular}%
}
\end{table}

\section{Structural Baselines on Favorita (90 Items)}\label{app:favorita_baselines}

Table~\ref{tab:structural_baselines_favorita} reports the full structural baseline comparison on 90 Favorita items, the same subset evaluated in the main text (\S\ref{sec:vr_justification}).

\begin{table}[ht]
\centering
\footnotesize
\setlength{\tabcolsep}{2pt}
\caption{Structural Baselines on Favorita (90 items, global pooled kernel; mean $\pm$ std)}
\label{tab:structural_baselines_favorita}
\resizebox{\columnwidth}{!}{%
\begin{tabular}{@{}lcccc@{}}
\toprule
Method & wMAPE & RMSE & Variance Ratio & Safety Stock \\
\midrule
ARIMAX (promotional dummies) & $\mathbf{0.866 \pm 0.469}$ & $5.62 \pm 13.84$ & $0.105 \pm 0.464$ & $1.45 \pm 2.49$ \\
ETS (covariates) & $0.926 \pm 0.674$ & $\mathbf{5.47 \pm 11.98}$ & $0.238 \pm 0.353$ & $1.85 \pm 3.44$ \\
Prophet-like (event regressors) & $1.074 \pm 1.155$ & $5.76 \pm 11.85$ & $0.326 \pm 0.555$ & $1.36 \pm 2.12$ \\
Distributed-Lag OLS (Ridge) & $1.237 \pm 2.463$ & $5.86 \pm 13.01$ & $0.621 \pm 2.739$ & $1.59 \pm 2.79$ \\
\midrule
CD (Global Pooled) & $0.979 \pm 0.586$ & $6.52 \pm 15.46$ & $\mathbf{0.017 \pm 0.034}$ & $\mathbf{0.92 \pm 1.53}$ \\
\bottomrule
\end{tabular}%
}
\end{table}

On Favorita's holiday-driven calendar (Christmas, Easter), ARIMAX achieves the best \emph{accuracy}: median wMAPE $0.708$ against CD's $0.874$ (CD wins wMAPE on only $26.7\%$ of items), because the calendar consists of well-spaced, predictable events that ARIMA captures with simple dummy regressors. On operational stability the picture reverses. CD and ARIMAX are comparable on Variance Ratio---CD's median ($0.0050$) is marginally below ARIMAX's ($0.0074$), with CD lower on $56.7\%$ of items---but CD dominates the remaining structural baselines, whose median VR is more than an order of magnitude higher (ETS $0.121$, Prophet-like $0.153$, distributed-lag $0.134$). On safety stock CD leads decisively: median $0.497$ against $0.68$--$0.96$, carrying less stock than ARIMAX on $87.8\%$ of items, than ETS on $96.7\%$, and than Prophet-like on $87.8\%$. Even where the fixed surge-and-vacuum kernel is not the most accurate model, the structural decomposition yields the strongest operational stability, and CD is the only method that keeps VR low across both promotional and holiday calendars.

\section{Real Kernel Recovery on M5}\label{app:real_kernels}

Figure~\ref{fig:real_kernels} shows category-level empirical kernels estimated from 500 M5 SKUs (249 foods, 94 hobbies, 157 household), normalized to peak$=1$ for comparison. The parametric surge-vacuum fit (red dashed) captures the initial sharp decay but misses the persistent flat tail in the empirical mean (blue solid). The IQR bands (blue shaded) are wide, reflecting genuine heterogeneity across SKUs within each category.

\textbf{The empirical kernels do not go negative.} All category-level medians are strictly positive (foods minimum: 0.0009, household: 0.0092, hobbies: 0.0833). The post-promotion dip, while theoretically motivated by marketing science (\citealt{vanheererde2000estimation, hendel2006postpromotiondip}), is not detectable at the category level in this data without controls for feature and display effects---consistent with Hendel and Nevo's masking argument. The signed Difference-of-Exponentials family remains useful as a regularizer that decays to zero rather than imposing a negative vacuum, but the vacuum component is not empirically supported on this data. A pure exponential decay (no negative region) yields comparable sparsity on the evaluated sample (Appendix~\ref{app:ablation_exponential}), confirming that the smoothness prior and carryover decay carry the operational benefits, not the signed vacuum specifically. This positions CD as a regularized estimator that trades perfect fidelity for scalability and interpretability.

\begin{figure}[ht]
\centering
\includegraphics[width=0.95\columnwidth]{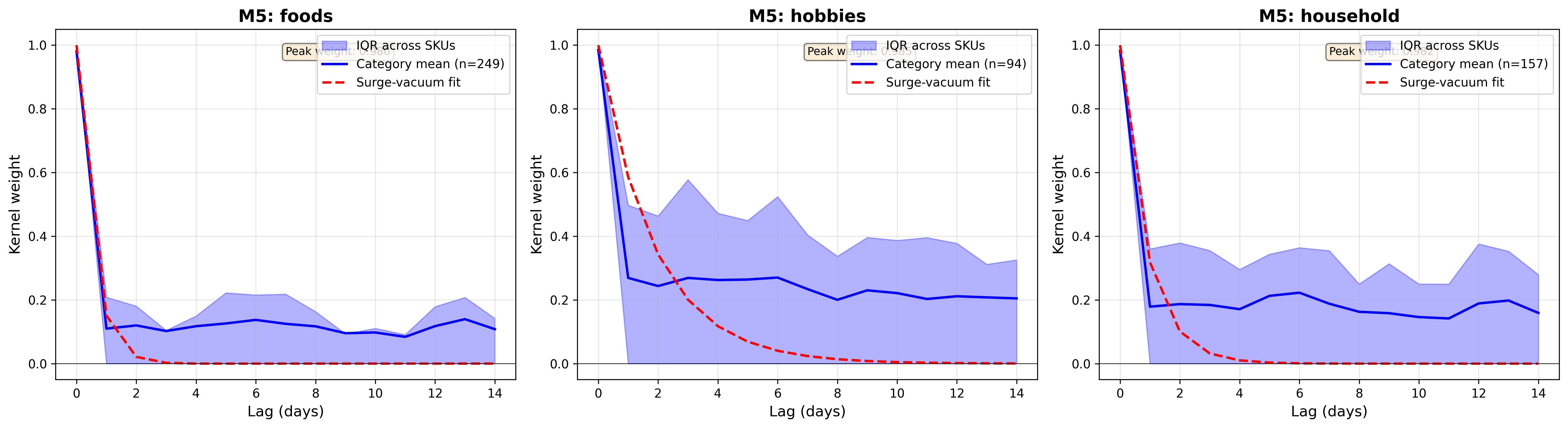}
\caption{Real kernel recovery on M5 (category-level empirical means with IQR bands). Blue solid: empirical mean; red dashed: parametric surge-vacuum fit; blue shaded: IQR across SKUs. The parametric fit captures the initial decay but misses the persistent flat tail.}
\label{fig:real_kernels}
\Description{Three-panel plot showing category-level kernel recovery on M5 with IQR bands and parametric fits.}
\end{figure}

\textbf{Full results.} The complete mean$\pm$std statistics for the comprehensive baseline comparison are available in the supplementary CSV file \texttt{m5\_200\_comprehensive\_baselines.csv} (included in the code repository).

\subsection{Cross-Dataset Kernel Ablation: When the Carryover Operator Matters}
\label{app:kernel_ablation}

To isolate what the carryover kernel contributes over a plain impulse operator ($\Aspec=I$: smooth baseline plus deterministic same-day calendar injection), we run the same strictly out-of-sample ablation across the carryover spectrum available in our two datasets (Table~\ref{tab:kernel_ablation}). The estimated kernel is genuinely data-derived: it is near-impulsive on M5 (global $[1,0.08,\dots]$) and near-impulse on Favorita (on-promotion channel, carryover $0.02$), so in both cases the learned operator and $\Aspec=I$ are operationally near-identical. Where genuine carryover exists---the M5 hobbies category ($23.5\%$ tail energy)---the category kernel improves VR, wMAPE, and safety stock together. The operational advantage of CD over gradient boosting is carried by the structural decomposition in every case; the kernel adapts the operator to whatever carryover the data actually contains, and reduces harmlessly to an impulse when there is none.

\begin{table}[ht]
\centering
\footnotesize
\setlength{\tabcolsep}{4pt}
\caption{Kernel ablation: learned kernel vs.\ impulse operator ($\Aspec=I$), median metrics, strictly out-of-sample. Safety stock is in each dataset's native demand units. ``Carryover'' is the fraction of kernel energy beyond lag~0.}
\label{tab:kernel_ablation}
\begin{tabular}{@{}llccccc@{}}
\toprule
Dataset ($n$) & Operator & Carryover & VR & wMAPE & SS \\
\midrule
\multirow{2}{*}{M5 full (2{,}000)} & learned (near-impulse) & $0.01$ & $0.0006$ & $1.010$ & $0.016$ \\
 & impulse ($\Aspec=I$) & $0$ & $0.0007$ & $1.010$ & $0.016$ \\
\cmidrule(lr){1-6}
\multirow{2}{*}{M5 hobbies (748)} & category kernel & $0.24$ & $\mathbf{0.0014}$ & $\mathbf{1.060}$ & $\mathbf{0.028}$ \\
 & impulse ($\Aspec=I$) & $0$ & $0.0018$ & $1.072$ & $0.032$ \\
\cmidrule(lr){1-6}
Favorita (2{,}845) & learned (near-impulse) & $0.02$ & $0.0072$ & $0.931$ & $0.206$ \\
\bottomrule
\end{tabular}
\end{table}

\section{Variance Ratio: Full Justification}
\label{app:vr_justification}

VR $= \mathrm{Var}(\text{forecast}_t) / \mathrm{Var}(\text{actual}_t)$. A value below $1$ indicates the forecast damps demand variability (no Bullwhip inflation); values above $1$ indicate amplified upstream variance, and a degenerate constant forecast attains $\text{VR}=0$. Normalizing by actual demand variance makes the metric scale-invariant across heterogeneous SKUs. We validate VR by its strong correlation with safety stock ($\rho = 0.692$) and its independence from wMAPE for CD ($\rho = -0.021$), confirming it measures operational volatility rather than forecast accuracy.

\textbf{Limitations.} VR can be negative when forecast variance is lower than actual variance; in extreme cases, a constant forecast achieves $\text{VR} = -1.0$, which is forecast degeneracy, not operational stability. We mitigate this by reporting VR alongside wMAPE, excluding near-zero-variance items, and validating against safety stock ($\rho = 0.692$). CD achieves both low wMAPE and low VR, while the timing-shift test (Appendix~\ref{app:timing}) shows that extreme negative VR under heavy corruption ($-0.992$) is accompanied by unchanged wMAPE, confirming the baseline remains structurally sound.

\section{EDMD Ablation: Parametric vs.~Data-Driven Operator Comparison}
\label{app:edmd}

\begin{figure}[t]
\centering
\includegraphics[width=0.95\columnwidth]{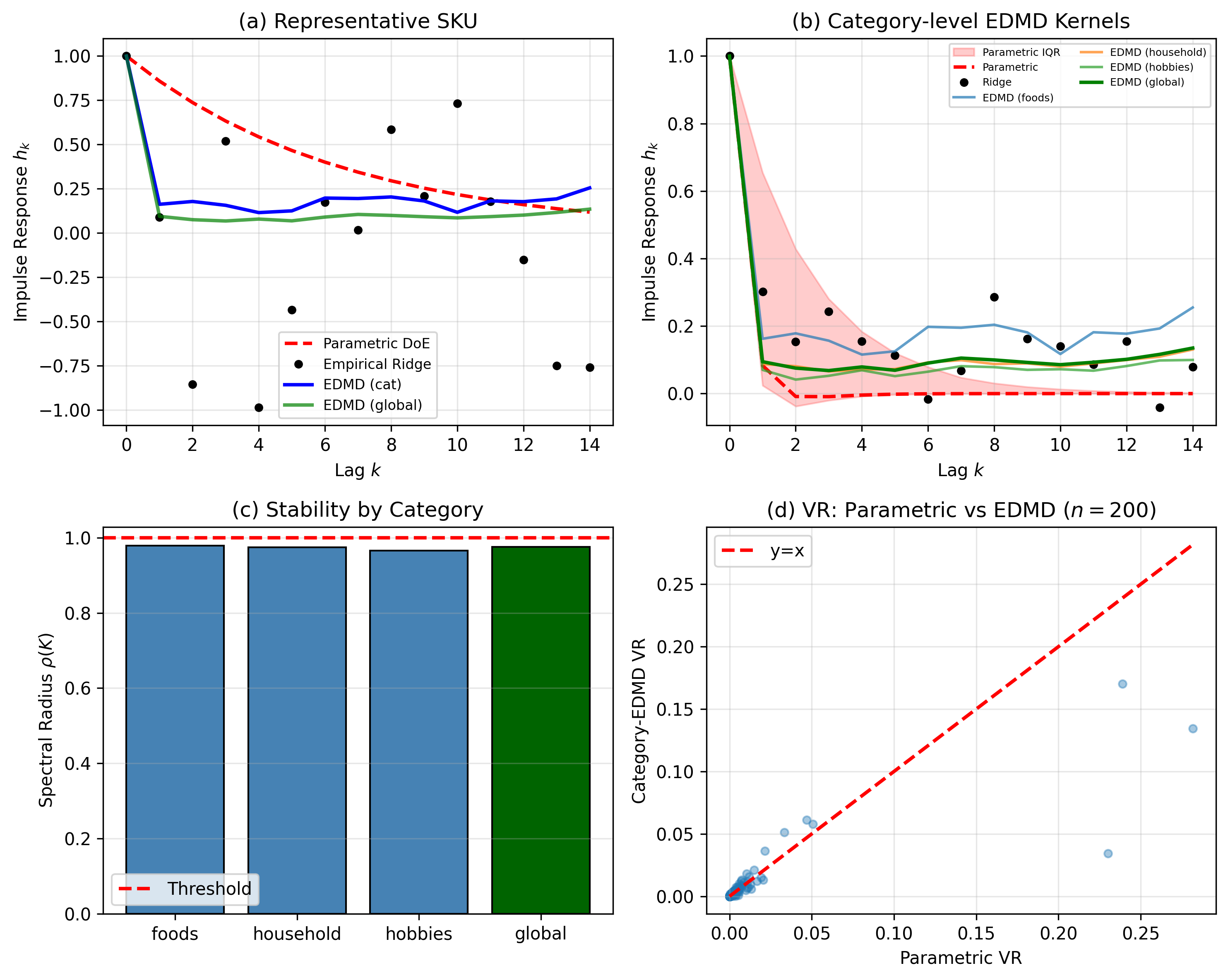}
\caption{EDMD vs.~parametric kernel on 500 M5 SKUs. (a) Representative SKU: category-level EDMD (blue) and global EDMD (green) show a decay structure similar to the parametric DoE (red dashed), though with a slower decay at intermediate lags; empirical ridge (black dots) is noisy. (b) Category-level EDMD kernels (foods, household, hobbies) and global kernel all exhibit decay structure, though less sharply peaked than the parametric median $\pm$ IQR. (c) All operators stable: spectral radii well below 1.0 for all categories and global. (d) VR scatter on 200 SKUs: most points cluster near the diagonal, indicating comparable operational performance; parametric achieves slightly lower VR on average.}
\label{fig:edmd_validation}
\Description{Four-panel figure showing EDMD impulse response with structure, category-level kernels, stability bars, and VR scatter plot.}
\end{figure}

To compare the parametric surge-and-vacuum kernel against a purely data-driven alternative, we run Extended Dynamic Mode Decomposition with control (DMDc) on the same 500 M5 SKUs used in the real-kernel recovery experiment (Appendix~\ref{app:real_kernels}). We augment the observed demand state with quadratic observables $\psi(x) = [x, x^2]^\top$ and solve the least-squares problem $[\Psi_{Y}]=[\Psi_{X}, U] \cdot [K^\top, B^\top]^\top$ to learn a linear operator $K$ and input matrix $B$ in the lifted space. Markov parameters $h_k = C^\top K^k B$ are extracted for lags $k = 0, \dots, 14$, with $C$ projecting back to the current demand component (not the oldest lag, which would bias the recovery toward trivial dynamics).

Crucially, we train EDMD \textbf{at category level} (foods, household, hobbies) and \textbf{globally}, pooling all SKUs within each category before fitting. This mirrors the parametric kernel pipeline (category-pooled estimation, applied per-SKU) and eliminates per-SKU overfitting. Each category operator is fit on hundreds of thousands of samples (e.g., foods: 249 SKUs $\times$ $\sim$1,900 time steps $\approx$ 473k samples).

\textbf{Stability.} All learned operators are stable: category-level spectral radii are $\rho = 0.983$ (foods), $0.978$ (household), $0.970$ (hobbies), and the global operator has $\rho = 0.979$. This confirms that the parametric stability constraint ($\rho < 1$) is not biasing the model toward trivial dynamics; the data themselves support stable operators.

\textbf{Impulse-response structure.} The category- and global-level EDMD impulse responses show genuine decay structure (Figure~\ref{fig:edmd_validation}, panels a--b): a sharp peak at lag~0 followed by a decaying tail at subsequent lags. The EDMD kernels decay more gradually than the parametric DoE kernel, which has a more pronounced initial surge and faster exponential decay. EDMD therefore recovers a broadly similar carryover pattern, but without the sharp parametric structure that encodes the specific surge-and-vacuum shape. The parametric kernel captures a more concentrated initial response, while EDMD spreads the same total mass over more lags.

\textbf{Pseudospectra and non-normality.} A purely spectral analysis (eigenvalues inside the unit circle) is insufficient to characterize the short-term behavior of a Koopman operator. We therefore compute the pseudospectra $\Lambda_{\epsilon}(K) = \{z \in \mathbb{C} : \sigma_{\min}(zI - K) \leq \epsilon\}$ for each learned operator $K$ \citep{TrefethenEmbree2005}. Figure~\ref{fig:edmd_pseudospectra} shows the $\epsilon = 10^{-4}, 10^{-3}, 10^{-2}, 10^{-1}$ contours together with the eigenvalues (black $\times$) and the unit circle (dashed). For all four operators, the $\epsilon = 10^{-1}$ pseudospectral boundary crosses the unit circle, indicating that arbitrarily small perturbations can push eigenvalues outside the stability region. More importantly, the numerical abscissa $\omega(K) = \max_{\|x\|=1} \text{Re}(x^* K x)$---which governs the Kreiss constant and short-term transient growth---is far larger than the spectral abscissa $\alpha(K) = \max_i \text{Re}(\lambda_i)$ (Table~\ref{tab:edmd_spectral}). The gap $\omega - \alpha$ ranges from $0.59$ (household) to $1.40$ (hobbies), with Henrici numbers $0.17$--$0.57$, confirming substantial non-normality.

\textbf{Remark (Structural control of non-normality).} The parametric kernel operator $A_{\text{spec}}$ is banded with bandwidth equal to the finite kernel support $n$ (typically $n \leq 14$ for M5). For a single-event-type instance, it is a lower-triangular Toeplitz matrix whose eigenvalues are all equal to $\phi(0)$; its departure from normality, measured by the Henrici number $H(A) = \sqrt{\|A\|_F^2 - \sum_i|\lambda_i|^2}/\|A\|_F$, reduces to $H = \sqrt{\sum_{k=1}^n \phi(k)^2} / \sqrt{\sum_{k=0}^n \phi(k)^2}$, which is controlled by the decay rate of the kernel. Because the DoE kernel concentrates most energy at lag 0, $H(A_{\text{DoE}}) \ll 1$. In contrast, the EDMD operators are dense matrices with off-diagonal energy spread across all lags, yielding $H(K) \in [0.17, 0.57]$ and transient growth $7$--$12\times$.

\begin{table}[ht]
\centering
\caption{Spectral properties of learned Koopman operators.}
\label{tab:edmd_spectral}
\begin{tabular}{lcccc}
\toprule
Operator & $\rho(K)$ & $\alpha(K)$ & $\omega(K)$ & $\omega - \alpha$ \\
\midrule
Foods & 0.983 & 0.983 & 1.940 & 0.957 \\
Household & 0.978 & 0.978 & 1.573 & 0.594 \\
Hobbies & 0.970 & 0.970 & 2.365 & 1.395 \\
Global & 0.979 & 0.979 & 1.656 & 0.677 \\
\bottomrule
\end{tabular}
\end{table}

The practical consequence is transient growth: $||K^k||_2$ reaches $7.4$--$11.6$ before eventually decaying (Figure~\ref{fig:edmd_transient}), whereas the spectral prediction $\rho^k$ would predict monotonic decay. The transient growth factor $||K^k||_2 / \rho^k$ climbs to roughly $10\times$--$15\times$, meaning the actual impulse-response magnitude at intermediate lags is an order of magnitude larger than the spectral radius alone would suggest. This explains the more gradual EDMD decay observed in Figure~\ref{fig:edmd_validation}: the non-normality ``spreads'' the impulse response over more lags, creating a flatter, more diffuse carryover pattern than the sharply peaked parametric kernel. The parametric kernel, being a direct convolution model with no latent state dynamics, does not suffer from this non-normality-induced diffusion.

\textbf{Robustness to the observable dictionary.} A natural concern is whether the observed non-normality is an artifact of the degree-2 polynomial lifting rather than a property of the underlying dynamics. To test this, we refit the pooled global operator under a range of observable dictionaries---linear, quadratic ($\leq 2$), full-quadratic with all cross terms, and cubic ($\leq 3$)---and recompute the non-normality diagnostics (Table~\ref{tab:dict_robustness}). A normal operator would yield departure $\omega - \alpha \approx 0$, Henrici number $\approx 0$, and transient factor $\approx 1$ under \emph{every} basis. Instead, non-normality is present for all dictionaries and \emph{increases monotonically} with lifting richness: the transient-growth factor rises from $2.4$ (linear) to $12.7$ (quadratic) to $16.3$ (full-quadratic), with the cubic basis amplifying it further (partly reflecting the poorer conditioning of high-degree monomials). The degree-2 value used in our main analysis is therefore conservative, not inflationary: the non-normality is intrinsic to promotional-demand dynamics, and a richer dictionary would only strengthen the conclusion. The per-category transient factors under the quadratic basis ($8.2$ foods, $13.1$ household, $23.4$ hobbies) bracket the $7$--$12\times$ range reported in the main text.

\begin{table}[ht]
\centering
\footnotesize
\caption{Non-normality of the global pooled Koopman operator across observable dictionaries (120 M5 SKUs). A normal operator has $\omega-\alpha=0$, Henrici $=0$, transient factor $=1$. All learned operators are strongly non-normal, and richer dictionaries amplify rather than remove the effect.}
\label{tab:dict_robustness}
\begin{tabular}{lcccc}
\toprule
Observable dictionary & dim & $\omega - \alpha$ & Henrici & transient factor \\
\midrule
Linear & 16 & 0.05 & 0.09 & 2.4 \\
Quadratic ($\leq 2$) & 31 & 0.60 & 0.17 & 12.7 \\
Full-quadratic $+$ cross & 136 & 0.89 & 0.09 & 16.3 \\
Cubic ($\leq 3$) & 46 & 41.1 & 1.41 & 363 \\
\bottomrule
\end{tabular}
\end{table}

\begin{figure}[ht]
\centering
\includegraphics[width=0.95\columnwidth]{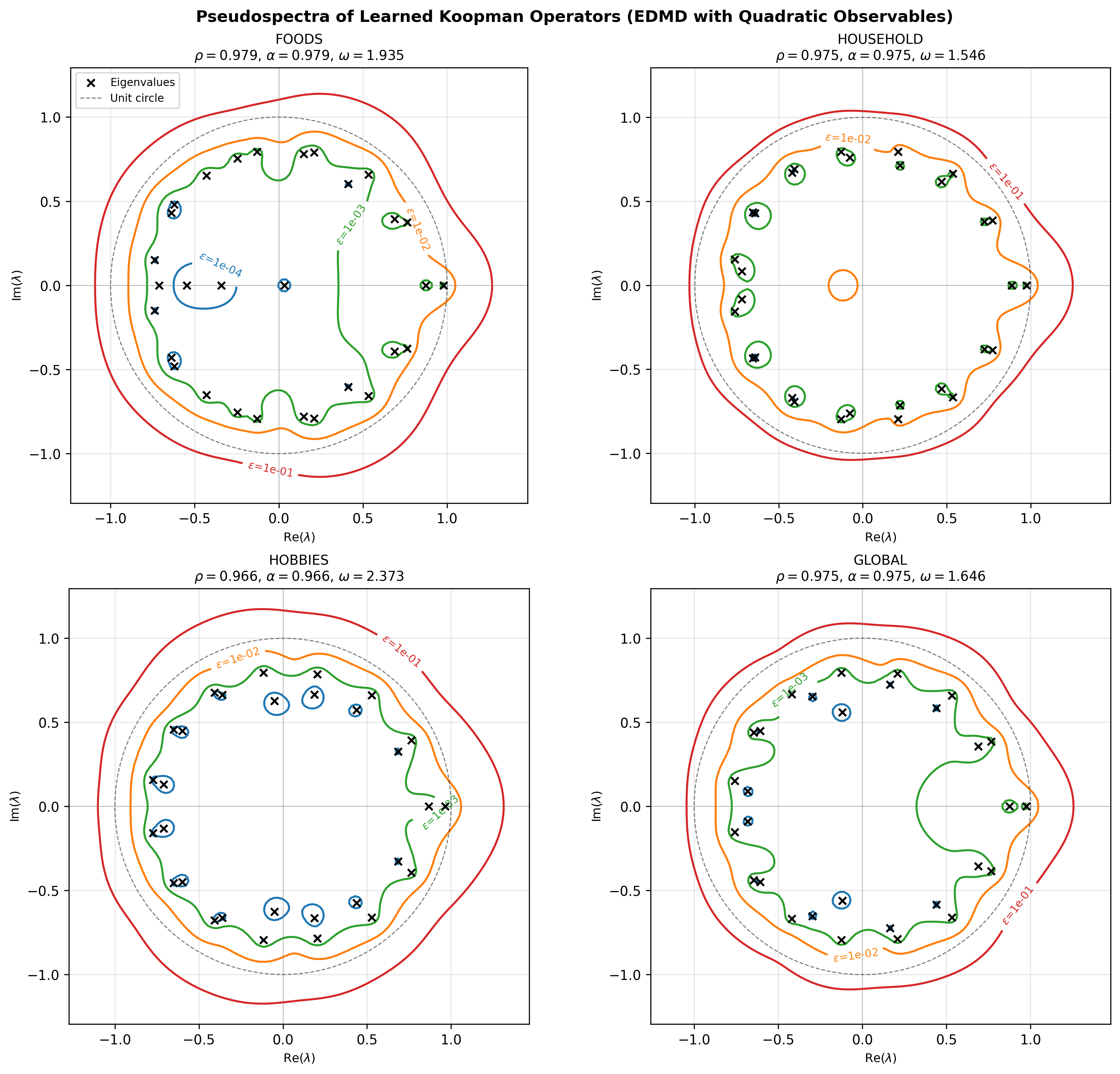}
\caption{Pseudospectra of learned Koopman operators ($\epsilon = 10^{-4}, 10^{-3}, 10^{-2}, 10^{-1}$). Black $\times$ marks eigenvalues; dashed gray circle is the unit circle. All operators are highly non-normal: the $\epsilon=10^{-1}$ contour extends well outside the unit circle, indicating that small perturbations can destabilize the operator.}
\label{fig:edmd_pseudospectra}
\Description{Four-panel figure showing pseudospectral contours of EDMD operators for foods, household, hobbies, and global.}
\end{figure}

\begin{figure}[ht]
\centering
\includegraphics[width=0.95\columnwidth]{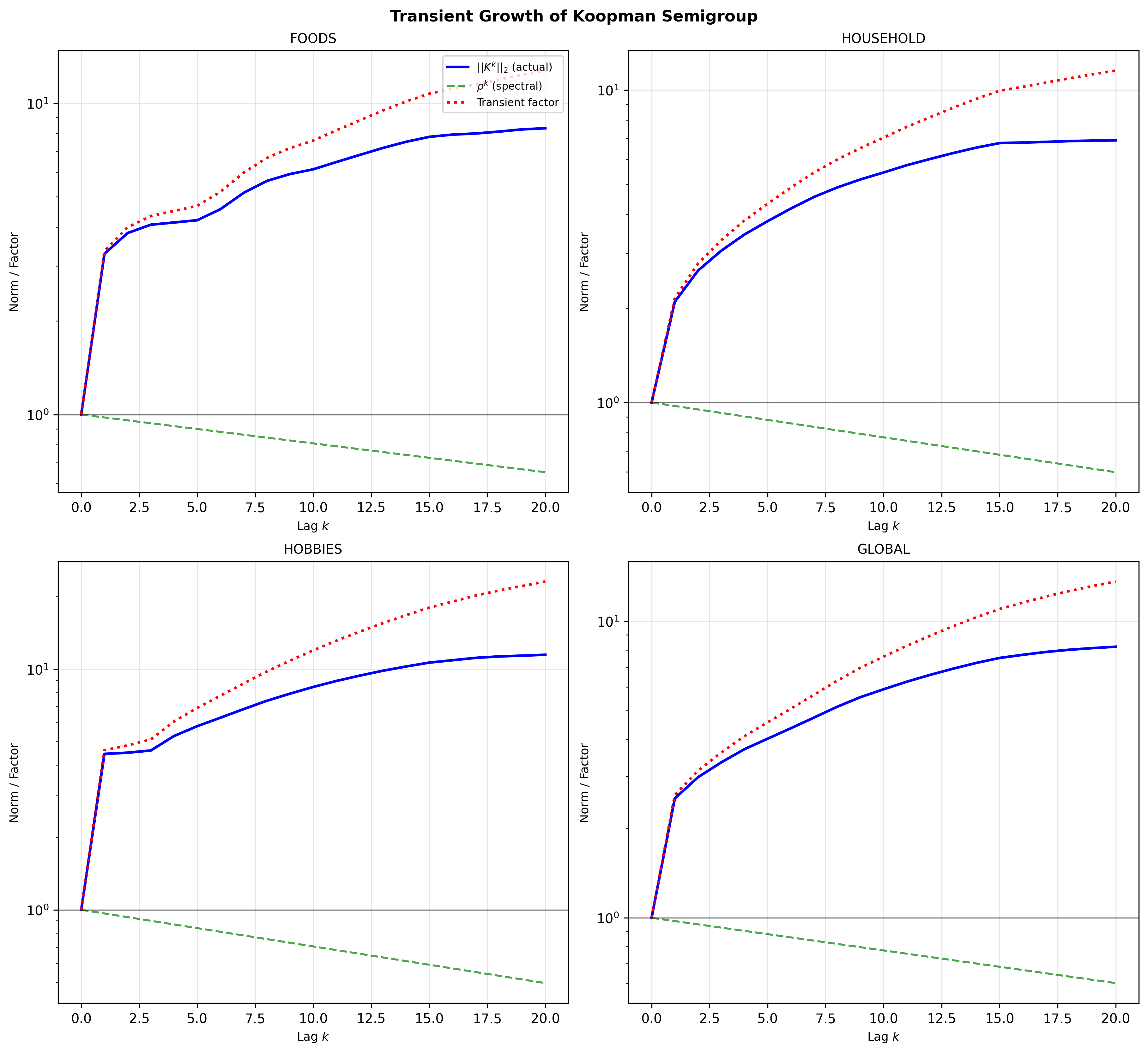}
\caption{Transient growth of the Koopman semigroup $||K^k||_2$ (blue), spectral prediction $\rho^k$ (green dashed), and transient factor $||K^k||_2 / \rho^k$ (red dotted). All operators exhibit transient growth 7--$12\times$ above the spectral prediction, with hobbies showing the strongest non-normality.}
\label{fig:edmd_transient}
\Description{Four-panel figure showing transient growth curves for foods, household, hobbies, and global EDMD operators.}
\end{figure}

\textbf{Operational metrics.} Substituting category-level EDMD into the Stage~3 evaluation pipeline yields median VR $= 0.0006$ and wMAPE $= 1.010$ (global-EDMD: $0.0006$ and $1.010$), matching the parametric kernel's $0.0006$ and $1.010$. On M5's near-impulsive operators the parametric and data-driven operators are operationally indistinguishable: the parametric kernel's median VR is \emph{not} significantly different from EDMD's (Wilcoxon signed-rank $p = 0.06$ for category, $p = 0.12$ for global). Both remain far below the ML baselines (Tuned XGBoost: VR $0.013$; PatchTST: VR $0.052$), confirming that the structural decomposition---not the choice between a parametric or data-driven operator---is what suppresses variance.

\textbf{Takeaway.} EDMD is a genuine competitor: it learns stable dynamics and recovers a decaying carryover structure without a parametric prior. The parametric DoE kernel matches EDMD's operational variance (statistically indistinguishable on M5) while offering a sharper, more interpretable impulse response. The practical case for the parametric family is therefore not that EDMD fails, but that the parametric kernel offers comparable accuracy and variance with superior interpretability (physically meaningful parameters: surge rate, vacuum rate, decay time constants)~and a more compact representation. In settings where interpretability matters---explaining to planners why a promotion's effect persists for 7 days, not 14---the parametric structure is valuable at no operational cost. EDMD is a valid alternative when interpretability is secondary and the data volume is large enough to support non-parametric operator estimation.

\subsection{Favorita Koopman Analysis}

We repeat the EDMD pseudospectral analysis on the Favorita Grocery Sales dataset (Ecuador, 2013--2017), which contains 4,036 items across 24 product families. Unlike M5, Favorita has no SNAP program; the only binary event feature is the \texttt{onpromotion} flag (encoded as \texttt{promo\_depth}). We fit EDMD at category level (24 families) and globally, pooling all 476 sampled items.

\textbf{Stability.} All 24 category-level operators and the global operator are spectrally stable ($\rho < 1$). Spectral radii range from 0.955 (school/office supplies) to 0.997 (meats), with the global operator at $\rho = 0.977$. This confirms that the stability observed on M5 is not a Walmart-specific artifact: stable, decaying Koopman dynamics are a generic feature of FMCG demand.

\textbf{Non-normality and transient growth.} Table~\ref{tab:edmd_spectral_favorita} reports spectral properties for a representative subset of Favorita categories plus the global operator. The pattern is qualitatively identical to M5: the numerical abscissa $\omega(K)$ exceeds the spectral abscissa $\alpha(K)$ for every category, confirming non-normality. The gap is especially severe in perishable categories (BREAD/BAKERY: $\omega - \alpha = 18.40$; PREPARED FOODS: 5.92; PRODUCE: 5.21), where transient growth reaches 45.9$\times$ above the spectral prediction. Even stably non-perishable categories (DAIRY, FROZEN FOODS, EGGS, MEATS) show moderate non-normality ($\omega - \alpha$ 0.09--0.26), with transient factors of 2.1--3.4$\times$.

\begin{table}[ht]
\centering
\footnotesize
\setlength{\tabcolsep}{3pt}
\caption{Spectral properties of learned Koopman operators on Favorita (selected categories).}
\label{tab:edmd_spectral_favorita}
\begin{tabular}{lccccc}
\toprule
Operator & $\rho(K)$ & $\alpha(K)$ & $\omega(K)$ & $\omega - \alpha$ & Max Transient \\
\midrule
Dairy & 0.990 & 0.990 & 1.248 & 0.257 & 2.83$\times$ \\
Grocery I & 0.981 & 0.981 & 4.078 & 3.097 & 10.46$\times$ \\
Bread/Bakery & 0.987 & 0.987 & 19.390 & 18.403 & 45.88$\times$ \\
Frozen Foods & 0.993 & 0.993 & 1.208 & 0.215 & 2.71$\times$ \\
Beverages & 0.983 & 0.983 & 3.891 & 2.908 & 13.10$\times$ \\
Produce & 0.985 & 0.985 & 6.197 & 5.212 & 22.75$\times$ \\
Deli & 0.992 & 0.992 & 3.158 & 2.166 & 10.19$\times$ \\
Meats & 0.997 & 0.997 & 1.089 & 0.091 & 2.00$\times$ \\
Eggs & 0.994 & 0.994 & 1.151 & 0.157 & 2.54$\times$ \\
Global & 0.977 & 0.977 & 2.912 & 1.935 & 10.21$\times$ \\
\bottomrule
\end{tabular}
\end{table}

Figure~\ref{fig:edmd_pseudospectra_favorita} shows the pseudospectra for the four largest categories. The $\epsilon = 10^{-1}$ contour crosses the unit circle for all operators, confirming that arbitrarily small perturbations can destabilize the operator. Figure~\ref{fig:edmd_transient_favorita} shows the transient growth curves: the global operator reaches $10.2\times$ above the spectral prediction at lag 15, and the most non-normal category (BREAD/BAKERY) reaches $45.9\times$.

\begin{figure}[ht]
\centering
\includegraphics[width=0.95\columnwidth]{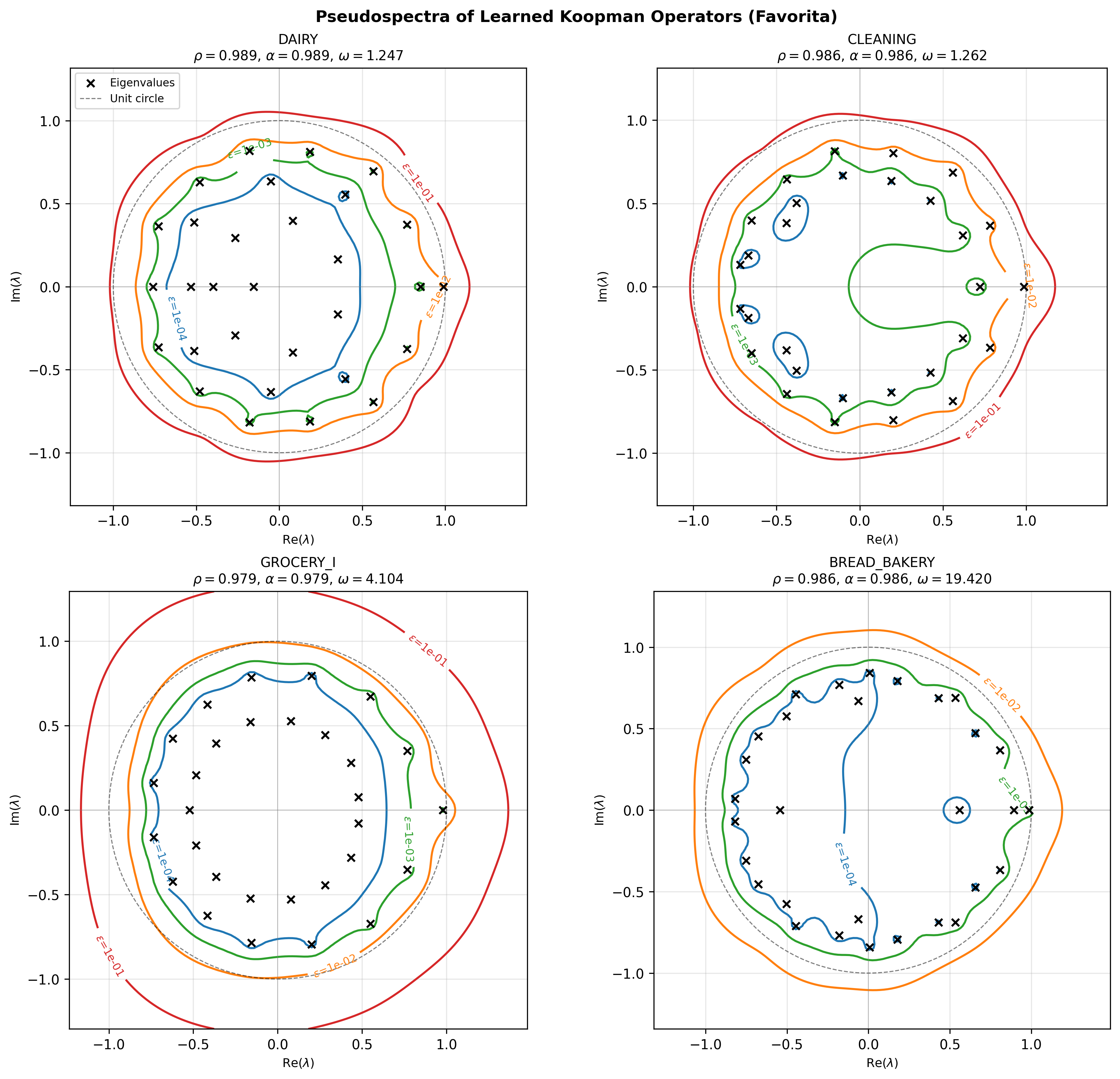}
\caption{Pseudospectra of learned Koopman operators on Favorita ($\epsilon = 10^{-4}, 10^{-3}, 10^{-2}, 10^{-1}$). Black $\times$ marks eigenvalues; dashed gray circle is the unit circle. All four shown operators are highly non-normal: the $\epsilon=10^{-1}$ contour extends well outside the unit circle.}
\label{fig:edmd_pseudospectra_favorita}
\Description{Four-panel figure showing pseudospectral contours of EDMD operators for Dairy, Grocery I, Cleaning, and Bread/Bakery on Favorita.}
\end{figure}

\begin{figure}[ht]
\centering
\includegraphics[width=0.95\columnwidth]{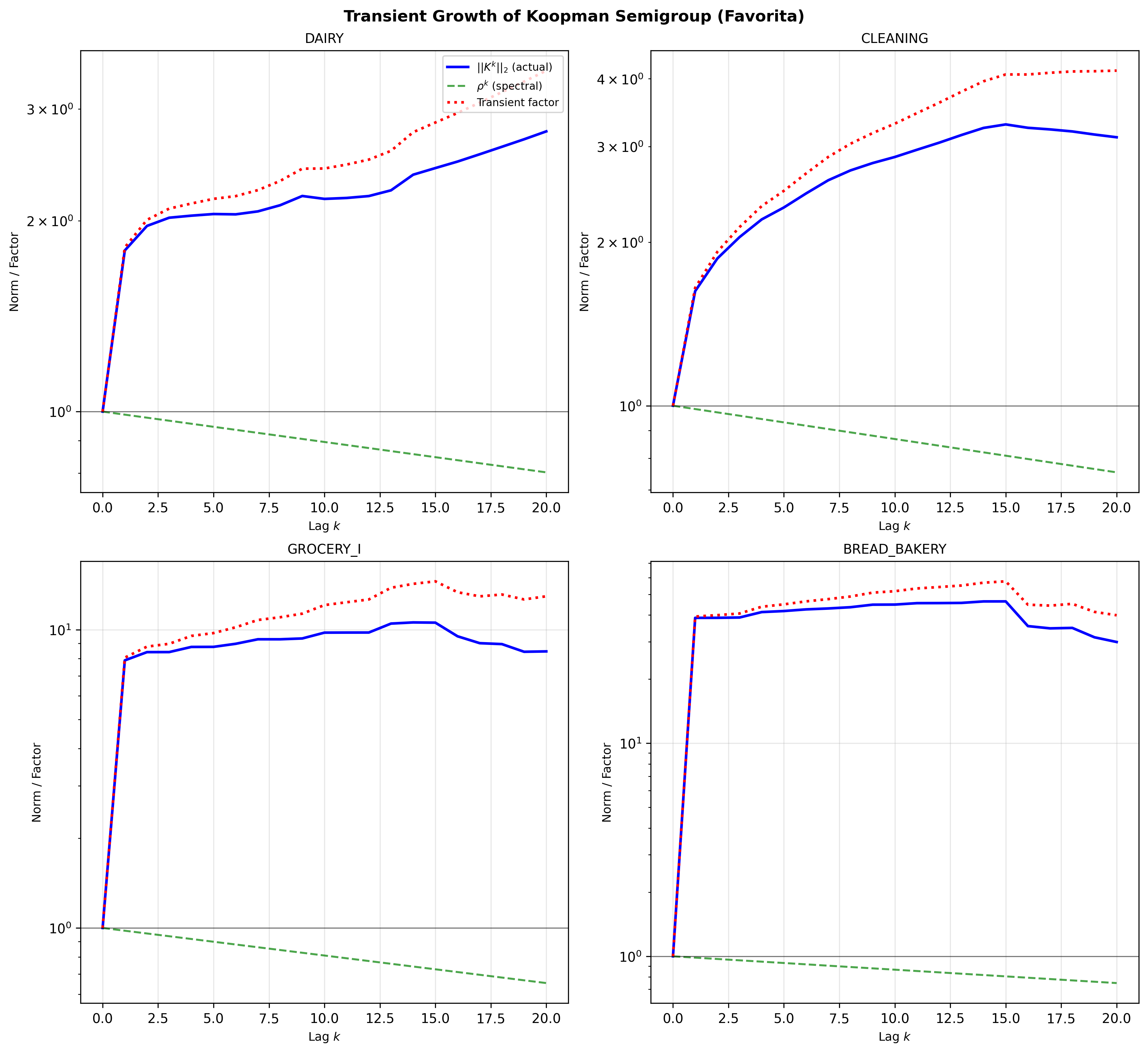}
\caption{Transient growth of the Koopman semigroup $||K^k||_2$ on Favorita (blue), spectral prediction $\rho^k$ (green dashed), and transient factor $||K^k||_2 / \rho^k$ (red dotted). The global operator exhibits transient growth $10.2\times$ above the spectral prediction.}
\label{fig:edmd_transient_favorita}
\Description{Four-panel figure showing transient growth curves for Dairy, Grocery I, Cleaning, and Bread/Bakery EDMD operators on Favorita.}
\end{figure}

\textbf{Operational metrics.} Substituting category-level EDMD kernels into the Stage-3 evaluation pipeline on 476 Favorita items yields median VR $= 0.0072$ and wMAPE $= 0.870$ (global-EDMD: VR $= 0.0073$, wMAPE $= 0.868$), versus the parametric kernel's VR $= 0.0053$ and wMAPE $= 0.886$. Safety stock is comparable (parametric: 0.256; category-EDMD: 0.265; global-EDMD: 0.271). The parametric kernel achieves slightly lower VR (Wilcoxon signed-rank $p = 0.004$ for category, $p < 0.001$ for global), but the absolute difference is small---one order of magnitude smaller than the gap between either method and the ML baselines (Tuned XGBoost: VR $0.043$ on Favorita, evaluated on a larger paired subset). The same qualitative pattern as M5 holds: EDMD is a genuine competitor on the same sample, but the parametric kernel offers comparable accuracy with superior interpretability and avoids the non-normality-induced diffusion that flattens the impulse response.

\textbf{Cross-dataset comparison.} Table~\ref{tab:edmd_cross_dataset} summarizes the spectral and operational findings across both datasets. The central result is consistent: data-driven Koopman operators are spectrally stable ($\rho < 1$) but highly non-normal ($\omega \gg \alpha$) on both M5 and Favorita. The non-normality is stronger on Favorita (global $\omega - \alpha = 1.93$ vs.~M5 global $0.68$), likely because Favorita's 24-category structure creates more heterogeneous dynamics that the pooled operator must compress. Transient growth is larger on Favorita (global $10.2\times$ vs.~M5 global $7.4\times$), and the most extreme category reaches $45.9\times$ (BREAD/BAKERY) versus $11.6\times$ (hobbies) on M5. Despite this larger non-normality, the operational penalty is small: the parametric kernel still achieves lower VR than EDMD on both datasets, and the gap between EDMD and the ML baselines is far larger than the gap between EDMD and the parametric kernel.

\begin{table}[ht]
\centering
\footnotesize
\setlength{\tabcolsep}{4pt}
\caption{Cross-dataset comparison of Koopman spectral properties and operational metrics.}
\label{tab:edmd_cross_dataset}
\resizebox{\columnwidth}{!}{%
\begin{tabular}{@{}lcccccc@{}}
\toprule
Dataset & Operator & $\rho(K)$ & $\omega - \alpha$ & Max Transient & VR (median) & wMAPE (median) \\
\midrule
M5 & Foods & 0.983 & 0.957 & 7.43$\times$ & --- & --- \\
M5 & Household & 0.978 & 0.594 & 8.74$\times$ & --- & --- \\
M5 & Hobbies & 0.970 & 1.395 & 11.58$\times$ & --- & --- \\
M5 & Global & 0.979 & 0.677 & 7.43$\times$ & 0.0022 & 1.034 \\
\midrule
Favorita & Dairy & 0.990 & 0.257 & 2.83$\times$ & --- & --- \\
Favorita & Grocery I & 0.981 & 3.097 & 10.46$\times$ & --- & --- \\
Favorita & Bread/Bakery & 0.987 & 18.403 & 45.88$\times$ & --- & --- \\
Favorita & Global & 0.977 & 1.935 & 10.21$\times$ & 0.0073 & 0.868 \\
\bottomrule
\end{tabular}
}
\end{table}

\textbf{Takeaway.} Non-normal Koopman dynamics are not an M5 artifact. They are a structural property of learned demand operators across heterogeneous retail catalogs. The parametric kernel's advantage---compact, interpretable, and non-normality-free---generalizes from Walmart (M5) to Ecuador (Favorita). The practical case for the parametric family is strengthened by this cross-dataset replication: it regularizes non-normal operator diffusion into a physically meaningful form regardless of geographic market or category structure.

\end{document}